\documentclass{article}
\usepackage{graphicx}
\usepackage{subcaption}

   \usepackage[preprint,nonatbib]{neurips_2024}

\usepackage[utf8]{inputenc} 
\usepackage[T1]{fontenc}    
\usepackage{url}            
\usepackage{booktabs}       
\usepackage{amsfonts}       
\usepackage{nicefrac}       
\usepackage{microtype}      
\usepackage{xcolor}  
\usepackage{soul}

\title{Revisiting Incremental Stochastic Majorization-Minimization Algorithms with Applications to Mixture of Experts}

\author{TrungKhang Tran$^{\star,1}$,~TrungTin Nguyen$^{\star ,2,3}$,~\textbf{Gersende Fort}$^{4}$,~Tung Doan$^{5}$,\\\\
	~\textbf{Hien Duy Nguyen}$^{6,7}$,~\textbf{Binh T. Nguyen}$^{8,9}$,~\textbf{Florence Forbes}$^{10}$,~\textbf{Christopher Drovandi}$^{2,3}$
	\\\\
	$^{1}$School of Computing, National University of Singapore, Singapore.\\
	$^{2}$ARC Centre of Excellence for the Mathematical Analysis of Cellular Systems. \\
	$^{3}$School of Mathematical Sciences, Queensland University of Technology, Brisbane, Australia.
	\\
	$^{4}$ Laboratoire d'Analyse et d'Architecture des Systèmes, CNRS, Toulouse, France.\\
	$^{5}$School of Medicine and Dentistry, Griffith University, Brisbane, Australia.\\
	$^{6}$Department of Mathematics and Physical Sciences, La Trobe University, Melbourne Australia.\\
	$^{7}$Institute of Mathematics for Industry, Kyushu University, Fukuoka, Japan.\\ 
	$^{8}$Faculty of Mathematics and Computer Science, University of Science, Ho Chi Minh City, Vietnam.\\
	$^{9}$Vietnam National University, Ho Chi Minh City, Vietnam.\\
	$^{10}$Univ. Grenoble Alpes, Inria, CNRS, Grenoble INP, LJK, 38000, Grenoble, France.\\
}

\usepackage[utf8]{inputenc}
\usepackage[english]{babel}
\usepackage{mathtools,amsfonts,amsmath,amsthm}
\usepackage{soul}

\usepackage{algorithm}
\usepackage{algorithmicx}
\usepackage{algpseudocode}

\usepackage{color}
\usepackage{comment}
\usepackage{url}

\newcommand{\md}{softmax-gated MoE }
\newcommand{\mdg}{softmax-gated Gaussian MoE }
\newcommand{\mdm}{softmax-gated multinomial logistic MoE }
\newcommand{\Md}{Softmax-gated MoE }
\newcommand{\Mdg}{Softmax-gated Gaussian MoE }
\newcommand{\Mdm}{Softmax-gated multinomial logistic MoE }

\usepackage{lipsum} 
\newcommand\blfootnote[1]{%
  \begingroup
  \renewcommand\thefootnote{}\footnote{#1}%
  \addtocounter{footnote}{-1}%
  \endgroup
}

\usepackage{xcolor}
\definecolor{forestgreen}{rgb}{0.13, 0.55, 0.13}

\definecolor{frenchblue}{rgb}{0.0, 0.45, 0.73}

\definecolor{cherryblossompink}{rgb}{1.0, 0.72, 0.77}

\definecolor{bittersweet}{rgb}{1.0, 0.44, 0.37}

\definecolor{navyblue}{rgb}{0.0, 0.0, 0.5}

\usepackage{enumerate}

\usepackage{bbm}

\usepackage{dirtytalk}

\usepackage[square,numbers]{natbib}

\usepackage{url}

\usepackage[colorlinks,linkcolor=blue,citecolor=blue,pagebackref=true]{hyperref}

\renewcommand*{\backrefalt}[4]{%
	\ifcase #1 \footnotesize{(Not cited.)}%
	\or        \footnotesize{(Cited on page~#2.)}%
	\else      \footnotesize{(Cited on pages~#2.)}%
	\fi}

\usepackage[nameinlink,capitalize,noabbrev]{cleveref}
\usepackage{environ}
\usepackage{nicefrac}

\newcounter{algsubstate}
\renewcommand{\thealgsubstate}{\alph{algsubstate}}

\def\ie{{\em i.e.,~}}
\def\eg{{\em e.g.,~}}

\newcommand{\zero}{\ensuremath{\mathbf{0}}}

\DeclareMathOperator*{\vect}{vec} 
\DeclareMathOperator*{\mat}{mat}
\DeclareMathOperator{\bdiag}{bdiag}

\newcommand{\sbm}{\cS^{\cB}_{(K,J)}} 

\let\inf\relax 
\DeclareMathOperator*\inf{\vphantom{p}inf}

\newcommand{\nn}{\nonumber} 

\newcommand{\I}{\mathbbm{1}}

\newcommand{\Ep}[1]{\mathbb{E}\left(#1\right)}

\newcommand{\Ebd}[2]{\mathbb{E}_{#1} \left[#2\right]}

\newcommand{\vertiii}[1]{{\left\vert\kern-0.25ex\left\vert\kern-0.25ex\left\vert #1 
		\right\vert\kern-0.25ex\right\vert\kern-0.25ex\right\vert}}
\newcommand{\vertii}[1]{{\left\vert\kern-0.25ex\left\vert#1\right\vert\kern-0.25ex\right\vert}}	

\newcommand{\cN}{\mathcal{N}}

\def\pib         {{\sbmm{\pi}}\XS}

\def\phib        {{\sbmm{\phi}}\XS}

\def\omegab      {{\sbmm{\omega}}\XS}      
\def\taub        {{\sbmm{\tau}}\XS}

\usepackage{xspace}
\def\XS{\xspace}
\DeclareMathAlphabet{\mathb}{OML}{cmm}{b}{it}
\def\sbm#1{\ensuremath{\mathb{#1}}}
\def\sbmm#1{\ensuremath{\boldsymbol{#1}}}  

  \def\vb{{\sbm{v}}\XS}
  
  \def\xb{{\sbm{x}}\XS}

\newtheorem{remark}{Remark}

\newtheorem{assumption}{Assumption}
\newtheorem{assumptionP}{P}

\newtheorem{lemma}{Lemma}

\newtheorem{theorem}{Theorem}
\newtheorem{proposition}{Proposition}

\newtheorem{condition}{Condition}

\crefname{assumption}{Assumption}{Assumptions}
\crefname{assumptionP}{P}{P}

\crefname{ineq}{Inequality}{Inequalities}
\creflabelformat{ineq}{#2(#1)#3}
\NewEnviron{inequality}{
\crefalias{equation}{ineq}
\begin{align}
    \BODY
\end{align}
}

\def\eqdef{:=}
\def\nset{{\mathbb{N}}}
\def\rset{\mathbb{R}}

\def\PP{\mathbb{P}_\pi} 
\def\PE{\mathbb{E}_\pi} 
\def\F{\mathcal{F}}
\def\param{\vtheta}
\def\barparam{\bar{\param}}

\def\Kset{\mathbb{K}}
\newcommand{\pscal}[2]{\left\langle#1,#2\right\rangle} 
\def\bars{\bar{\vS}}
\def\sfg{\mathsf{g}}
\def\sfm{\mathsf{m}}
\def\ex{\mathsf{e}}
\def\lyap{V}
  \newcommand{\ooint}[1]{\left(#1\right)}

\usepackage{amsmath,amsfonts,bm}

\def\eqref#1{equation~\ref{#1}}

\def\1{\bm{1}}

\def\ry{{\textnormal{y}}}

\def\rvs{{\mathbf{s}}}

\def\rvx{{\mathbf{x}}}
\def\rvy{{\mathbf{y}}}
\def\rvz{{\mathbf{z}}}

\def\vmu{{\bm{\mu}}}
\def\veta{{\bm{\eta}}}

\def\vxi{{\bm{\xi}}}
\def\vtau{{\bm{\tau}}}
\def\valpha{{\bm{\alpha}}}
\def\vsigma{{\bm{\sigma}}}
\def\vkappa{{\bm{\kappa}}}
\def\vzeta{{\bm{\zeta}}}
\def\vSigma{{\bm{\Sigma}}}
\def\vomega{{\bm{\omega}}}
\def\evomega{\omega}
\def\vupsilon{{\bm{\upsilon}}}
\def\evupsilon{\upsilon}
\def\vtheta{{\bm{\theta}}}
\def\va{{\bm{a}}}
\def\vb{{\bm{b}}}
\def\vc{{\bm{c}}}

\def\vf{{\bm{f}}}
\def\vg{{\bm{g}}}

\def\vr{{\bm{r}}}
\def\vs{{\bm{s}}}
\def\vS{{\bm{S}}}
\def\vt{{\bm{t}}}
\def\vu{{\bm{u}}}
\def\vv{{\bm{v}}}
\def\vw{{\bm{w}}}
\def\vx{{\bm{x}}}
\def\vy{{\bm{y}}}
\def\vz{{\bm{z}}}

\def\evomega{{\omega}}

\def\evc{{c}}

\def\evu{{u}}
\def\evv{{v}}
\def\evw{{w}}
\def\evx{{x}}
\def\evy{{y}}

\def\mA{{\bm{A}}}
\def\mDelta{{\bm{\Delta}}}
\def\mB{{\bm{B}}}
\def\mC{{\bm{C}}}
\def\mD{{\bm{D}}}

\def\mI{{\bm{I}}}

\def\mM{{\bm{M}}}

\def\mS{{\bm{S}}}

\def\mSigma{{\bm{\Sigma}}}
\def\mUpsilon{{\bm{\Upsilon}}}

\DeclareMathAlphabet{\mathsfit}{\encodingdefault}{\sfdefault}{m}{sl}
\SetMathAlphabet{\mathsfit}{bold}{\encodingdefault}{\sfdefault}{bx}{n}

\def\sF{{\mathbb{F}}}

\def\sL{{\mathbb{L}}}

\def\sN{{\mathbb{N}}}

\def\sP{{\mathbb{P}}}

\def\sR{{\mathbb{R}}}
\def\sS{{\mathbb{S}}}
\def\sT{{\mathbb{T}}}

\def\sX{{\mathbb{X}}}
\def\sY{{\mathbb{Y}}}
\def\sZ{{\mathbb{Z}}}

\newcommand{\E}{\mathbb{E}}

\newcommand{\R}{\mathbb{R}}

\DeclareMathOperator*{\argmin}{arg\,min}

\begin{document}

\maketitle

\begin{abstract}
\blfootnote{$^{\star}$Co-first author.}

Processing high-volume, streaming data is increasingly common in modern statistics and machine learning, where batch-mode algorithms are often impractical because they require repeated passes over the full dataset. This has motivated incremental stochastic estimation methods, including the incremental stochastic Expectation-Maximization (EM) algorithm formulated via stochastic approximation. In this work, we revisit and analyze an incremental stochastic variant of the Majorization-Minimization (MM) algorithm, which generalizes incremental stochastic EM as a special case. Our approach relaxes key EM requirements, such as explicit latent-variable representations, enabling broader applicability and greater algorithmic flexibility. We establish theoretical guarantees for the incremental stochastic MM algorithm, proving consistency in the sense that the iterates converge to a stationary point characterized by a vanishing gradient of the objective. We demonstrate these advantages on a softmax-gated mixture of experts (MoE) regression problem, for which no stochastic EM algorithm is available. Empirically, our method consistently outperforms widely used stochastic optimizers, including stochastic gradient descent, root mean square propagation, adaptive moment estimation, and second-order clipped stochastic optimization. These results support the development of new incremental stochastic algorithms, given the central role of softmax-gated MoE architectures in contemporary deep neural networks for heterogeneous data modeling. Beyond synthetic experiments, we also validate practical effectiveness on two real-world datasets, including a bioinformatics study of dent maize genotypes under drought stress that integrates high-dimensional proteomics with ecophysiological traits, where incremental stochastic MM yields stable gains in predictive performance.

\end{abstract}


\tableofcontents

\section{Introduction}

\subsection{Mixture of experts models}

Mixture of Experts (MoE) models, originally introduced by \cite{jacobs1991adaptive,jordan1994hierarchical}, are extensively utilized to capture intricate non-linear relationships between input and output variables in heterogeneous datasets. These models effectively address the dual challenges of regression and clustering by decomposing the predictive framework into a combination of gating models and expert models, both of which depend on the input variables. As a notable example of conditional computation \citep{bengio_deep_2013,chen_towards_2022}, MoE models allocate distinct experts to specific regions within the input space. This architecture allows MoE models to substantially enhance model capacity while maintaining nearly constant training and inference costs by leveraging only a subset of parameters for each example \citep{shazeer_outrageously_2017,fedus_switch_2022,dai_deepseekmoe_2024}.

Moreover, MoE models have gained prominence due to their universal approximation capabilities and nearly optimal estimation rates in various contexts. These include mixture models \citep{genovese2000rates,rakhlin2005risk, nguyen2013convergence, ho_convergence_2016, ho_strong_2016, nguyen_approximation_2020, nguyen_approximation_2022,chong_risk_2024}, mixture of regression models \citep{do_strong_2022,ho_convergence_2022}, and more generally, mixture of experts models \citep{jiang_hierarchical_1999, norets_approximation_2010, nguyen2016universal,nguyen2019approximation,nguyen_approximations_2021,nguyen_non_asymptotic_2022,nguyen_demystifying_2023,nguyen_bayesian_2024,nguyen_general_2024,nguyen_towards_2024,le_mixture_2024,thai_model_2025,hai_dendrograms_2026}. For a comprehensive overview of MoE models and their applications, we direct readers to the texts of \cite{yuksel2012twenty,masoudnia2014mixture,nguyen2018practical,nguyen_model_2021,cai_survey_2025}.

\subsection{Majorization–Minimization algorithms}\label{sec_intro_onlineMM}
The Expectation-Maximization (EM) \citep{dempster1977maximum,meilijson_fast_1989,mclachlan_em_1997} algorithm has become a fundamental tool in computational statistics, boasting a wide range of statistical and signal processing applications. Its significance was quickly recognized and embraced by the international statistical community. In contrast, the body of literature on the more general Majorization–Minimization (MM) algorithm \citep{ortega_iterative_1970,leeuw_application_1977,de_leeuw_convergence_1977,lange2016mm,sun_majorization_minimization_2017,nguyen_introduction_2017,lange_nonconvex_2021,mairal_incremental_2015,mairal2013stochastic} has traditionally been smaller. A key strength of these algorithms is their ability to leverage computationally efficient surrogate functions, which effectively replace challenging optimization objectives, thereby simplifying the computational iterations. While the MM principle can be recognized as early as the work of \cite{ortega_iterative_1970} in the context of line search methods, its first statistical application appeared in  \cite{leeuw_application_1977} and \cite{de_leeuw_convergence_1977}, in the context of multidimensional scaling. The limited recognition of these pioneering papers arguably delayed the broader adoption and development of MM algorithms within computational statistics. They are well-established optimization frameworks that play a pivotal role in the development of estimation methodologies for a wide range of data analysis models. While these frameworks are commonly applied to finite mixture models, their application to the more general MoE models, such as \md models~\citep{jacobs1991adaptive,jordan1994hierarchical}, remains limited.

With the increasing volume of data and the incremental nature of data acquisition, significant advancements have been made in the development of incremental stochastic and mini-batch algorithms. These approaches enable model estimation without requiring simultaneous access to the entire dataset.
Many incremental and mini-batch extensions of EM algorithms can be interpreted as instances of the classical Stochastic Approximation framework (see, e.g., \cite{borkar_stochastic_2008,kushner_stochastic_2003,dieuleveut_stochastic_2023,fort_stochastic_2023,wang_spiderboost_2019}) and the online-EM formulations in \cite{cappe_Online_2009,cappe_online_2011,fort_fast_2021,fort_stochastic_2020,nguyen_mini_batch_2020,kuhn_properties_2020,karimi_global_2019,corff_online_2013,karimi_non_asymptotic_2019,oudoumanessah2024,oudoumanessah2025,chen_stochastic_2018,fort_geom_spider_em_2021,fort_perturbed_2021}. At the same time, there also exist incremental and sparse EM variants that are justified from a free-energy or variational viewpoint rather than from stochastic approximation, such as the framework of \cite{neal_view_1998}. Numerical evaluations across this body of work show that such online and incremental EM-type procedures are highly effective for estimation in mixture models and related machine-learning problems.
In contrast, incremental stochastic and mini-batch adaptations of MM algorithms have primarily been developed using techniques from incremental stochastic learning and convex optimization (see, \eg~\cite{chouzenoux2022sabrina,hazan_introduction_2016,jiang2024stochastic,lan_first_order_2020,lupu2024convergence,razaviyayn_stochastic_2016,shalev_shwartz_online_2012}), with illustrative examples outlined in \cite{mairal_incremental_2015,mensch_stochastic_2018,nguyen_online_2023,chouzenoux_stochastic_2017,lyu2024stochastic,lyu_online_2022,le_thi_online_2024,le_thi_online_2020,le_thi_dca_2021,fort_sequential_2024,dieuleveut_federated_2025}. These methods have similarly proven valuable for addressing the challenges of optimizing complex models in incremental data environments.

\subsection{Contributions and practical implications}\label{sec_intro_overall_contribution}
In this paper, we revisit the Stochastic Approximation formulation of the incremental stochastic MM algorithm introduced in~\cite{nguyen_online_2023} (see also Algorithm~1 in~\cite{dieuleveut_federated_2025}), developed within the framework of~\cite{cappe_Online_2009}, and we provide an incremental stable MM scheme for modern MoE models. This approach adheres more closely to the fundamental principles of the original batch-mode EM algorithm, ensuring consistency with its theoretical and methodological foundations. Our approach imposes regularity assumptions concerning the surrogates and independent and identically distributed assumptions on the examples (see \cref{assumptionA1} to \Cref{assumptionA4}). Similar to recent contributions on stochastic MM (see e.g. \cite{fort_stochastic_2020,dieuleveut_federated_2025}), the incremental stochastic MM algorithm  \cite{nguyen_online_2023,dieuleveut_federated_2025} accommodates general surrogate functions and therefore encompasses the incremental stochastic EM algorithm outlined in \cite{cappe_Online_2009,cappe_online_2011} which defines surrogate functions by using the latent variable stochastic representations. This feature is particularly advantageous when developing estimation algorithms for complex latent models such as MoE models (see, \eg~\cite{nguyen_laplace_2016}) or optimization problems beyond the scope of inference in latent data models (see examples in Section 2.3 of \cite{dieuleveut_federated_2025}).

Additionally, the proposed approach is supported by theoretical guarantees about consistency. In particular, we provide sufficient conditions such that the algorithm yields a consistent and efficient estimator in the sense that it generates a sequence which, when stable with probability one, converges to a stationary point, where the gradient of the objective function vanishes; see \cref{sec_convergence_issues} for further details. Such a theoretical analysis is notably absent in other widely-used stochastic online or incremental stochastic optimization algorithms, such as stochastic approximation or stochastic gradient descent (SGD)~\cite{robbins_stochastic_1951}, root mean square propagation (RMSProp)~\cite{hinton_neural_2012}, adaptive moment estimation (Adam)~\cite{kingma_adam_2015,dereich_ode_2025}, and second-order clipped stochastic optimization (Sophia) \cite{liu_sophia_2024}, which are commonly applied to MoE models.

To demonstrate the practicality and effectiveness of our proposed method, we apply it to the \md regression problem in \cref{sec_application_SGMoE,sec_experimental_studies}. This setting is particularly challenging, as existing variants of incremental stochastic EM and MM algorithms, such as those proposed in~\cite{cappe_Online_2009,cappe_online_2011,nguyen_mini_batch_2020,karimi_global_2019}, fail to be applicable; see \cref{sec_why_fails_MM_EM} for a detailed discussion. Notably, our proposed incremental stochastic MM algorithms outperform state-of-the-art optimization methods, including SGD, RMSProp, Adam, and Sophia \cite{liu_sophia_2024}, in the task of fitting \md~models, as evidenced in \cref{sec_experimental_studies}. This is a significant result, especially considering the growing prominence of MoE models in modern statistics and machine learning~\citep{dai_deepseekmoe_2024,zhang_moefication_2022,chen_sparse_2023,fedus_switch_2022,do_hyperrouter_2023,pham_competesmoeeffective_2024,guo_deepseek_r1_2025,jiang_mixtral_2024} for heterogeneous data analysis. MoE-based approaches have found widespread applicability across numerous domains, including robotics~\cite{jaquier2021tensor}, speech recognition~\cite{you_speechmoe2_2022}, natural language processing~\cite{do_hyperrouter_2023,pham_competesmoeeffective_2024,puigcerver_sparse_2024}, computer vision~\cite{lathuiliere_deep_2017}, medicine~\cite{li_drug_2019,boux2021}, planetary remote sensing~\cite{forbes_summary_2022,kugler2022,deleforge2015hyper}, bioinformatics~\cite{blein_nicolas_nonlinear_2024,nguyen_joint_2024}, and the physical sciences~\cite{kuusela_semi_supervised_2012}.

Last but not least, in this work, we revisit and correct the bound presented in Lemma 2 of~\cite{bohning1992multinomial}. We provide a more explicit and rigorous formulation of the bound and leverage it to construct a valid majorizer for the softmax-gated MoE optimization problem. This corrected bound may be of independent interest to researchers developing or applying MM algorithms in the context of multinomial logistic regression.

{\bf Notation.}
Throughout the paper, the set $\{1, 2, \ldots, N\}$ is abbreviated as $[N]$ for any positive integer $N \in \nset$; $\bm{1}_N$ and $\mI_{N}$ denote the $N$-dimensional vector of ones and $N\times N$ identity matrix, respectively. In accordance with standard conventions, $\mA^\top$ represents the transpose of a given vector or matrix $\mA$, and all vectors $\vv = (\evv_1,\ldots,\evv_N)^\top \equiv [\evv_1,\ldots,\evv_N]^\top$ are expressed as column vectors. The Loewner order between two matrices is denoted by $\mA \succeq \mB$, implying that the matrix $\mA - \mB$ is positive semi-definite (PSD). Additionally, the symbol $\mA \otimes \mB$ refers to the Kronecker product (also known as the matrix direct product) of the matrices $\mA$ and $\mB$.
Let $\pscal{\cdot}{\cdot}$ denote the Euclidean scalar product. 
Given any differentiable function $f = (f_1,\ldots,f_M)^\top : \rset^T \to \rset^M$, we denote by
$\nabla_{\vtheta} f 
= \left( \frac{\partial f}{\partial \vtheta_1}, \ldots, \frac{\partial f}{\partial \vtheta_T} \right)^\top$
the gradient of $f$ when $M=1$, and by $\nabla_{\vtheta} f^\top$ the $T \times M$ matrix whose columns are the gradients, that is,
$\nabla_{\vtheta} f^\top = \big(\nabla_{\vtheta} f_1, \ldots, \nabla_{\vtheta} f_M\big).$
When the variable of differentiation is clear from the context, we may omit the subscript and simply write $\nabla f$ instead of $\nabla_{\vtheta} f$.
Accordingly, the symbol $\nabla_\vtheta f^\top$ is to be understood as either a vector or a matrix, depending on whether the function $f$ is scalar or vector-valued. In the latter case, the usual Jacobian matrix is the transpose of $\nabla_\vtheta f^\top$.
When $f$ is twice differentiable, we use the notation $\nabla^2_\vtheta f$ as the Hessian matrix which is a $T\times T$ matrix with components given by $[\nabla^2_\vtheta f]_{i,j} =\partial^2 f/\partial \vtheta_i\partial \vtheta_j, i,j \in[T]$.
For any $\vr \in \rset^D$ and any set $\sS \subset \rset^D$, we define
$d(\vr,\sS) \coloneqq \inf_{\vs \in \sS} \|\vr-\vs\|,$
that is, the distance from $\vs$ to $\sS$, where $\|\cdot\|$ denotes the $L_2$–norm on $\rset^D$.
Finally, for any $d \in \mathbb{N}$ and any vector $\vx = (x_1, x_2, \dots, x_D)^\top \in \rset^D$, we define the element-wise $d$-th power of $\vx$ by $\vx^d \coloneqq (x_1^d, x_2^d, \dots, x_D^d)^\top.$

{\bf Paper organization.}
The remainder of this paper is organized as follows. In \cref{sec_online_MM_problem}, we provide an overview of the incremental stochastic MM algorithm, setting the foundation for subsequent discussions. In \cref{sec_convergence_issues}, we provide sufficient conditions for the consistency
of the incremental stochastic MM algorithm, which are further substantiated through empirical evaluations. Applications to \md models are presented in \cref{sec_application_SGMoE}, followed by an experimental study in \cref{sec_experimental_studies} to validate the theoretical findings. Then, we conclude the article and outline directions for future research in \cref{sec_conclusion_future}. Finally, all proofs of theoretical results are deferred to the supplementary material in \cref{sec_proofs_main_results,sec_technical_proof,sec_technical_results}.

\section{The incremental stochastic MM algorithm}\label{sec_online_MM_problem}

\subsection{Primitive optimization problem}\label{sec_original_optimization}

We begin by examining the following primitive optimization problem in a statistical framework:
\begin{equation}
\underset{\vtheta \in \sT}{\arg\min}~ \Ebd{\rvz \sim \PP}{f\left(\vtheta; \rvz \right)} 
\equiv\underset{\vtheta \in \sT}{\arg\min}~ \PE\left[f\left(\vtheta; \rvz \right)\right].\label{eq_main_problem}
\end{equation}
In this context, the function $f: \sT \times \sZ \rightarrow \rset$ is measurable and represents a family indexed by the parameter $\vtheta \in \sT$. The set $\sT$ is defined as a measurable open subset of $\rset^T$, and $\sZ \subseteq \rset^K$ is a Euclidean subspace endowed with its Borel sigma-algebra. Additionally, $\rvz$ is an $\sZ$-valued random variable on the probability space $(\Omega, \mathcal{A}, \sP)$, with its probability density function on $\sZ$ denoted by $\pi$. The symbols $\PP$ and $\E_\pi$ are used to represent the probability measure and the expected value, respectively, under $\pi$.

In this paper, we focus on the scenario where the expectation $\PE\left[|f\left(\vtheta; \rvz \right)|\right]$ is assumed to be finite on $\mathbb{T}$ but cannot be expressed in closed form, and the corresponding optimization problem is addressed using an MM-based algorithm.  According to the terminology introduced in \cite{lange2016mm}, the function $g:\left(\vtheta, \vz, \vtau\right) \mapsto g\left(\vtheta, \vz; \vtau\right)$, defined over the domain $\sT \times \sZ \times \sT$, is referred to as a {\em majorizer of $f$} if, for any $\vtau \in \sT$ and any $(\vtheta, \vz) \in \sT \times \sZ$, the following conditions are satisfied:
\begin{align}\label{eq_minorizer_fct}
f(\vtheta;\vz) - f(\vtau;\vz) \leq g(\vtheta,\vz; \vtau) - g(\vtau, \vz; \vtau). 
\end{align}
In this paper, we study the scenario in which the majorizer function $g$ satisfies the following conditions:
\begin{assumptionP}[Exponential family majorizer surrogate]\label{assumptionA1} The majorizer surrogate $g$ belongs to an exponential family:
\begin{equation} \label{eq_exponential_family}
g\left(\vtheta,\vz;\vtau\right)\eqdef - \psi\left(\vtheta\right)+ \pscal{\bars(\vtau;\vz)}{\phib(\vtheta)},
\end{equation}
where $\psi:\sT \to \rset$, $\phib: \sT \to \rset^D$ and $\bars: \sT \times \sZ \to \rset^D$ are measurable functions. In addition, $\phib$ and $\psi$ are continuously differentiable on $\sT$.
\end{assumptionP}
\begin{assumptionP}[Convex set]\label{assumptionA2} There exists a measurable, open, and convex set $\sS \subseteq \rset^D$ such that
\[
  \bars(\vtau; \vz) \in \sS \quad \text{for all } (\vtau, \vz) \in \sT \times \sZ.
\] 
\end{assumptionP}
\begin{assumptionP}[Incrementally independent and identically distributed data with finite expectation]\label{assumptionA3}
The expectation $\PE[\bars(\vtheta; \rvz)]$ exists, lies within $\sS$, for any $\vtheta \in \sT$, but it not have a closed-form expression. 

Nevertheless, an incremental sequence of independent examples $\{\rvz_n, n \geq 0\}$, distributed under $\pi$, is available. We assume that $\{\rvz_n, n \geq 0\}$ are random variables on $(\Omega, \mathcal{A}, \mathbb P)$.
\end{assumptionP}
    \begin{inequality} \label{eq: min inq}
        \PE\left[f(\vtheta; \rvz)\right] - \PE\left[f(\vtau; \rvz)\right]   
        &\le \psi(\vtau) - \psi(\vtheta) + \pscal{\PE\left[\bars(\vtau; \rvz)\right]}{\phib(\vtheta) - \phib(\vtau)},  
    \end{inequality}  
    thereby we establish a valid majorizer function for the variation of the objective function $\vtheta \mapsto \PE\left[f(\vtheta; \rvz)\right] - \PE\left[f(\vtau; \rvz)\right]$ as follows:
    $G(\cdot, \vtau) - G(\vtau, \vtau)$ where $G(\cdot, \vtau) :=  -\psi(\cdot) + \pscal{\PE\left[\bars(\vtau; \rvz)\right]}{\phib(\cdot)}$.

We also assume that the majorizing function $\vtheta \mapsto g(\vtheta,\vz; \tau)$ possesses a unique minimizer. 
\begin{assumptionP}[Unique global minimum]\label{assumptionA4}
For any $\vs \in \sS$, there exists a unique root in $\mathbb{T}$ to the function $\vtheta \mapsto - \nabla \psi(\vtheta) + \nabla \phib(\vtheta)^\top \vs$. This root corresponds to the unique global minimum of the function $\vtheta \mapsto h(\vs; \vtheta) \eqdef - \psi(\vtheta) + \pscal{\vs}{\phib(\vtheta)}$ over $\sT$. The root is denoted by $\barparam(\vs)$, where  
\begin{align}\label{eq_global_maximum}
    \barparam(\vs) \eqdef \argmin_{\vtheta \in \sT}[- \psi(\vtheta) + \pscal{\vs}{\phib(\vtheta)}], \quad - \nabla \psi(\barparam(\vs)) + \nabla \phib(\barparam(\vs))^\top \vs = \zero.
\end{align}
\end{assumptionP}
A regularity assumption is also assumed on the objective function.

\begin{assumptionP}\label{assumptionObjectiveC1}
    The function $\vtheta \mapsto \PE\left[f(\vtheta;\rvz)\right]$ is continuously differentiable on $\sT$.
\end{assumptionP}

\subsection{Incremental MM algorithm}\label{sec_exponential_minorizer_surrogate}
Given that the expectation in \cref{eq_main_problem} may lack a closed form, but infinitely large datasets $\{\rvz_n, n \geq 0\}$ are available (as stated in \cref{assumptionA3}), we consider the {\tt Incremental (Online) MM} algorithm introduced in \cite{nguyen_online_2023} (see also \cite{dieuleveut_federated_2025}, which is more general as it does not assume a specific form of the oracle of $\PE\left[\bars(\vtheta_n; \rvz)\right]$) and is described in \cref{algorithm_onlineMM}.

This algorithm defines a sequence $\left\{\vs_n, n \geq 0\right\},$ where the update mechanism, detailed in \cref{eq_s_step}, is a Stochastic Approximation iteration. The approach involves constructing a sequence of majorizer functions by defining their parameter $\vs_n$ within the functional space spanned by $-\psi, \phi_1, \ldots, \phi_D.$
\begin{algorithm}[H]
\caption{Incremental MM algorithm}
\label{algorithm_onlineMM}
\begin{algorithmic}[1]
    \Require Incremental data points $\{\vz_{n}:n \ge 0\}$, an initial value $\vs_0 \in \sS$ and positive step sizes $\{\gamma_{n+1}, n \geq 0 \}$ in  $\ooint{0,1}$.
    
    \Ensure A $\sT$-valued sequence $\{\param_n, n \ge 1\}$ and a $\sS$-valued sequence $\{\vs_n, n \ge 0\}$.

    \For{$n = 0,1,\cdots,$}
        \State Compute 
        \begin{align}
        \vs_{n+1} & =\vs_{n}+\gamma_{n+1}\left\{ \bars\left(\barparam(\vs_{n});\vz_{n+1}\right)-\vs_{n}\right\},     \label{eq_s_step}\\
        \vtheta_{n+1} &= \barparam(\vs_{n+1}) \eqdef \argmin_{\vtheta \in \sT}[- \psi(\vtheta) + \pscal{\vs_{n+1}}{\phib(\vtheta)}].\label{eq_argmin_step}
        \end{align}
    \EndFor
\end{algorithmic}
\end{algorithm}

Under \cref{assumptionA2}, it is readily seen that since $\vs_0 \in \sS$ and $\gamma_n \in \ooint{0,1}$
\begin{equation}\label{eq:sninS}
 \sP\left( \forall n \geq 0, \vs_n \in \sS\right) =1;   
\end{equation} with probability one, the algorithm generates an $\sS$-valued sequence.

\section{Convergence properties of the incremental stochastic MM algorithm}\label{sec_convergence_issues}
In this section, we investigate the convergence properties of \cref{algorithm_onlineMM}. Throughout this section, it is assumed that the MM framework described by \Cref{eq_minorizer_fct} and \Cref{assumptionA1} to \Cref{assumptionObjectiveC1} holds.

\subsection{Limiting points of \texorpdfstring{\cref{algorithm_onlineMM}}{Algorithm~\ref*{algorithm_onlineMM}} and stationary points of the original optimization problem}
\cref{eq_s_step} in \cref{algorithm_onlineMM} is a Stochastic Approximation algorithm~\cite{robbins_stochastic_1951}. Such iterative schemes are designed to solve root-finding (fixed-point) problems, that is, to find $\vs^0$ such that
$\PE\big[\bars(\barparam(\vs^0); \rvz)\big] = \vs^0 .$
Thus, \cref{algorithm_onlineMM} is formulated to approximate the intractable expectation at $\barparam(\vs^0),$ where $\vs^0$ adheres to this fixed-point condition.

\cref{proposition_stationary} was initially proposed and proved as Lemma 1 in \cite{nguyen_online_2023}. For the sake of completeness and to maintain consistency with the notations used in our paper, we have also included the proof of \cref{proposition_stationary} in \cref{proof_proposition_stationary}. 
\cref{proposition_stationary} elucidates the connection between the limiting points of \cref{eq_s_step} and the optimization problem defined in \cref{eq_main_problem}. Notably, it demonstrates that any limiting value $\vs^0$ corresponds to a stationary point $\vtheta^{0} \eqdef \barparam(\vs^0)$ of the objective function $\PE\left[f\left(\vtheta; \rvz\right)\right],$ indicating that $\vtheta^{0}$ is a root of its derivative. The proof leverages the methodology presented in Proposition 1 of~\cite{cappe_Online_2009}, which has also been used in Lemma 2 of~\cite{delyon_convergence_1999}, Proposition 1 of~\cite{fort_fast_2021}, Proposition 2.1 of~\cite{fort_stochastic_2023}, and the contemporaneous Proposition 1 of~\cite{dieuleveut_federated_2025}, to establish convergence rates for stochastic EM-type algorithms, including the incremental stochastic EM algorithm studied here. We commence by defining the following notations:
\begin{align}\label{eq_def_hsGamma}
    \veta(\vs) \eqdef \PE\left[ \bars\left(\barparam(\vs);\rvz\right)\right] -\vs\text{,} \qquad  \sF \eqdef \{\vs \in \sS:  \veta(\vs) =\zero \}.
\end{align}
\begin{proposition}
\label{proposition_stationary}
Let $\sL$ denote the set of stationary points of this function, defined as
\[
  \sL := \{\vtheta \in \sT : \nabla_\vtheta \PE\left[f(\vtheta;\rvz)\right] = \zero\}.
\]
If $\vs^0 \in \sF$, then $\barparam(\vs^0) \in \sL$. Conversely, if $\vtheta^{0} \in \sL$, it follows that
$\vs^0 \eqdef \PE\left[\bars\left(\vtheta^{0};\rvz\right)\right] \in \sF$.
\end{proposition}

{\bf Additional regularity assumptions for Lyapunov function.}
Leveraging the results from \cite{delyon_convergence_1999} on the asymptotic convergence of stochastic approximation algorithms, along with the additional regularity conditions specified in \cref{assumptionA5,assumptionA6,assumptionA7} for $\psi$, $\phib$, and $\barparam$, \cref{prop_Lyapunov_function} (whose proof is detailed in \cref{proof_prop_Lyapunov_function}) establishes that the algorithm described in \cref{eq_s_step} admits a continuously differentiable Lyapunov function $\lyap$, defined on $\sS$ as  
$$\vs \mapsto V(\vs) \eqdef \PE\left[f(\barparam(\vs); \rvz) \right], $$  
which satisfies $\pscal{\nabla\lyap(\vs)}{\veta(\vs)} \leq 0$, with strict inequality holding outside the set $\sF$ (refer to \citep[Prop. 2]{cappe_Online_2009}).  Furthermore, the additional regularity assumptions detailed below must also be satisfied, in addition to \cref{assumptionA1,assumptionA2,assumptionA3,assumptionA4,assumptionObjectiveC1}.

\begin{assumption}[Twice continuously differentiable exponential family surrogate]\label{assumptionA5}
The parameter space $\sT$ is defined as a convex and open subset of $\rset^T$, with the functions $\psi$ and $\phib$ in \cref{eq_exponential_family} assumed to be twice continuously differentiable on $\sT$.
\end{assumption}
\begin{assumption}[Continuously differentiable global minimum function]\label{assumptionA6}
The global minimizer function $\vs \mapsto \barparam(\vs)$, defined in \cref{eq_global_maximum}, is continuously differentiable on $\sS$. 
\end{assumption}
\begin{proposition}\label{prop_Lyapunov_function}
Under \Cref{assumptionA5} and \Cref{assumptionA6}, the following holds:
\begin{enumerate}
    \item[(a)]  The function $V$ is continuously differentiable on $\sS$.

    \item[(b)] For any $\vs \in \sS$, we have $\langle \nabla\lyap(\vs), \veta(\vs)\rangle \le 0$, and $
        \{\vs \in \sS : \langle \nabla\lyap(\vs), \veta(\vs)\rangle = 0\} \supseteq \sF$.
\end{enumerate}
\end{proposition}

\subsection{Consistency of the incremental stochastic MM algorithm}
We now establish the almost-sure convergence of the sequence $\{\vs_n : n \ge 0\}$ stated in \cref{theorem_consistency}. The detailed proof is provided in \cref{proof_proposition_consistency}, where we verify the sufficient conditions for convergence of a general stochastic approximation algorithm as given in~\cite{andrieu_stability_2005}. Moreover, the sequence $\{\barparam(\vs_n), n \geq 0\}$ is also proven to converge almost surely to the set $\sL$, which corresponds to the stationary points of the objective function $\vtheta \mapsto \PE\left[f(\vtheta; \rvz)\right]$.

To this end, we require the following assumptions on moment conditions on the statistics $\bars$, on the model at hand, on the stability of the algorithm, and on the step sizes:
\begin{assumption}[Second finite moment]\label{assumptionA7}
For all compact subsets $\Kset \subset \sS$,
    \begin{align*}
        \sup_{\vs \in \Kset}\, \PE\left[ \|\bars\left(\barparam(\vs);\rvz\right)\|^2\right] < \infty.
    \end{align*}
\end{assumption}
\begin{assumption}\label{assumptionA9} The function $\rvs \mapsto \veta(\vs)$ is continuous on $\sS$. In addition, $
        \{\vs \in \sS : \langle \nabla\lyap(\vs), \veta(\vs)\rangle = 0\} \subseteq \sF$.
\end{assumption}
Combined with \Cref{prop_Lyapunov_function}, \Cref{assumptionA9} implies that the two sets are equal and are closed, since the set $\sF$ collects the zeros of a continuous function.  
It is established in the proof of \cref{prop_Lyapunov_function} (see \cref{proof_prop_Lyapunov_function})  that 
\begin{align*}
  \langle \nabla \lyap(\vs), \veta(\vs)\rangle
  = -\Big( \{(\nabla \phib)(\barparam(\vs)) \}^\top \veta(\vs)\Big)^\top
     \mathsf{A}(\vs)
     \Big( \{(\nabla \phib)(\barparam(\vs))\}^\top \veta(\vs)\Big),
\end{align*}
where $\mathsf{A}(\vs)$ is a positive definite matrix under \cref{assumptionA4}. Consequently,
\(\langle \nabla \lyap(\vs), \veta(\vs)\rangle = 0\) iff \(\{(\nabla \phib)(\barparam(\vs))\}^\top \veta(\vs)=\zero\), so that if 
$\{(\nabla \phib)(\barparam(\vs))\}$ has full row rank, then  \(\veta(\vs)=\zero\). Thus, this rank condition is a sufficient condition for the two sets  $\{\vs: \langle \nabla \lyap(\vs), \veta(\vs)\rangle = 0 \}$ and $\sF$  to be equal. 
\begin{assumption}\label{assumptionA8}
The closure of $V(\sF)$ has an empty interior. 
\end{assumption}
\begin{assumption}[Sublevel sets of the Lyapunov function]\label{assumption_compact_sublevel}
There exist $M_0<M_1$ such that $\sF \subset \{\vs \in \sS: \lyap(\vs) < M_0 \}$ and $\{\vs \in \sS: \lyap(\vs) \leq M_1 \}$ is compact. 
\end{assumption}
\begin{assumption}[Stability assumption]\label{assumptionA10}
    With probability one, there exists a compact set $\mathbb K$ of $\sS$ such that $\mathbb{K} \cap \sF \neq \emptyset$ and $\vs_n \in \mathbb K$ for all $n \geq 0$.
\end{assumption}
\Cref{assumptionA10} states that, with probability one, $\limsup_n\|\vs_n\| < \infty$ and $\liminf_n d(\vs_n,\mathbb S^c) >0$.
It imposes a stability assumption, which is not straightforward. In general, the stability of the algorithm can be guaranteed by truncating the algorithm updates, either within a fixed set, as outlined in \cite[Chapter 2]{kushner_stochastic_2003}, or within a progressively expanding sequence of sets, as demonstrated in \cite{andrieu_stability_2005}. To maintain conciseness in the exposition, we refrain from explicitly performing these constructions, which also impact the convergence analysis, and leave them as directions for future work.

\begin{assumption}\label{assumptionA11} The step size sequence is a monotone nonincreasing sequence,  $\gamma_n \in \ooint{0,1}$, 
$\sum_n \gamma_n = + \infty$ and $\sum_n \gamma_n^2 < \infty$.
\end{assumption}

\cref{assumptionA11}  represents a conventional requirement for stochastic approximation methods employing vanishing step sizes, as highlighted in \cite{kushner_stochastic_2003}. This condition is satisfied, for instance, by selecting the step size $\gamma_n = \gamma_0n^{-\alpha}$, with $\alpha$ belonging to the interval $(\frac{1}{2}, 1]$. The supplementary conditions that $\gamma_n \in (0,1)$ and that $\vs_0 \in \sS$, together with \cref{assumptionA2}, ensure that the entire sequence $\{\vs_n\}_{n \ge 0}$ remains in $\sS$ for all $n$.

\begin{theorem}[Consistency]\label{theorem_consistency}   
   Let $\mathbb{F}$ and $\mathbb{L}$ be given by \cref{eq_def_hsGamma} and \cref{proposition_stationary}. Under \Cref{assumptionA5} to \Cref{assumptionA11}, with probability $1$, 
    \begin{align}
        \limsup_{n\rightarrow \infty} d(\vs_n,\mathbb{K} \cap \sF) = 0 \text{ and } \limsup_{n \rightarrow\infty} d(\param_n,\sL) = 0.
    \end{align}
\end{theorem}

\section{An application to the \md models}\label{sec_application_SGMoE}
To demonstrate the effectiveness of the proposed approach, we consider a regression model that, as outlined in \cref{sec_original_optimization}, represents a scenario where the expectation $\PE\left[f\left(\vtheta; \rvz \right)\right]$ in \cref{eq_main_problem} can not be explicitly computed in closed form. The associated optimization problem is therefore tackled using the MM-framework-based \cref{algorithm_onlineMM} in \cref{sec_exponential_minorizer_surrogate}.

{\bf \Md models for heterogeneous data.}
We refer to the outputs $\rvy \in \sY$, as the target or response variables and the inputs $\mathbf{\rvx} \in \sX \subset \sR^P,~P \in \nset$,
as the fixed explanatory or predictor variables. We consider an incremental sequence $\{(\rvx_n,\rvy_n): n\ge1\}$ of i.i.d.\ copies of $(\rvx,\rvy)=:\rvz$, which follows a distribution characterized by the true but unknown probability density function $\pi(\cdot \mid \rvx=\vx)$. Here, the corresponding observed values are denoted by $\left(\vx, \vy\right)=:\vz$. Motivated by universal approximation theorems for MoE models, the function $\pi$ can be estimated using a \md model, defined as:
\begin{align}
\label{eq_def_SGMoE}
	s_{\vtheta}(\vy \vert \vx) = \sum_{k=1}^{K}\sfg_k (\vw(\vx)) \ex_k(\vy;\vv(\vx)).
\end{align}
The softmax gating networks are defined as follows:
\begin{align*}
    \sfg_k(\vw(\vx)) = \frac{\exp\left(w_k(\vx)\right)}{\sum_{l=1}^{K} \exp\left(w_l(\vx)\right)}, \quad \forall k \in [K].
\end{align*}
Here, $\vw(\vx) = \left(w_1(\vx), \dots, w_K(\vx)\right)$ represents the weight functions, and the expert networks $\ex_k(\vy;\vv(\vx))$, for all $k \in [K]$, are specified later in \cref{sec_SGMoE,sec_SMLMoE} for continuous and discrete output variables respectively, and $K$ is the number of mixture of expert components. In \cref{eq_def_SGMoE}, $\vtheta$ collects all \emph{finite-dimensional} parameters that specify the gating and expert networks, i.e.,
the weight functions $\{w_k(\cdot)\}_{k\in[K]}$ and the expert predictors $\{\vv(\cdot)\}$ (and hence $\{\ex_k(\cdot;\vv(\cdot))\}$)
are assumed to belong to a prescribed parametric family indexed by $\vtheta$.

{\bf From batch MM to incremental stochastic MM for \md models.} 
Since the true conditional density $\pi(\cdot\mid \rvx=\vx)$ is typically unknown, we formulate parameter learning for the \md model in \cref{eq_def_SGMoE} as an expected-risk minimization problem of the form \cref{eq_main_problem}. We consider
\begin{align}
    \vtheta^{0} \in \arg\min_{\vtheta\in\sT}\; \E\!\left[f(\vtheta;\rvz)\right],
    \qquad
    f(\vtheta;\rvz)\coloneqq \ell\!\big(\rvy, s_{\vtheta}(\cdot\mid \rvx)\big),
    \label{eq_moe_expected_risk}
\end{align}
where $s_{\vtheta}(\vy\mid \vx)=\sum_{k=1}^K \sfg_k(\vw(\vx))\,\ex_k(\vy;\vv(\vx))$ is defined in \cref{eq_def_SGMoE}, and
$\ell(\cdot,\cdot)$ is a chosen loss (e.g., a proper scoring rule). Here, we consider the negative log predictive density, i.e., 
\begin{align}
    \label{multi_D_lossfunc}
  \ell\!\big(\rvy, s_{\vtheta}(\cdot\mid \rvx)\big) =-\log s_{\vtheta}(\rvy\mid \rvx).
\end{align}
The expectation in \cref{eq_moe_expected_risk} is taken with respect to the joint law of $\rvz=(\rvx,\rvy)$:
equivalently, $\E[f(\vtheta;\rvz)]=\E_{\rvx}\!\left[\E_{\rvy\mid \rvx}\!\left[-\log s_{\vtheta}(\rvy\mid \rvx)\mid \rvx\right]\right]$,
so $\vtheta^{0}$ is a single global parameter (it does not depend on $\vx$). In general, the expectation $\E[f(\vtheta;\rvz)]$ is not available in closed form. A direct batch strategy would replace $\E[f(\vtheta;\rvz)]$ by an empirical average based on
$\{\rvz_n\}_{1\le n\le N}$ and then apply an MM procedure to this finite-sum objective. 
However, such a batch approach implicitly assumes the availability of the entire dataset
$\{\rvz_n\}_{1\le n\le N}$ for each $N$, which is infeasible in practice when observations are accessed as a stream and when memory constraints make storing the full sequence impossible for large $N$.
Accordingly, we work under \cref{assumptionA3}, where an incremental stream of i.i.d.\ observations
$\{\rvz_n=(\rvx_n,\rvy_n)\}_{n\ge 0}$ is available.
We therefore apply the incremental (online) MM scheme in \cref{algorithm_onlineMM}. 
This algorithm maintains an auxiliary statistic sequence $\{\vs_n\}_{n\ge 0}$ updated by the stochastic approximation recursion \cref{eq_s_step}, which constructs data-driven majorizer surrogates within the functional class spanned by $-\psi,\phi_1,\ldots,\phi_D$ (see \cref{assumptionA1}). 
The parameter iterate is then obtained via the deterministic map $\vtheta_{n+1}=\barparam(\vs_{n+1})$ in \cref{eq_argmin_step}.

\subsection{\Mdg (SGMoE) models for continuous output data} \label{sec_SGMoE}
In this case, the response variable $\rvy$ is continuous, \ie~$\rvy\in \sY \subset \sR^Q,~ Q \in \nset$, and the Gaussian expert, $\ex_k(\vy;\vv(\vx)) := \cN(\vy, \vmu_k(\vx), \mSigma_k),$ is parameterized by the mean function $ \vmu_k(\vx)$ and the covariance matrix $\mSigma_k$, \ie
\begin{align*}
    \ex_k(\vy;\vv(\vx)) :=\cN(\vy;\vmu_k(\vx),\mSigma_k) = \frac{\exp \left(-\frac{1}{2}(\vy - \vmu_k(\vx))^\top \mSigma_k^{-1} (\vy - \vmu_k(\vx))\right)}{\sqrt{(2\pi)^Q \lvert \mSigma_k \rvert}}.
\end{align*}
For simplicity, we assume that $\mSigma_k$ is diagonal, with diagonal entries $\sigma_{k,q}^2$ for $q\in[Q]$. We leave the extension to a general covariance matrix for future work.

\subsubsection{Motivation for polynomial regression}\label{section_polynomial_regression}

To address the heterogeneous regression problem, certain authors have employed softmax-gated Gaussian MoE models under specific simplifying assumptions. Notably, \cite{chamroukhi2019regularizedIJCNN,chamroukhi2019regularized} explored Gaussian MoE models for multiple regression with univariate output variables, where both the weights and means are modeled as linear functions of the input variables. However, this linear assumption restricts the capacity of MoE models. Indeed, in the context of convolutional neural networks, \cite{chen_towards_2022} empirically demonstrated that while mixtures of linear experts outperform single-expert models, they fall significantly short when compared to mixtures of non-linear experts.
Motivated by these findings, we aim to enhance the flexibility of MoE models by incorporating nonlinearities. Specifically, we define the gating weights and expert means as polynomials of the input variables. Regarding the convergence properties of such polynomial-based MoE models, we refer to \cite{mendes2012convergence}, which provides insights into the optimal convergence rates of MoE models where each expert employs a polynomial regression framework.
Motivated by regression with polynomial features, both the gating weights $w_k$ and the expert means $\vmu_k$ are modeled as polynomial functions of the input variables $\vx$, as follows: given any $\vomega_{k,d} = (\omega_{k,d,p})_{p\in[P]}\in \sR^P$,
\begin{align}
	w_k(\vx) &:=w_k(\vx;\vomega_k)\!=\! 
    \sum_{d=0}^{D_W}\vomega_{k,d}^\top \vx^d\!=\!\sum_{d=0}^{D_W}\left(\sum_{p=1}^{P}\omega_{k,d,p}\evx_p^d\right),\nonumber\\
	\vmu_k (\vx) &:=\vmu_k (\vx;\mUpsilon_k)\!= \!\sum_{d=0}^{D_V} \!\mUpsilon_{k,d} \vx^d, \text{ with } \mUpsilon_{k,d} \!\in \!\sR^{Q \times P}.\nn
\end{align}
Then, let $\vomega = \left(\vomega_k\right)_{k\in[K]}$, $\vomega_k = \left(\omega_{k,d,p}\right)_{  d\in\{0,\ldots,D_W\},p\in [P]}$ and $\mUpsilon_k = \left(\mUpsilon_{k,d}\right)_{d\in\{0,1,\ldots,D_V\}}$ be the tuples of unknown coefficients with the maximum degrees $D_{W}$ and $D_{V}$ of polynomials for the weight and mean functions, respectively.
Subsequently, the unknown parameters of the model are denoted as follows: $\vtheta = \left(\vomega_k, \mUpsilon_k, \mSigma_k\right)_{k \in [K]}.$

At this point, the global regression function is defined as the conditional expectation of the mixed probability density function in \cref{eq_def_SGMoE}, that is
\begin{align}\label{eq_def_SGaBloME_Expectation}
	\sfm(\vx) = \Ep{\vy \vert \vx} = \sum_{k=1}^{K}\sfg_k (\vw(\vx)) \vmu_k(\vx),
\end{align}

\subsubsection{Identifiability of \mdg models}\label{section_identifiability_SGMoE}
It is worth noting that if we translate $\vomega_{k,d}$ to $\vomega_{k,d}+\vt$, $\vt \in \sR^{P}$, the values of the softmax gating function remain unchanged. This implies that the gating parameters $\vomega_{k,d}$ are identifiable only up to a translation. This issue has also been discussed and addressed in the context of establishing the convergence rate of parameter estimation via MLE for \mdg models in \cite{nguyen_demystifying_2023}.
To alleviate this problem, based on \cite{hennig2000identifiablity,jiang1999identifiability}, we parameterize the gating parameters via the constraints, without loss of generality, $\left\{\omega_{K,d,p}\right\}_{d\in\{0,\ldots,D_W\},p\in [P]} = \zero$ such that
\begin{align*}
	\sfg_K\left(\vx;\vomega\right):= 1 - \sum_{k=1}^{K-1}\sfg_k\left(\vx, \vomega\right),~
	\sfg_k\left(\vx; \vomega \right):=  \frac{\exp\left(w_k(\vx; \vomega_k)\right)}{1+\sum_{l=1}^{K-1}\exp\left(w_l(\vx;\vomega_l)\right)}, ~\forall k \in [K-1].
\end{align*}

\subsubsection{Incremental stochastic MM algorithm for \mdg models}\label{section_onlineMM_SGMoE}
We begin by introducing the following notations:
\begin{itemize}
    \item $\vs_{n,4,k:q}$ is the segment of the vector $\vs_{n,4}$ spanning from position $\left\{[(k-1)Q+q]-1\right\}P(D_{V}+1)+1$ to $[(k-1)Q+q]P(D_{V}+1)$. Similarly, $\vs_{n,5,k:q}$ denotes the segment of the vector $\vs_{n,5}$ ranging from position $\left\{[(k-1)Q+q]-1\right\}[P(D_{V}+1)]^2+1$ to $[(k-1)Q+q][P(D_{V}+1)]^2$.
    
    \item $s_{n,3,kq}$ is the $kq$-th component of the vector $\vs_{n,3}$, while $s_{n,6,kq}$ denotes the $kq$-th component of the vector $\vs_{n,6}$.

    \item $\mat(\vs_{n+1,2})$ is the process of reshaping the vector $\vs_{n+1,2}$ into a matrix with dimensions $(K-1)(D_W+1)P \times (K-1)(D_W+1)P$. Similarly, $\mat(\vs_{n+1,5,k})$ denotes the reshaping of the vector $\vs_{n+1,5,k:q}$ into a square matrix of dimensions $P(D_{V}+1) \times P(D_{V}+1)$.

    \item The posterior probability that the data point $(\vx_n, \vy_n)$ belongs to the $k$-th expert is
   \begin{align}\label{eq_t_iterations}
       \tau_{n,k}
    =\frac{\sfg_k\left(\vx_n; \vomega_n \right) \cN(\vy_n;\vmu_{n,k} (\vx_n;\mUpsilon_{n,k}),\mSigma_{n,k})}{\sum_{l=1}^{K}\sfg_l\left(\vx_n; \vomega_n \right) 
	\cN(\vy_n;\vmu_{n,l} (\vx_n;\mUpsilon_{n,l}),\mSigma_{n,l})},
   \end{align}
    where, $\param_n = (\vomega_{n,k},\mUpsilon_{n,k},\mSigma_{n,k})_{k\in[K]}$ denotes the parameter vector at the $n$-th iteration (time point).

    \item $\vxi_n=[\tau_{n,1}\evx_{n,1}^{0}, \tau_{n,1}\evx_{n,1}^{1},\dots, \tau_{n,1}\evx_{n,1}^{D_W},\dots,\tau_{n,K-1}\evx_{n,1}^{D_W},\dots,\tau_{n,K-1}\evx_{n,P}^{D_W}]^{\top}$.
    
    \item $\vtau_{n,:}=[\tau_{n,1},\tau_{n,2},\dots,\tau_{K,n}]^{\top}$, $\vr_n = \vect\left(\left[\vx_n^0,\vx_n,\vx_n^2,\ldots,\vx_n^{D_V}\right]\right)$.
    
    \item $\mUpsilon_k=\left[\mUpsilon_{k,0},\dots,\mUpsilon_{k,D_V}\right]$, $\mUpsilon_{k,q,:}$ is the vector containing the $q$-th row of $\mUpsilon_k$ for $q\in[Q]$.
  
    \item Given $\hat{\vx}_n=\left[\evx_{n,1}^{0},\dots,\evx_{n,1}^{D_W},\dots,\evx_{n,P}^{0},\dots,\evx_{n,P}^{D_W}\right]^{\top}$, we define the following matrix:
    \begin{align}\label{eq_define_B_n_with_noise}
    \mB_{n,K}=\left(\frac{3}{4}\mI_{K-1}-\frac{\bm{1}_{K-1}\bm{1}_{K-1}^{\top}}{2(K-1)}\right)\otimes\hat{\vx}_n\hat{\vx}_n^{\top}+\varepsilon^*\mI_{(K-1)P(D_W+1)},
   \end{align}
   where $\varepsilon^*$ is a positive real number chosen sufficiently small. 
    
\end{itemize}
With $\mB_{n,K}$ defined in \cref{eq_define_B_n_with_noise}, it can be proven that any eigenvalue of $\mB_{n,K}$ is bounded in the interval $\left(\dfrac{1}{M_0},M_0\right)$ for some positive real number $M_0$. Then we can define $\mathbb{S}$ as follows:
\begin{align}\label{eq_define_S2_main}
    \mathbb{S}=&\big\{(\vs_1, \vect(\mS_2),\vs_3, \vs_4, \vs_5, \vs_6):  \vs_1\in \sR^{(K-1)P(D_W+1)};\nonumber\\
    &\quad \mS_2\in \sR^{(K-1)P(D_W+1)\times (K-1)P(D_W+1)},\nonumber\\ 
    &\quad \mS_2 \succeq \zero, \lambda \in \left(\dfrac{1}{M_0},M_0\right) \text{ for all eigenvalues $\lambda$ of } \mS_2;\nonumber\\
    &\quad \vs_3 \in \sR^{KQ}; \quad \vs_4\in \sR^{KPQ(D_V+1)}; \quad \vs_6\in\sR_{+}^{KQ};\nonumber\\
    &\quad \vs_5\in \sR^{KQ(P(D_V+1))^2}, \mat(\vs_{5,k:q}) \succeq \zero \text{ for all $k\in[K],q\in [Q]$}\big\}.
\end{align}
Through essential inequalities for constructing surrogate functions as outlined in \cref{sec_technical_results}, and following the procedure in \cref{sec_exponential_minorizer_surrogate}, we derive the exponential family majorizer surrogate function for \cref{multi_D_lossfunc}. This is formally presented in \cref{proposition_surrogate_construction_SGMoE} and its proof is provided in \cref{sec_exponential_family_surrogate}.
\begin{proposition}\label{proposition_surrogate_construction_SGMoE} 
The construction of the exponential family majorizer surrogate function for the \mdg models is as follows:
\begin{align} 
	-\log s_{\vtheta}(\vy_n\mid \vx_n) 
    &\le\sum_{k=1}^{K}\tau_{n,k}\log(\tau_{n,k})+\underline{\sfg}(\vomega_n;\vx_n)-\vomega_{n}^{\top}\nabla \underline{\sfg}(\vomega_n;\vx_n)-\sum_{k=1}^{K}\tau_{n,k}w_k(\vx_n)\nn\\
    &\quad+\left\{\vomega^{\top}\nabla \underline{\sfg}(\vomega_n;\vx_n)+\frac{1}{2}(\vomega-\vomega_{n}^{\top})\mB_{n,K}\left(\vomega-\vomega_n\right)\right\}\nn
   \\
    &\quad-\sum_{k=1}^{K}\tau_{n,k}\left[\log(\cN(\vy_n,\vmu_k(\vx_n),\mSigma_k))\right]=:-\log[g(\vtheta, \vx_{n},\vy_{n}; \vtheta_{n})].\label{eq_surrogate_Taylor_SGGMoE} 
\end{align}
Here $\underline{\sfg}(\vomega_n;\vx_n)=\log\left(1+\displaystyle\sum_{k=1}^{K-1}\exp\left(w_{n,k}(\vx_n)\right)\right) \text{ where } w_{n,k}(\vx_n)=\displaystyle\sum_{d=0}^{D_W}\vomega_{n,k,d}^{\top}\vx_n^{d}.$
\end{proposition}

We adapt the incremental stochastic MM algorithm \cref{algorithm_onlineMM} for the \mdg models as in \cref{algorithm_onlineMM_MoE} as follows.

\begin{algorithm}[H]
\caption{Incremental stochastic MM algorithm for \mdg models}
\label{algorithm_onlineMM_MoE}
\begin{algorithmic}[1]
    \Require Incremental data points $\{(\vx_{n},\vy_{n}):n \ge 0\}$, positive step sizes $\{\gamma_{n+1}, n \geq 1 \}$ in  $\ooint{0,1}$ and an initial value $\vs_0 = (\vs_{0,1},\ldots,\vs_{0,6}) \in \sS$, where $\sS$ is defined in \cref{eq_define_S2_main}, such that:
    \begin{align*}
        \mat(\vs_{0,2})&\succeq \zero,~\vs_{0,3}=\bm{1}_{KQ},~\widehat{\vs} \in \R^{P(D_V+1)},~\vs_{0,4}=\textbf{1}_{KQ}\otimes\widehat{\vs},\\
        \vs_{0,5}&=\textbf{1}_{KQ}\otimes \vect(\widehat{\vs}\widehat{\vs}^{\top}+\mI_{P(D_V+1)}),~\vs_{0,6}>0.
    \end{align*}
    
    \Ensure A $\sT$-valued sequence $\{\param_n, n \ge 0\}$ and a $\sS$-valued sequence $\{\vs_n, n \ge 0\}$.

    \For{$n = 0,1,\dots,$}
        \State Compute
            \begin{align*}
            \vs_{n+1,1}&=\vs_{n,1}+\gamma_{n+1}\left(-\vxi_{n+1}+\nabla \underline{\sfg}(\vomega_{n};\vx_n)-\mB_{n+1,K}\vomega_{n}-\vs_{n,1}\right),\\
            \vs_{n+1,2}&=\vs_{n,2}+\gamma_{n+1}\left[\frac{1}{2}\vect\left(\mB_{n+1,K}\right)-\vs_{n,2}\right],\\
            \vs_{n+1,3}&=\vs_{n,3}+\gamma_{n+1}\left(\vtau_{n+1,:}\otimes \vy_{n+1}^{2}-\vs_{n,3}\right),\\  
            \vs_{n+1,4}&=\vs_{n,4}+\gamma_{n+1}\left(-2\vtau_{n+1,:}\otimes(\vy_{n+1}\otimes\vr_{n+1})-\vs_{n,4}\right), \\
            \vs_{n+1,5}&=\vs_{n,5}+\gamma_{n+1}\left(\vtau_{n+1,:}\otimes\left(\vect\left(\vr_{n+1}\vr_{n+1}^{\top}\right)\otimes \mathbf{1}_Q\right)-\vs_{n,5}\right),\\
            \vs_{n+1,6}&=\vs_{n,6}+\gamma_{n+1}\left(\vtau_{n+1,:}\otimes \mathbf{1}_Q-\vs_{n,6}\right), \\
       \vomega_{n+1} &= -\left(2\mat(\vs_{n+1,2}) \right)^{-1} \vs_{n+1,1},\\
        \mUpsilon_{n+1,k,q,:}&=-\left(2\mat\left(\vs_{n+1,5,k:q}\right)\right)^{-1}\vs_{n+1,4,k:q},\\
        \sigma_{n+1,k,q}^{2}&=\frac{s_{n+1,3,kq}+\vs_{n+1,4,k:q}^{\top}\mUpsilon_{k,q,:}^{\top}+\vs_{n+1,5,k:q}^{\top}\vect(\mUpsilon_{k,q,:}^{\top}\mUpsilon_{k,q,:})}{s_{n+1,6,kq}},\\
        \vtheta_{n+1} &= \left(\vomega_{n+1,k,p}, \mUpsilon_{n+1,k,q,:}, \sigma_{n+1,k,q}^{2}\right)_{k \in [K];q\in[Q];p\in [P]}.
        \end{align*}
    \EndFor
\end{algorithmic}
\end{algorithm}
A proof of \cref{theorem_onlineMM_MoE}, which provides the theoretical guarantees for \cref{algorithm_onlineMM_MoE}, is given in \cref{proof_theorem_onlineMM_MoE}. Meanwhile, a detailed derivation of the update procedure in \cref{algorithm_onlineMM_MoE} is presented in \cref{proof_algorithm_onlineMM_MoE}.

\begin{theorem}[Stability and consistency for \cref{algorithm_onlineMM_MoE}]
\label{theorem_onlineMM_MoE}
Consider the softmax-gated Gaussian MoE setting described in \cref{sec_SGMoE}, and assume that the
covariates and responses are bounded.
With the surrogate construction given in \cref{eq_surrogate_Taylor_SGGMoE} and the choice of
$\mB_{n,K}$ in \cref{eq_define_B_n_with_noise}, the model-specific assumptions
\cref{assumptionA1,assumptionA2,assumptionA3,assumptionA4,assumptionObjectiveC1,assumptionA5,assumptionA6,assumptionA7,assumptionA10}
are \emph{verified} (see \cref{proof_assumption_theorem_onlineMM_MoE}).
The remaining assumptions required to invoke the general stochastic-approximation convergence results are
treated as \emph{standing assumptions} for this MoE instantiation, namely
\cref{assumptionA8,assumptionA9,assumption_compact_sublevel}
(Lyapunov-level regularity, stability, and compact sublevel sets for almost-sure convergence).
In addition, the step-size condition \cref{assumptionA11} can be \emph{enforced by design},
for example by taking $\gamma_n=\gamma_0 n^{-\alpha}$ with $\alpha\in\left(\tfrac12,1\right]$
and $\gamma_0\in(0,1)$.
Under the verified assumptions above, together with these standing assumptions,
the conclusions of
\cref{proposition_stationary,prop_Lyapunov_function,theorem_consistency}
apply to \cref{algorithm_onlineMM_MoE}.
\end{theorem}

\subsection{\Mdm models for discrete output data} \label{sec_SMLMoE}
In this situation, the output variable $\ry \in \sY:= [M]$, $M \in \sN$, is a discrete response variable. Furthermore, the multinomial (multiclass) logistic regression expert network is defined via the conditional probability function as follows:
\begin{align*}
    \ex_k(\ry=m;\vv(\vx;\vupsilon)):=\ex_k(\vv_{y_n}(\vx_n;\vupsilon)) = \frac{\exp\left(\evv_{m,k}(\vx)\right)}{\sum_{l=1}^{M} \exp\left(\evv_{l,k}(\vx)\right)},~\forall k \in [K],~\forall m\in[M].
\end{align*}
Here the expert function $\vv=(\vv_m)_{m\in[M]}=(\evv_{m,k})_{m\in[M],k\in[K]}$ is modeled as polynomial function of the input variable $\vx$ as follows:
\begin{align}
	\evv_{m,k}(\vx) &= 
    \sum_{d=0}^{D_V}\vupsilon_{m,k,d}^\top \vx^d=\sum_{d=0}^{D_V}\left(\sum_{p=1}^{P}\upsilon_{m,k,d,p}\evx_p^d\right) \text{ for any $\vupsilon_{m,k,d} = (\evupsilon_{m,k,d,p})_{p\in[P]}\in \sR^P$}.\nn
\end{align}
Then, let $\vomega = \left(\vomega_k\right)_{k\in[K]}$ where $\vomega_k = \left(\omega_{k,d,p}\right)_{  d\in\{0,\ldots,D_W\},p\in [P]}$ and $\vupsilon=(\vupsilon_m)_{m\in[M]}=(\vupsilon_{m,k})_{m\in[M],k\in[K]}$ where $\vupsilon_{m,k} = \left(\evupsilon_{m,k,d,p}\right)_{d\in\{0,1,\ldots,D_V\},p\in [P]}$ be the tuples of unknown coefficients with the maximum degrees $D_{W}$ and $D_{V}$ of polynomials for the softmax-gated and expert functions, respectively. 
Then, we define 
\begin{align*}
	\ex_k(\ry=m;\vv(\vx;\vupsilon)):=\ex_k(\vv_{y_n}(\vx_n;\vupsilon)), ~\forall k \in [K].
\end{align*}
Finally, the unknown parameters of the \mdm models are denoted as follows: $\vtheta = \left(\vomega_k, \vupsilon_{m,k}\right)_{m\in[M],k \in [K]}.$

\subsubsection{Identifiability of \mdm models}
Apart from addressing the identifiability of the parameters from in the \mdg models in \cref{section_identifiability_SGMoE}, we also propose an additional constraint on the parameters, which is also considered in the study of the convergence rate of parameter estimation for \mdm models in \cite{nguyen_general_2024}. Specially, for the multinomial logistic regression expert network, we impose the following constraint: without loss of generality, $\left\{\evupsilon_{M,k,d,p}\right\}_{d\in\{0,\ldots,D_W\},p\in [P]} = \zero$ for all $k\in[K]$ such that
\begin{align*}
	  \ex_k\left(\vv_{M}(\vx)\right) = 1 - \sum_{m=1}^{M-1}\ex_k\left(\vv_{m}(\vx)\right), ~\ex_k\left(\vv_{m}(\vx)\right) =\frac{\exp\left(\evv_{m,k}(\vx)\right)}{1+\sum_{l=1}^{M-1} \exp\left(\evv_{l,k}(\vx)\right)}, ~\forall m \in [M-1].
\end{align*}
\subsubsection{Incremental MM algorithm for \mdm models}\label{section_onlineMM_SGMLMoE}
Let $ \I(z,l)$ be an indicator that maps $[M]\times [M]\rightarrow \{0,1\}$ where $\I(z,l)=1$ if $z=l$ and 0, otherwise. We can choose that function as follows: $ \I(z,l)=[\prod_{q=1}^{M}(z-q)]/\left[(z-l)(z-1)!(M-z)!(-1)^{M-z}\right]$. Then, given $r_{d,l,p,k}=\I(y,l)\tau_{k}\evx_p^d$, we denote
\begin{align*}
    \vr_n&=[r_{n,0,1,1,1},\dots,r_{n,D_V,1,1,1},\dots,r_{n,D_V,M,1,1},\dots,r_{n,D_V,M,P,1},\dots,r_{n,D_V,M,P,K}]^{\top}.
\end{align*}
Given any $t \in \sN, n \in \sN$, and any matrix $\mA \in \sR^{tn\times tn}$, we define ${\bdiag}_t(\mA)\in \sR^{nt\times t}$ as follows:
\begin{align*}
    {\bdiag}_t(\mA)_{[it+1:it+t],[1:t]}=\mA_{[it+1:it+t],[it+1:it+t]},\quad \forall i \in \{0,\dots,n-1\}.
\end{align*}

For any $K \in \sN$ and $M \in \sN$, let $\mB_{n,K}$, $\mB_{n,M}$ and $\vtau_{n,:}$ to have the same definition as the previous \cref{section_onlineMM_SGMoE}.
Now we go into the main part of finding the surrogate function of \cref{multi_D_lossfunc} via \cref{proposition_surrogate_construction_SGMLMoE}, which is proved in \cref{proof_proposition_surrogate_construction_SGMLMoE}.
\begin{proposition}\label{proposition_surrogate_construction_SGMLMoE}
The construction of the exponential family majorizer surrogate construction for \mdm models is as follows:
\begin{align}
    -\log s_{\vtheta}(y_n\mid \vx_n)
    &\le C_{n,k}-\left(\vomega^{\top}\vxi_n\right)-\left(\vupsilon^{\top}\vr_n \right)\nn\\
    &\quad+\Bigg\{\underline{\sfg}(\vomega_n;\vx_n)+ (\vomega-\vomega_n)^{\top}\nabla \underline{\sfg}(\vomega_n;\vx_n)+\frac{1}{2}(\vomega-\vomega_n)^{\top}\mB_{n,K}(\vomega-\vomega_n)\Bigg\} 
   \nn\\
   &\quad +\displaystyle\sum_{k=1}^{K}\tau_{n,k}\Bigg\{\underline{\ex}(\vc_{n,k};\vx_n)+ (\vc_{k}-\vc_{n,k})^{\top}\nabla \underline{\ex}(\vc_{n,k};\vx_n)\nonumber\\
   &\quad +\frac{1}{2}(\vc_{k}-\vc_{n,k})^{\top}\mB_{n,M}(\vc_{k}-\vc_{n,k})\Bigg\}=:-\log[g(\vtheta, \vx_{n}, y_{n}; \vtheta_{n})],\label{eq_surrogate_Taylor_SGMLMoE}
\end{align}
where:
\begin{align*}
    \vc_{n,k}&=[\vupsilon_{n,1,k,0,1},\dots,\vupsilon_{n,1,k,D_V,1},\dots,\vupsilon_{n,M,k,D_V,1},\dots,\vupsilon_{n,M,k,D_V,P}]^{\top},\\
     \underline{\ex}(\vc_{n,k};\vx_n)&=\log\left(1+\sum_{l=1}^{M-1} \exp\left(\evv_{n,l,k}(\vx_n)\right)\right), \quad \underline{\sfg}(\vomega_n;\vx_n)=\log\left(1+\sum_{k=1}^{K-1} \exp\left(\evw_{n,k}(\vx_n)\right)\right),\\
    \tau_{n,k}&=\frac{\sfg_k\left(\vx_n; \vomega_n \right) \ex_k(\vv_{y_n}(\vx_n;\vupsilon_n))}{\sum_{l=1}^{K}\sfg_l\left(\vx_n; \vomega_n \right) \ex_l(\vv_{y_n}(\vx_n;\vupsilon_n))},
\end{align*}
\end{proposition}

\begin{algorithm}[H]
\caption{Incremental MM algorithm for \mdm models}
\label{algorithm_onlineMM_SGMLMoE}
\begin{algorithmic}[1]
    \Require Incremental data points $\{(\vx_{n},\vy_{n}):n \ge 0\}$, positive step sizes $\{\gamma_{n+1}, n \geq 1 \}$ in  $\ooint{0,1}$, and an initial value $\vs_0 = (\vs_{0,1},\ldots,\vs_{0,6}) \in \sS$ such that:
    \begin{align*}
        \mat(\vs_{0,2})\succeq \zero, \quad {\bdiag}_{\iota}^{-1}({\mat}_{\iota}(\vs_{0,4}))\succeq \zero, \text{ where } \iota = MP(D_V+1). 
    \end{align*}
    
    \Ensure A $\sT$-valued sequence $\{\param_n, n \ge 0\}$ and a $\sS$-valued sequence $\{\vs_n, n \ge 0\}$:

    \For{$n = 0,1,\dots,$}
        \State Compute
            \begin{align*}
            \vs_{n+1,1}&=\vs_{n,1}+\gamma_{n+1}\left(-\vs_{n+1}+\nabla \underline{\sfg}(\vw_{n};\vx_n)-\mB_{n+1,K}\vw_{n}-\vs_{n,1}\right),\\
            \vs_{n+1,2}&=\vs_{n,2}+\gamma_{n+1}\left[\frac{1}{2}\vect\left(\mB_{n+1,K}\right)-\vs_{n,2}\right],\\
            \vs_{n+1,3}&=\vs_{n,3}\!+\!\gamma_{n+1}\!\left(\!-\vr_{n+1}\!+\!\vect\!\left(\![\tau_{n+1,k}\!\nabla \!\underline{\ex}(\!\vupsilon_{n};\vx_n\!)-\!\tau_{n+1,k}\!\mB_{\!n+1,M}\vupsilon_{n}\!]_{k=1,\dots,K}\!\right)\!-\!\vs_{n,3}\!\right)\!,\\
            \vs_{n+1,4}&=\vs_{n,4}+\gamma_{n+1}\left[\frac{1}{2}\vect\left(\vtau_{n+1,:}\otimes\mB_{n+1,M}\right)-\vs_{n,4}\right],\\
        \vw_{n+1}& = -\left(\mat(\vs_{n+1,2}) + \mat(\vs_{n+1,2})^{\top}\right)^{-1} \vs_{n+1,1},\\
        \vupsilon_{n+1} &= -\left[{\bdiag}_{\iota}^{-1}({\mat}_{\iota}(\vs_{n+1,4}))+{\bdiag}_{\iota}^{-1}({\mat}_{\iota}(\vs_{n+1,4}))^{\top}\right]^{-1} \vs_{n+1,3},\\
         \vtheta_{n+1} &= \left(\vw_{n+1,k},\vupsilon_{n+1,k}\right)_{k\in[K]}.
        \end{align*}
    \EndFor
\end{algorithmic}
\end{algorithm}

We provide the complete derivation of the stochastic recursion underlying
\cref{algorithm_onlineMM_SGMLMoE} in \cref{proof_algorithm_onlineMM_SGMLMoE}.
The resulting sequence of iterates fits naturally within the general
incremental MM framework.
Consequently, the convergence properties of the algorithm follow from
standard stochastic-approximation arguments, after verifying the
model-dependent assumptions.
This leads to the stability and consistency result stated below, whose proof
is given in \cref{proof_assumptions_theorem_onlineMM_SGMLMoE}.

\begin{theorem}[Stability and consistency for \cref{algorithm_onlineMM_SGMLMoE}]
\label{theorem_onlineMM_SGMLMoE}
Consider the softmax-gated multinomial logistic MoE model introduced in
\cref{sec_SMLMoE}, and assume that both the covariates and the responses are
bounded.
With the surrogate construction specified in
\cref{eq_surrogate_Taylor_SGMLMoE} and the parameter-update rules defined in
\cref{algorithm_onlineMM_SGMLMoE}, the MoE-specific assumptions
\cref{assumptionA1,assumptionA2,assumptionA3,assumptionA4,assumptionObjectiveC1,assumptionA5,assumptionA6,assumptionA7,assumptionA10}
are verified (see \cref{proof_assumptions_theorem_onlineMM_SGMLMoE}).
The remaining hypotheses required to apply the general stochastic-approximation
convergence theory are treated as standing assumptions for this model
instantiation, namely
\cref{assumptionA8,assumptionA9,assumption_compact_sublevel}
(Lyapunov regularity, stability, and compactness of sublevel sets).
Moreover, the step-size condition \cref{assumptionA11} can be enforced by
construction, for instance by choosing $\gamma_n=\gamma_0 n^{-\alpha}$ with
$\alpha\in\left(\tfrac12,1\right]$ and $\gamma_0\in(0,1)$.
Under the verified assumptions together with these standing conditions, all
conclusions of
\cref{proposition_stationary,prop_Lyapunov_function,theorem_consistency}
hold for \cref{algorithm_onlineMM_SGMLMoE}.
\end{theorem}

\subsection{Why have other variants of incremental stochastic MM and EM algorithms not been successfully implemented in practice for the \md models?}\label{sec_why_fails_MM_EM}

In the literature, numerous incremental stochastic EM and MM algorithms have been proposed to optimize complex latent-variable models. However, most of these methods are not directly applicable to the \md model in \cref{eq_def_SGMoE}. A common limitation is that many online EM variants are formulated under an \emph{exponential-family} assumption on the target model, whereas \cref{eq_def_SGMoE} does not generally belong to an exponential family. For example, \cite{cappe2009line} present an online EM procedure assuming that the complete-data model is in an exponential family (see Assumption~1(a) therein). While this is stated as a structural assumption on the model itself, a careful reading of their analysis shows that the key ingredient exploited in the convergence argument is the availability of a surrogate/objective representation with an exponential-family-like structure, namely a decomposition in terms of sufficient statistics and a parameter update given by a deterministic maximization map. In our setting, we do not impose an exponential-family form on the \md model \cref{eq_def_SGMoE}; instead, we require this structure only for the \emph{majorizer surrogate} (cf.~\cref{assumptionA1,assumptionA2,assumptionA3,assumptionA4,assumptionObjectiveC1}), which yields a strictly weaker requirement and allows a more general treatment. Subsequent extensions, such as \cite{cappe2011online} (Hidden Markov Models) and \cite{cappe2009online} (state-space models), build on the same framework and retain the exponential-family condition, and thus they likewise do not directly cover \cref{eq_def_SGMoE}. Other EM-based adaptations exhibit similar constraints: the Fast iEM algorithm of \cite{karimi_global_2019} introduces a bias-correction step but still leverages exponential-family structure; the SAEM approach in \cite{kuhn_properties_2020} employs an ergodic Markov sampling kernel yet preserves the same requirement; and the Monte Carlo E-step strategy of \cite{rohde2011online}, which samples from distributions defined by the current parameter iterate, again relies on exponential-family assumptions.

MM-based schemes often fail for various reasons, primarily due to their reliance on overly restrictive regularity conditions. For instance, the stochastic MM algorithm proposed in \cite{mairal2013stochastic} requires strict convexity and Lipschitz continuity conditions that are typically violated by \cref{eq_def_SGMoE}. The incremental stochastic MM algorithm proposed in \cite{fort_sequential_2024} relies on expectation bounds stated in Assumptions A2(a) and A3(a) of their paper (discussed in \cref{rem_Assumption_A2_A3}), which are generally not satisfied by the surrogate functions defined in \cref{eq_surrogate_Taylor_SGGMoE} and \cref{eq_surrogate_Taylor_SGMLMoE}. These surrogate functions correspond to the original objective function in \cref{eq_def_SGMoE} for the softmax-gated Gaussian and multinomial logistic MoE models for continuous and discrete output data, respectively. In \cite{mairal_incremental_2015}, a variety of surrogate functions are proposed; however, aside from the quadratic majorizer (also adopted in our work), all other surrogates require properties such as L-smoothness, Lipschitz continuity, or convexity. Lastly, the large-scale incremental stochastic MM approach in \cite{karimi_minimization_2022} employs the MISSO method, which is based on the MISO method proposed by \cite{mairal_incremental_2015} and updates surrogate functions only partially at each iteration. However, they require that the surrogate function satisfies Assumption H2 (see \cref{rem_assumption_H2}), and even though this condition is true if $\hat{e}\left(\vtheta; \vtheta_n \right)$ is L-smooth. Nonetheless, formulating such a surrogate for \cref{eq_def_SGMoE} is still extremely challenging.

In summary, there is no existing incremental stochastic variant of the EM or MM algorithm that can be directly applied to SGMoE models, according to our knowledge. The failure of SGMoE to satisfy strict conditions such as Lipschitz continuity, L-smoothness, or convexity is demonstrated in \cref{sec: prop_check for SGMoE}. While these conditions can be rigorously shown to not hold in general, \cref{rem_Assumption_A2_A3} and \cref{rem_assumption_H2} remain unproven.
This is not due to their validity but rather the inherent difficulty of establishing the non-existence of a suitable surrogate function. At best, we can only suggest that constructing such a surrogate appears to be intractable in practice.
\subsubsection{Restrictive regularity conditions in existing incremental stochastic MM frameworks}
We now examine the regularity assumptions imposed by several recent variants of incremental
stochastic MM algorithms and explain why these conditions are not satisfied by the SGMoE
objective considered in this work.

\begin{condition}[Assumptions A2(a) and A3(a) in \cite{fort_sequential_2024}]
\label{rem_Assumption_A2_A3}
Let $g(\vtheta,\vz,\vu)$ denote a surrogate function for the objective, where $\vu$ belongs to an
auxiliary index set $\mathbb{U}$. Assumption~A2(a) in \cite{fort_sequential_2024} requires
\[
\PE\big[|g(\vtheta,\rvz,\vu)|\big] < \infty,
\qquad \forall (\vtheta,\vu)\in\mathbb{T}\times\mathbb{U},
\]
while Assumption~A3(a) further assumes the existence of some $p>1$ such that
\[
\PE\!\left[\sup_{(\vtheta,\vu)\in\mathbb{T}\times\mathbb{U}}
|g(\vtheta,\rvz,\vu)|^p \right] < \infty .
\]
\end{condition}

\begin{condition}[Assumption H2 in \cite{karimi_minimization_2022}]
\label{rem_assumption_H2}
Consider objective functions $\mathcal{L}_n(\vtheta)$ indexed by $n\in[N]$, and let
$\hat{\mathcal{L}}_n(\vtheta;\bar{\vtheta}_n)$ be a majorizing surrogate of $\mathcal{L}_n$ at
$\bar{\vtheta}_n\in\mathbb{T}$. Define
\[
\hat e(\vtheta;\bar{\vtheta}_n)
:= \hat{\mathcal{L}}_n(\vtheta;\bar{\vtheta}_n) - \mathcal{L}_n(\vtheta).
\]
Assumption~H2 in \cite{karimi_minimization_2022} postulates the existence of a constant $L>0$ such
that
\[
\|\nabla \hat e(\vtheta;\bar{\vtheta}_n)\|^2
\le 2L\, \hat e(\vtheta;\bar{\vtheta}_n),
\qquad \forall\,\vtheta\in\mathbb{T}.
\]
\end{condition}

\begin{remark}
The surrogate function constructed in
\cref{eq_surrogate_Taylor_SGGMoE} (see
\cref{proposition_surrogate_construction_SGMoE}) does not satisfy the moment conditions in
\cref{rem_Assumption_A2_A3}. Indeed, if the covariance matrices $\mSigma_k$ are allowed to grow
unbounded while all other components of $\vtheta$ are fixed, the Gaussian log-likelihood term
\[
-\sum_{k=1}^K \tau_{n,k}
\log\!\left(\cN(\vy_n;\vmu_k(\vx_n),\mSigma_k)\right)
\]
diverges due to the $\log|\mSigma_k|$ contribution. Consequently,
$\sup_{(\vtheta,\vu)} |g(\vtheta,\rvz,\vu)|^p$ is not integrable for any $p>1$, violating both
Assumptions~A2(a) and~A3(a) of \cite{fort_sequential_2024}.

Assumption~H2 of \cite{karimi_minimization_2022} is likewise not applicable in our setting.
That framework requires the construction of surrogates directly majorizing the objective
$\mathcal{L}_n$, which in turn relies on strong global regularity properties of
$f(\vtheta;\vz)$ such as Lipschitz continuity of the gradient, $L$-smoothness, or convexity.
As we show below, these properties do not generally hold for the SGMoE objective.
\end{remark}

\subsubsection{Failure of Lipschitz continuity, smoothness, and convexity for SGMoE}
\label{sec: prop_check for SGMoE}

Many existing incremental stochastic MM algorithms rely on strong global regularity assumptions on the objective function, such as Lipschitz continuity, $L$-smoothness, or convexity.  
We now show that these properties fail, in general, for the SGMoE negative log predictive density
\[
f(\vtheta;\vz) = -\log s_{\vtheta}(\vy\mid\vx),
\]
even when covariates and responses are bounded.

Let $\vmu(\vx)$ be a polynomial function of $\vx$. Then there exists a parameter matrix
$\mUpsilon_0$ such that $\vmu_0(\vx)=\vy$.
Define a sequence of parameter vectors
\[
\vtheta_n := (\vomega,\mUpsilon_0,\mSigma_n),
\qquad
\vtheta^\star := (\vomega,\mUpsilon_0,\mSigma^\star),
\]
where $\mSigma_n \succeq \zero$ with $\mSigma_n \downarrow \zero$ and $\mSigma^\star \succeq \zero$ is fixed.

\paragraph{Failure of global Lipschitz continuity.}
For Gaussian experts, when $\vmu_0(\vx)=\vy$, the SGMoE density satisfies
\[
s_{\vtheta_n}(\vy\mid\vx)
\;\ge\;
\sfg_k(\vw(\vx))\,\cN(\vy;\vy,\mSigma_n)
\]
for some $k\in[K]$, and therefore
\[
f(\vtheta_n;\vz)
= -\log s_{\vtheta_n}(\vy\mid\vx)
\;\ge\;
-\log \cN(\vy;\vy,\mSigma_n)
= \tfrac12 \log |\mSigma_n| + C
\;\xrightarrow[n\to\infty]{}\; +\infty.
\]
At the same time, the parameter difference
$\|\vtheta_n-\vtheta^\star\|$ remains bounded for any norm on $\rset^T$.
Hence,
\[
\frac{|f(\vtheta_n;\vz)-f(\vtheta^\star;\vz)|}
{\|\vtheta_n-\vtheta^\star\|}
\;\xrightarrow[n\to\infty]{}\; +\infty,
\]
which rules out global Lipschitz continuity of $f(\vtheta;\vz)$.

\paragraph{Failure of $L$-smoothness.}
The gradient of the Gaussian negative log-likelihood with respect to $\mSigma$ contains terms of the form
$\mSigma^{-1}$ and $\mSigma^{-2}$. Consequently,
\[
\|\nabla_\vtheta f(\vtheta_n;\vz)\|
\;\xrightarrow[n\to\infty]{}\; +\infty
\quad \text{as } \mSigma_n \downarrow \zero,
\]
showing that no global Lipschitz constant on the gradient can exist. Thus, $f$ is not $L$-smooth.

\paragraph{Failure of convexity.}
Let $\vmu_1(\vx)\neq \vy$ and $\vmu_2(\vx)\neq \vy$ satisfy
\[
\vmu_0(\vx)=\tfrac12\big(\vmu_1(\vx)+\vmu_2(\vx)\big),
\]
and define
\[
\vtheta_i := (\vomega,\mUpsilon_i,\mSigma_n), \quad i=1,2,
\]
where $\mUpsilon_i$ corresponds to $\vmu_i$.
As $\mSigma_n \downarrow \zero$, both $f(\vtheta_1;\vz)$ and $f(\vtheta_2;\vz)$ remain finite, since
$\vmu_i(\vx)\neq \vy$.
However,
\[
f\!\left(\tfrac{\vtheta_1+\vtheta_2}{2};\vz\right)
= -\log \cN(\vy;\vy,\mSigma_n)
\;\xrightarrow[n\to\infty]{}\; +\infty.
\]
This violates Jensen’s inequality and shows that the SGMoE objective is not convex, and hence not strongly convex, in general.

These results demonstrate that the SGMoE negative log predictive density does not satisfy global Lipschitz continuity, $L$-smoothness, or convexity. As a consequence, several existing incremental stochastic MM frameworks that rely on these restrictive regularity conditions are not directly applicable to SGMoE models.

\section{Numerical experiments}\label{sec_experimental_studies}
In this section, we conduct comprehensive experiments on both synthetic and real-world datasets to evaluate the effectiveness and robustness of the proposed method. All experiments were conducted using Python 3.12.2 on a standard Unix-based system.

\subsection{Experimental settings}

{\bf Assessment of initialization schemes.} We consider two initialization strategies:
\begin{enumerate}
    \item \emph{Perturbed ground truth initialization:} The initial parameters are generated by adding Gaussian noise to the true data-generating parameters. This strategy mostly serves to test the local stability of the algorithm. 
    \item \emph{Data-driven initialization:} This strategy involves partitioning the first batch of incoming data into clusters using a standard K-means clustering procedure and then compute an initial set of parameters. For real-world data, this is the realistic initialization approach.
\end{enumerate}


{\bf Assessment of regression performance.} Furthermore, we benchmark the regression performance of our algorithm against five baseline stochastic models using a standard train-test evaluation framework. With the regression function in \cref{eq_def_SGaBloME_Expectation}, there are two approaches for assessment. Firstly, for strictly synthetic data, we compute the regression values using the estimated parameters and compare it to the ones computed with true parameters, this is called \emph{estimation error}. Secondly, for both synthetic and real-world data, we measure the ``fit" of the regression line on the testing samples, similarly to how regression models are traditionally bench-marked, this is called \emph{prediction error}. For both approaches, we utilize the following point-wise error metrics: mean squared error (MSE), mean absolute percentage error (MAPE), and normalized root mean squared error (NRMSE). It is also worth noting that, throughout this section, the maximum log-likelihood at each iteration of the incremental stochastic MM algorithm is computed by substituting the parameter estimates from that iteration into \cref{multi_D_lossfunc}.

{\bf Comparative methods.} In this simulation study on MoE models, we compare our proposed method against several widely-used stochastic online or incremental stochastic optimization algorithms, including stochastic approximation or SGD~\cite{robbins_stochastic_1951}, RMSProp~\cite{hinton_neural_2012}, Adam~\cite{kingma_adam_2015}, and Sophia~\cite{liu_sophia_2024}.

\subsection{Experiments on synthetic data}
We now present experiments conducted on synthetic data. Specifically, we consider a dataset with two generating Gaussian distributions, which correspond to two clusters. Each cluster consists of $N = 1,000$ samples. We set the model configurations as follows: $D_W = 1$, $D_V = 1$, $P = 2$, and $Q = 1$. The true parameters, denoted by $\boldsymbol{\omega}^0$, $\boldsymbol{\upsilon}^0$, and $\boldsymbol{\sigma}^0$, are specified as
\begin{align*}
\vomega^{0} &=
\begin{bmatrix}
\begin{bmatrix}
0 & 8 \\
0 & 0
\end{bmatrix},
\begin{bmatrix}
0 & 0 \\
0 & 0
\end{bmatrix}
\end{bmatrix}, \quad 
\vupsilon^{0} =
\begin{bmatrix}
\begin{bmatrix}
\begin{bmatrix} 0 & 0 \end{bmatrix} \\
\begin{bmatrix} -2.5 & 0 \end{bmatrix}
\end{bmatrix}, 
\begin{bmatrix}
\begin{bmatrix} 0 & 0 \end{bmatrix} \\
\begin{bmatrix} 2.5 & 0 \end{bmatrix}
\end{bmatrix}
\end{bmatrix}^\top, \quad
\vsigma^{0} =
\begin{bmatrix}
\begin{bmatrix} 1 \end{bmatrix},
\begin{bmatrix} 1 \end{bmatrix}
\end{bmatrix}.
\end{align*}

A visualization of the simulated dataset is provided in \cref{fig_simulation_dataset_visualization}. The dataset is divided into training and testing subsets, with 80\% of the data allocated to training.  It is important to note that under the \emph{perturbed ground truth initialization} setting, the focus is to analyze the convergence behavior of the MLE, and thus only the training data is used. In contrast, the \emph{data-driven initialization} scheme utilize both the training and testing data to evaluate generalization performance.

\begin{figure}[!ht]
    \centering
    \begin{subfigure}[b]{0.49\textwidth}
        \centering
        \includegraphics[trim={0 0 0 0.75cm},clip,width=\textwidth]{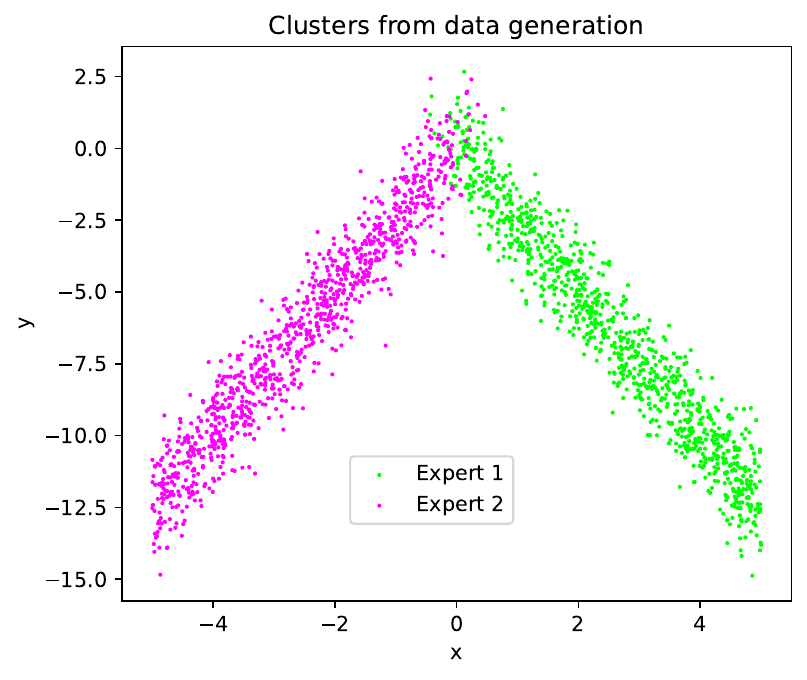} 
        \caption{Typical realization of the synthetic dataset.}
    \end{subfigure}
    \hfill
        \begin{subfigure}[b]{0.49\textwidth}
        \centering
        \includegraphics[trim={0 0 0 0.75cm},clip,width=\textwidth]{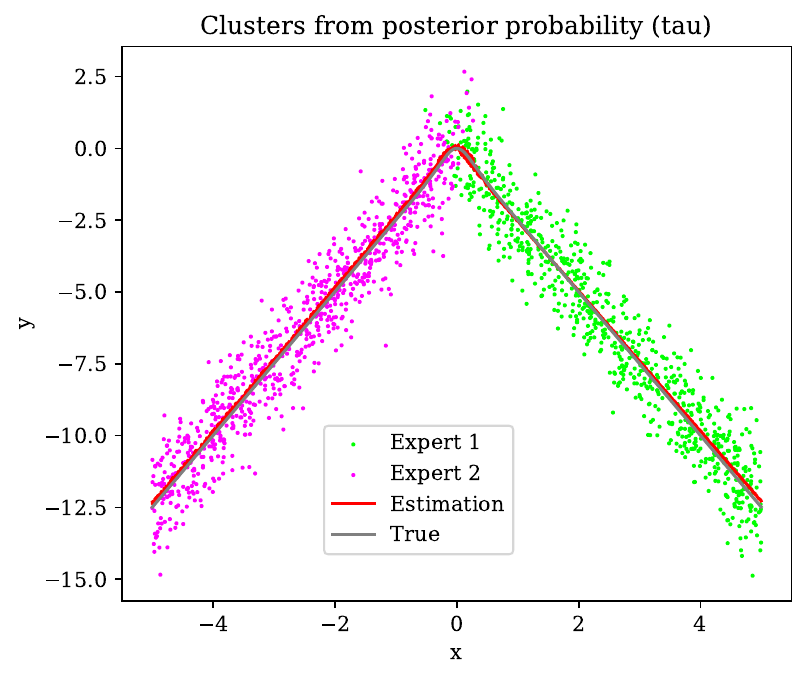} 
         \caption{Estimated clusters and regression functions.}\label{fig_Predicted_mean_line_IDEAL}
    \end{subfigure}
    \caption{Visualization of the synthetic dataset generated from the true parameters and estimated clusters and regression functions by our proposed model with perturbed ground truth initialization in softmax-gated MoE model.}
    \label{fig_simulation_dataset_visualization}
\end{figure}
\subsubsection{Evaluation on synthetic data with perturbed ground truth initialization} \label{fig_ideal_init_lowdime}

We first evaluate the performance of our algorithm under the \textit{perturbed ground truth initialization} setting. More precisely, the noise variables are independently drawn from a Gaussian distribution with mean 0 and variance $1$, and then scaled by a factor of $0.005$. These small perturbations are added to the true parameters to generate the initial values. For each series, the initial latent state $s_{i,0}$ is computed as the average of the corresponding $S(\vz)$ values over the first $85$ observations of $\vz$, using the initialized parameters.

In \cref{fig_ideal_MLE}, we present the progression of the maximum log-likelihood value over $N=1600$ iterations. Furthermore, we apply \textit{Polyak averaging}, a widely adopted technique used to reduce variance in optimization trajectories and to dampen the influence of initially volatile parameter estimates. Starting from a designated iteration $N_0$ (we choose $N_0 = 100$), the averaged sequence is initialized as $\boldsymbol{\vtheta}_{N_0}^P = 0$. For all $n \geq N_0$, the updates follow the rule:
\[
\boldsymbol{\vtheta}_{n+1}^P = \boldsymbol{\vtheta}_n^P + \alpha_{n - N_0 + 1} \left( \boldsymbol{\vtheta}_n - \boldsymbol{\vtheta}_n^P \right), \quad \text{where } \alpha_n \text{ is typically set to } \frac{1}{n}.
\]
In the figure, the true MLE trajectory is plotted in black, and the MLE trajectory with Polyak averaging is shown in red. It can be observed that both trajectories converge, and the red line, which incorporates Polyak averaging, exhibits greater overall stability. Additionally, we plot the trajectories of selected parameter dimensions, including $\vomega^0_{1,0,1}$, $\vupsilon^0_{0,1,0,0}$, and $\vsigma^0_{0,0}$, as shown in \cref{fig_ideal_parameter}. These trajectories clearly demonstrate that our optimization method effectively drives the parameters to converge to their true values. Furthermore, in \cref{fig_Predicted_mean_line_IDEAL}, we visualize the predicted mean values from our model compared to the true mean, providing further insight into the predictive accuracy of the estimated parameters.

\begin{figure}
    \centering
    \begin{subfigure}[b]{0.495\textwidth}
        \centering
        \includegraphics[trim={0 0 0 .9cm},clip,width=\textwidth]{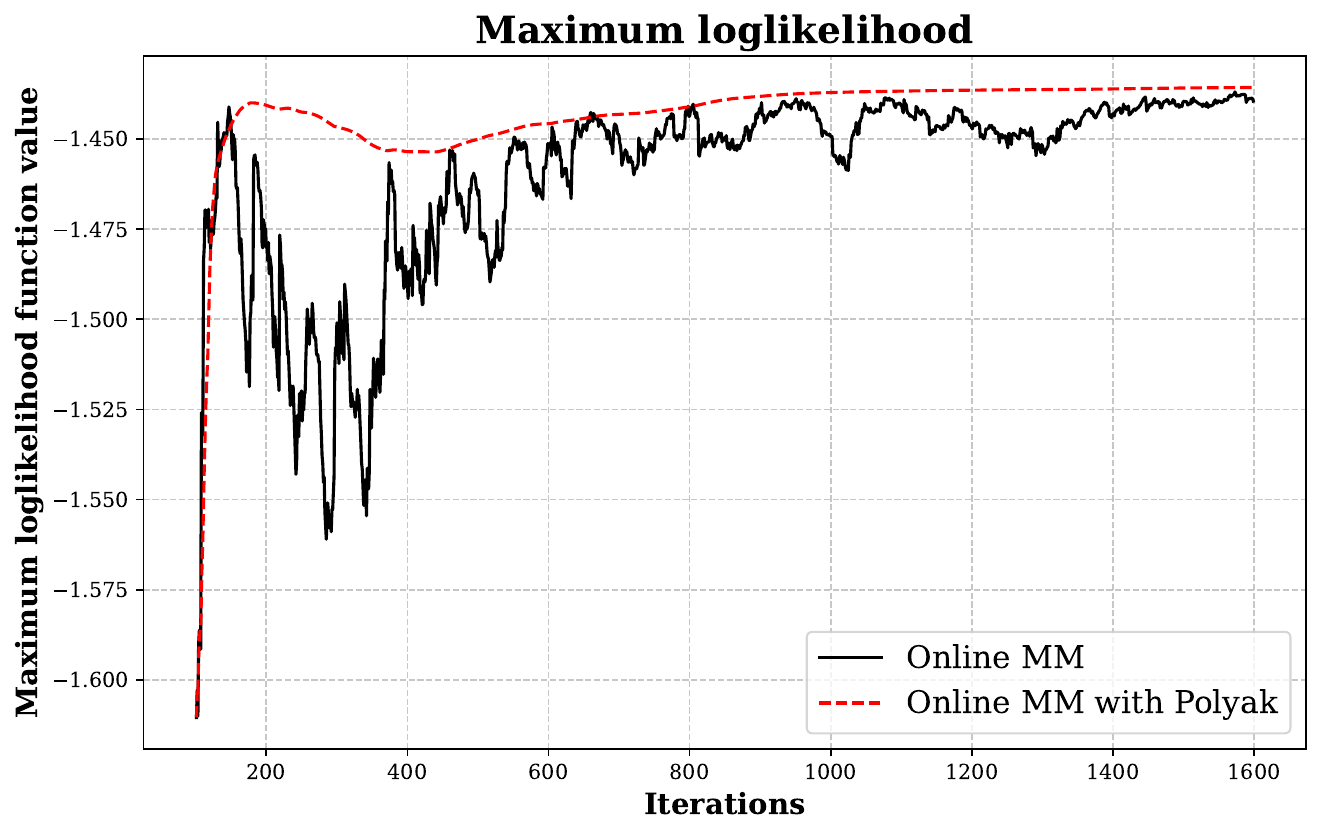} 
        \caption{Loss function with perturbed ground truth initialization.}\label{fig_ideal_MLE}
    \end{subfigure}
    \hfill
    \begin{subfigure}[b]{0.495\textwidth}
        \centering
        \includegraphics[trim={0 .2cm 0 .9cm},clip,width=\textwidth]{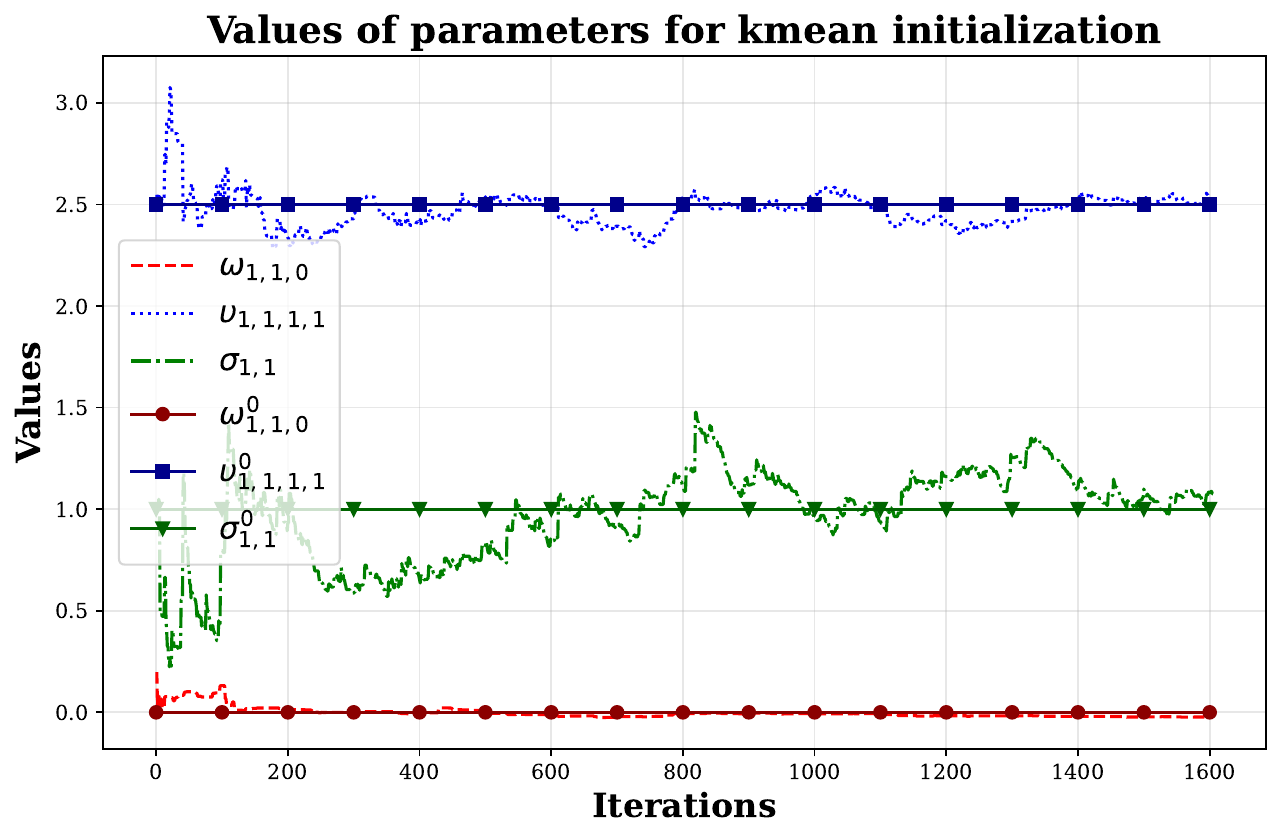} 
        \caption{Parameter estimation with perturbed ground truth initialization.}\label{fig_ideal_parameter}
    \end{subfigure}
    
       \centering
    \begin{subfigure}[b]{0.495\textwidth}
        \centering
        \includegraphics[trim={0 0 0 0.85cm},clip,width=\textwidth]{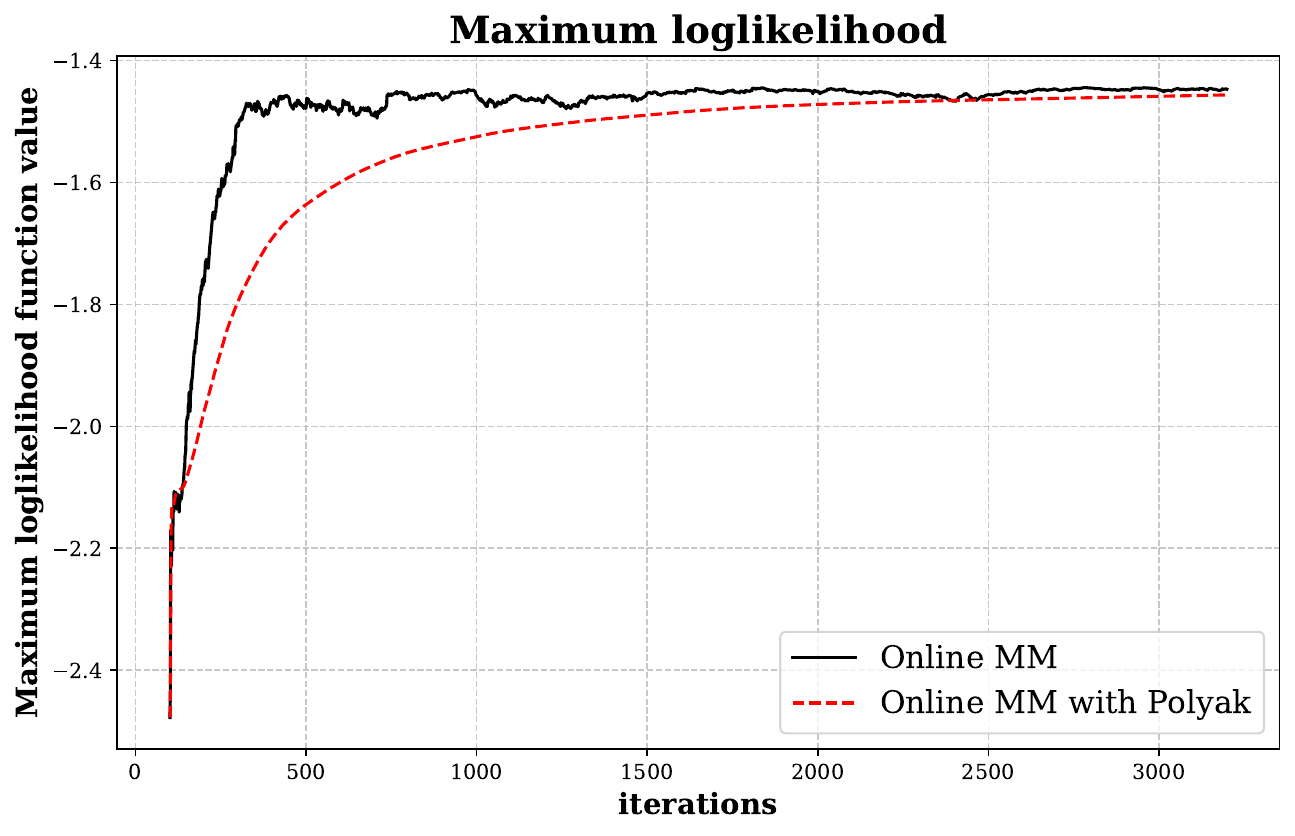} 
        \caption{Loss function with data-driven initialization.}\label{fig_kmean_MLE}
    \end{subfigure}
    \hfill
    \begin{subfigure}[b]{0.495\textwidth}
        \centering
        \includegraphics[trim={0 0 0 .8cm},clip,width=\textwidth]{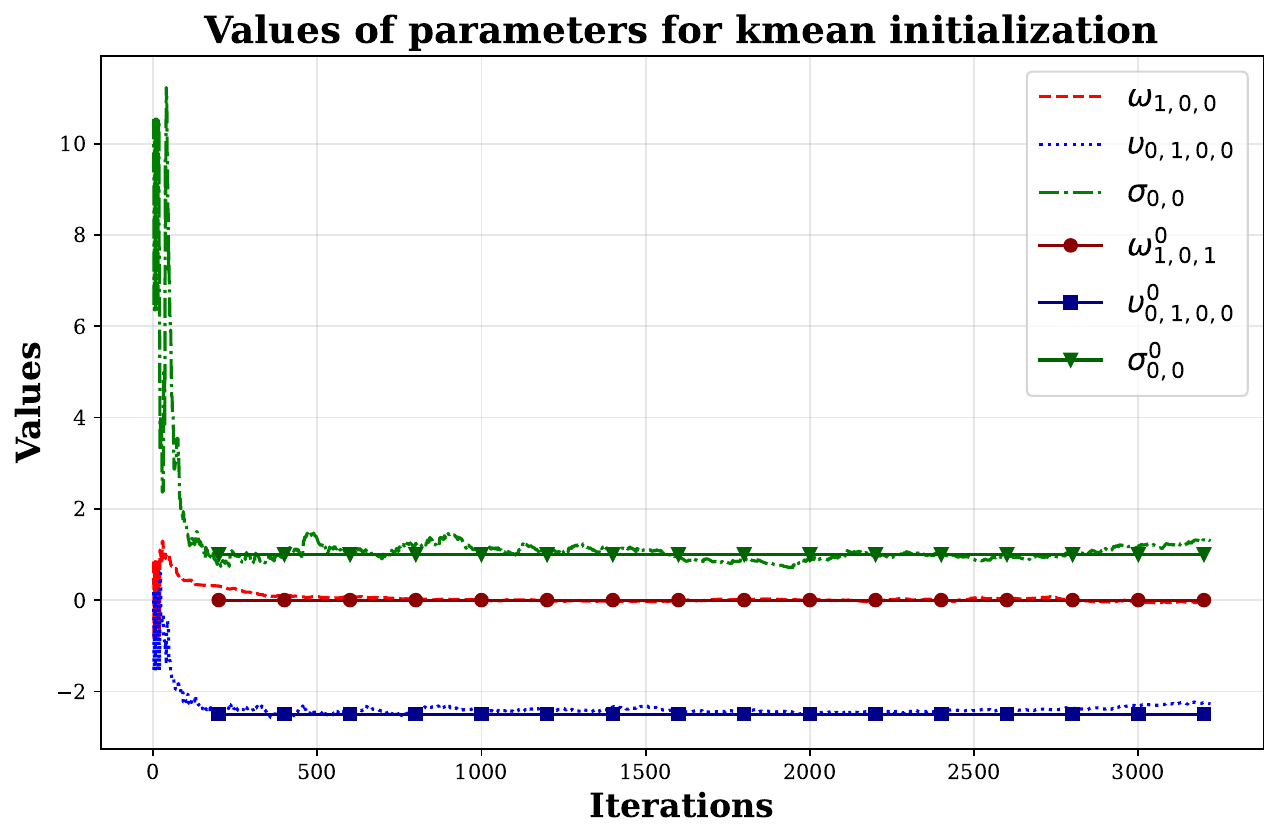} 
        \caption{Parameter estimation with data-driven initialization.}\label{fig_kmean_parameters}
    \end{subfigure}
    \caption{Maximum log-likelihood function with and without applying the Polyak average method over each iteration and value of selected parameters over $N=1600$ iteration.}
\end{figure}

\subsubsection{Evaluation on synthetic data with data-driven initialization}\label{Kmean_init_syn_data_lowdime}

To further demonstrate the robustness of our algorithm, we employ an alternative initialization strategy based on K-means clustering. Specifically, K-means is executed using randomly selected centroids, repeated up to ten times, and the configuration with the lowest distortion is selected as the initialization. The initial latent state values for each series, denoted by $s_{i,0}$, are computed in the same manner as previously described. The progression of the maximum log-likelihood function over iterations is presented in \cref{fig_kmean_MLE}, while the convergence of selected parameter values is depicted in \cref{fig_kmean_parameters}. Despite the inferior starting point provided by K-means initialization compared to the ideal setting, our algorithm exhibits strong robustness, successfully guiding the parameters toward their true values.

Additionally, \cref{fig_Distance_to_MLE} illustrates the trajectory of the distance between the estimated and true log-likelihood functions over iterations. This plot clearly demonstrates that our proposed method achieves the fastest convergence rate among all compared baselines. To mitigate the impact of numerical instability observed during early iterations, we report the maximum log-likelihood values starting from iteration 400.
The results, shown in \cref{tab_Kmean_tm}, indicate that our proposed method consistently outperforms all baseline approaches across all metrics. Second, each algorithm is trained solely on the training data and evaluated on the test set, with prediction accuracy again measured against the ground truth using the same metrics. Details of the train/test data split are provided at the beginning of this section. The corresponding results are summarized in \cref{tab_Kmean_tt}.
\begin{figure}[H]
    \centering
    \includegraphics[trim={0 0.25cm 0 1cm},clip,width=0.8\linewidth]{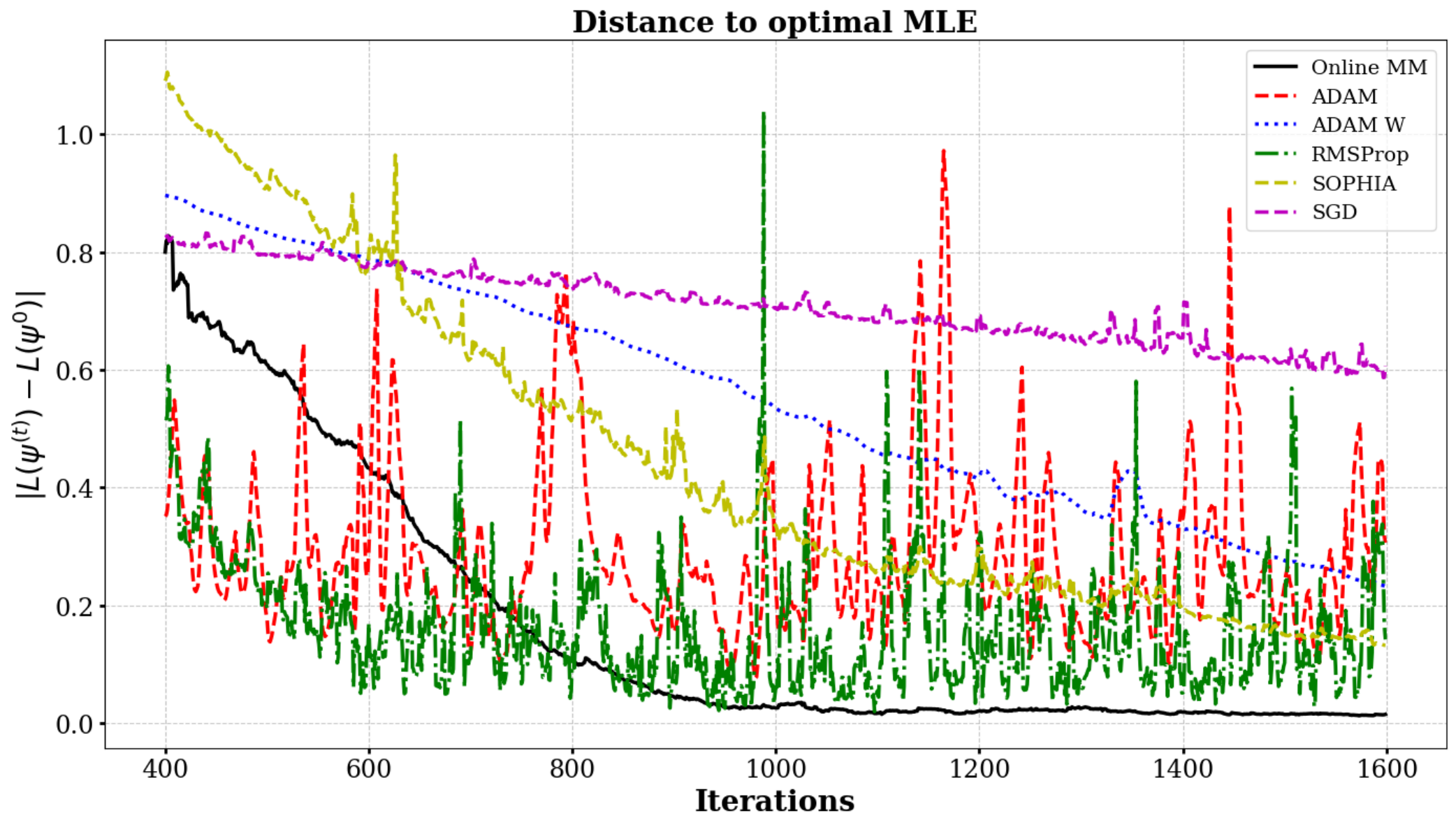}
    \caption{Distance to the true maximum log-likelihood of all algorithms for each iteration.}
    \label{fig_Distance_to_MLE}
\end{figure}
\begin{table}[H]
\centering
\begin{minipage}{0.45\textwidth}
    \centering
    \caption{Estimation errors with data-driven initialization}
    \begin{tabular}{|l|c|c|c|}
\hline
\textbf{Model} & \textbf{MSE$\downarrow$} & \textbf{MAPE$\downarrow$} & \textbf{NMRSE$\downarrow$} \\
\hline
\textbf{Inc. MM} & \textbf{0.014} & \textbf{0.228} & \textbf{0.009} \\
{SGD}       & {0.415} & {2.230} & {0.050} \\
{Adam}      & {0.824} & {1.272} & {0.061} \\
{Adam W}    & {0.192} & {1.816} & {0.035} \\
{RMSProp}   & {0.322} & {0.545} & {0.042} \\
{Sophia}    & {0.272} & {1.633} & {0.042} \\
\hline
\end{tabular}

    \label{tab_Kmean_tm}
\end{minipage}\hfill
\begin{minipage}{0.45\textwidth}
    \centering
    \caption{Prediction errors with data-driven initialization}
    \begin{tabular}{|l|c|c|c|}
\hline
\textbf{Model} & \textbf{MSE$\downarrow$} & \textbf{MAPE$\downarrow$} & \textbf{NMRSE$\downarrow$} \\
\hline
\textbf{Inc. MM} & \textbf{0.950} & 0.480 & \textbf{0.059} \\
SGD       & 1.517 & 0.332 & 0.075 \\
Adam      & 1.789 & 0.422 & 0.082 \\
Adam W    & 1.139 & \textbf{0.219} & 0.065 \\
RMSProp   & 1.505 & 0.478 & 0.075 \\
Sophia    & 1.256 & 0.257 & 0.069 \\
\hline
\end{tabular}

    \label{tab_Kmean_tt}
\end{minipage}
\end{table}

\subsection{Experiments on higher-dimensional synthetic data}

In the previous section, we demonstrated that our proposed methods perform exceptionally well on low-dimensional synthetic data. However, real-world datasets are often higher-dimensional, and the performance observed in low-dimensional settings may not necessarily generalize. To evaluate the robustness of our approach under more realistic conditions, we conduct experiments on simulated data with significantly higher dimensionality. As we show in this section, even in these more challenging scenarios, our model consistently outperforms baseline methods, underscoring its robustness and effectiveness in high-dimensional settings.

Similar to the previous section, we set the number of data points and the number of clusters to 1,000 and 2, respectively. The parameters \( D_W \) and \( D_V \) are set to 1, while \( P \) and \( Q \) are chosen to be 10 and 1, respectively. Additionally, the true parameters, denoted by \( \vomega^0, \vupsilon^0 \), and \( \vsigma^0 \), are initialized using \cref{alogo_initialize_of_highdim_gen}.

\begin{algorithm}[H]
\caption{Initialization of \( \vomega^0 \), \( \vupsilon^0 \), and \( \vsigma^0 \) for Higher-Dimensional Simulation}
\label{alogo_initialize_of_highdim_gen}
\begin{algorithmic}[1]
\State \textbf{Input:} Number of experts \( K \), dimensions \( D_V \), \( D_W \), parameters \( P \), \( Q \), seed \( s_e \).
\State Initialize tensors:
\begin{itemize}
    \item \( \vupsilon^0 \in \sR^{K \times (D_V + 1) \times Q \times P} \gets \zero \).
    \item \( \vsigma^0 \in \sR^{K \times Q \times Q} \gets \zero \).
    \item \( \vomega^0 \in \sR^{P \times K \times (D_W + 1)} \gets \zero \).
\end{itemize}

\State \textbf{Step 1: Initialization of variance components}
\For{$k = 1$ to $K$} 
    \For{$q = 1$ to $Q$}
        \State Set random seed \( s_e + kQ + q \).
        \State Sample \( \vsigma^0[k, q, q] \sim \text{Uniform}(0.5, 1.5) \).
    \EndFor
\EndFor

\State \textbf{Step 2: Initialization of expert parameters}
\For{$(k, d, q, p)$ over their respective ranges $([K]\times \{\{0\}\cup [D_V]\}\times[Q]\times[P])$}
    \State Set random seed: $s_e + s_e + k(D_V+1)QP + dQP + qP + p$. 
    \If{$p = 1$}
        \State Sample \( \vupsilon^0[k, d, q, p] \sim \text{Uniform}(1.5, 5) \).
    \EndIf
    \If{$k > 1$ and $d > 0$}
        \State \( \vupsilon^0[k, d, q, p] \gets -\vupsilon^0[k, d, q, p] \).
    \EndIf
\EndFor

\State \textbf{Step 3: Initialization of gating parameters}
\For{$(p, k, d)$ over their respective ranges $([P]\times[K-1]\times \{\{0\}\cup [D_V]\})$}
    \If{$p = 1$, $k = 1$, $d = 1$}
        \State \( \vomega^0[p, k, d] \gets 8 \).
    \Else
        \State \( \vomega^0[p, k, d] \gets 0 \).
    \EndIf
\EndFor
\end{algorithmic}
\end{algorithm}

\subsubsection{Evaluation on higher dimensional synthetic data with perturbed ground truth initialization}
Following the same procedure as in \cref{fig_ideal_init_lowdime}, we begin our evaluation under the setting using perturbed ground truth initialization. Specifically, the initialization is performed by perturbing the ground-truth parameters with Gaussian noise, and the latent series $\vs_i$ is initialized accordingly.
The progression of the maximum log-likelihood values across epochs is illustrated in \cref{tab_Highdime_ideal_loss}. In this plot, the trajectory corresponding to the true log-likelihood is depicted in black, while the log-likelihood obtained via Polyak averaging is shown in red. The figure clearly demonstrates that, as the number of iterations increases, the proposed Incremental MM algorithm effectively maximizes the log-likelihood function. Moreover, the inclusion of Polyak averaging enhances the stability of the learning process and leads to consistently superior performance in terms of log-likelihood values.
As summarized in \cref{tab_Param_converge_metric}, the error values across all parameters remain low under all metrics, thereby confirming the ability of our algorithm to recover the true model parameters with high precision.

\begin{table}[H]
\centering
\begin{tabular}{|c|c|c|c|}
\toprule
\textbf{Parameters} & \textbf{MSE} & \textbf{MAPE} & \textbf{NMRSE} \\
\midrule
Omega  & 0{,}0051 & 0{,4750} & 0{,0087} \\
Upsilon & 0{,1849} & 1{,3460} & 0{,0917} \\
Sigma   & 0{,0015} & 0{,0183} & 0{,0435} \\
\bottomrule
\end{tabular}
\caption{The distance between each parameter and its ground truth using different metrics.}
\label{tab_Param_converge_metric}
\end{table}
\begin{figure}[H]
    \centering
    \includegraphics[trim={0 0.5cm 0 1.76cm},clip,width=0.9\linewidth]{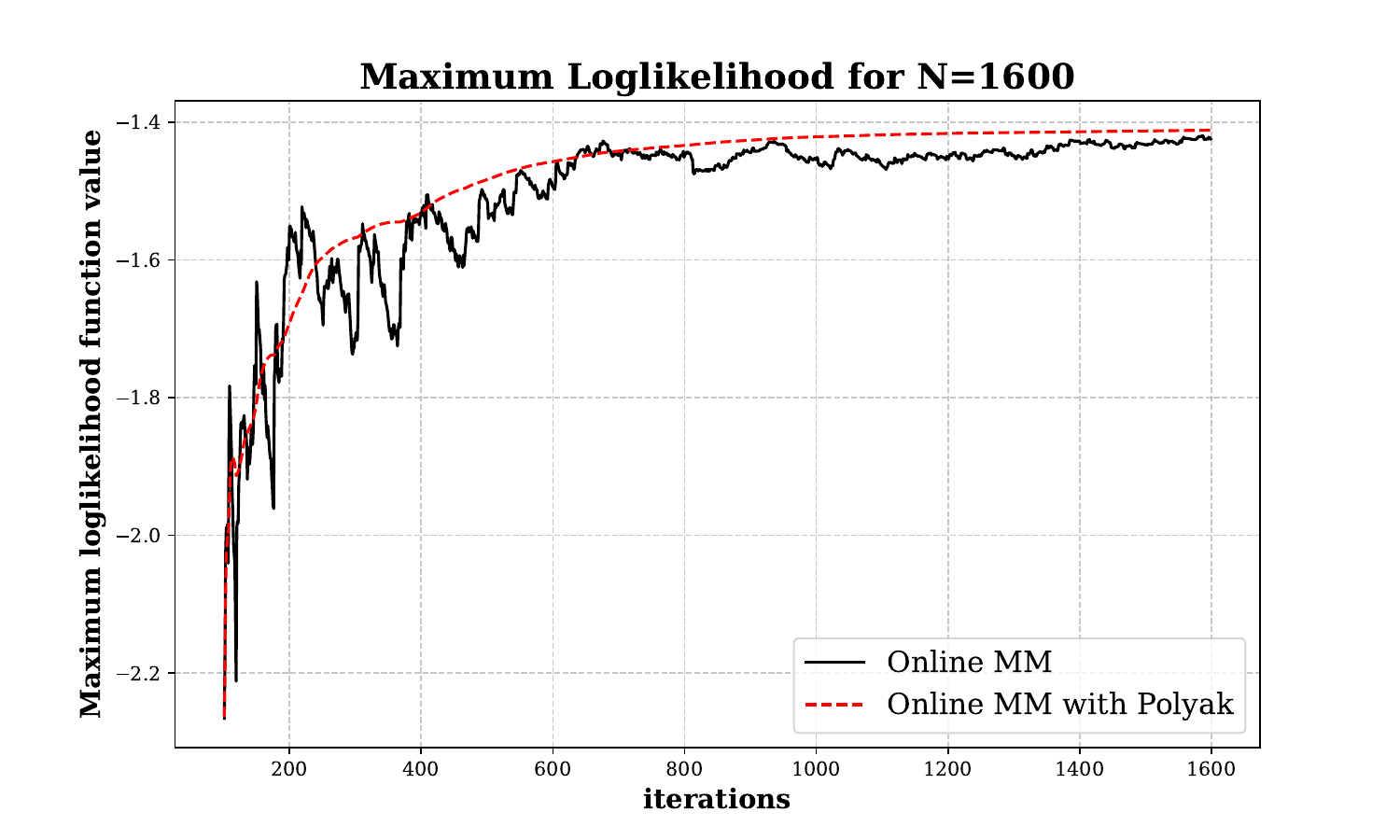}
    \caption{Maximum log-likelihood function with and without applying the Polyak averaging method over each iteration for higher dimensional synthetic data.}
    \label{tab_Highdime_ideal_loss}
\end{figure}

\subsubsection{Evaluation on higher dimensional synthetic data with data-driven initialization}
In this section, we evaluate the performance of our method on high-dimensional synthetic data using data-driven initialization. This initialization strategy is generally more adaptable and better aligned with real-world scenarios, thus serving as a robust framework for assessing the resilience of clustering algorithms. To ensure consistency and comparability, we adopt the same configuration as described in \cref{Kmean_init_syn_data_lowdime}, using the synthetic dataset introduced at the beginning of this section.
As in \cref{Kmean_init_syn_data_lowdime}, we utilize two complementary evaluation strategies. The first involves comparing the predicted means against the true mean trajectory of the data, while the second employs a train-test split to assess the generalization performance. The results of the first evaluation are summarized in \cref{tab_Kmean_tm_hidime}, and the results of the train-test evaluation are presented in \cref{tab_Kmean_tt_hidime}. These findings further demonstrate the robustness and superior accuracy of our proposed method in high-dimensional settings.

\begin{table}[H]
\centering
\begin{minipage}{0.45\textwidth}
    \centering
    \caption{Estimation errors with data-driven initialization for high dimensional simulated data}
    \begin{tabular}{|l|c|c|c|}
\hline
\textbf{Model} & \textbf{MSE$\downarrow$} & \textbf{MAPE$\downarrow$} & \textbf{NMRSE$\downarrow$} \\
\hline
\textbf{Inc. MM} & \textbf{2,809} & \textbf{0,197} & 0,081 \\

SGD & 3,278 & 0,243 & \textbf{0,073} \\

Adam & 3,402 & 0,245 & 0,074 \\

Adam W & 2,904 & 0,233 & 0,080 \\

RMSProp & 2,840 & 0,226 & 0,079 \\

Sophia & 2,912 & 0,256 & 0,077 \\
\hline
\end{tabular}

    \label{tab_Kmean_tm_hidime}
\end{minipage}\hfill
\begin{minipage}{0.45\textwidth}
    \centering
    \caption{Prediction errors with data-driven initialization for high dimensional simulation data}
    \begin{tabular}{|l|c|c|c|}
\hline
\textbf{Model} & \textbf{MSE$\downarrow$} & \textbf{MAPE$\downarrow$} & \textbf{NMRSE$\downarrow$} \\
\hline
\textbf{Inc. MM} & \textbf{4,397} & \textbf{0,293} & 0,098 \\
SGD & 4,607 & 0,304 & 0,099 \\

Adam & 4,725 & 1,148 & \textbf{0,087} \\

Adam W & 4,434 & 0,439 & 0,090 \\

RMSProp & 4,478 & 0,359 & 0,106 \\

Sophia & 4,674 & 0,314 & 0,105 \\
\hline
\end{tabular}

    \label{tab_Kmean_tt_hidime}
\end{minipage}
\end{table}

\subsection{Experiments on real-world data}

To further assess the robustness and practical applicability of our proposed model, we conducted empirical evaluations on two real-world datasets.

\subsubsection{Dent maize genotypes dataset}
The first dataset used in our study aims to elucidate the genetic and molecular mechanisms underlying drought-responsive traits in maize \cite{prado2018phenomics,blein2020systems}. It comprises 254 dent maize genotypes that capture the genetic diversity of the species. These genotypes were grown under two distinct watering conditions and evaluated for seven ecophysiological traits, although some measurements were incomplete \cite{prado2018phenomics}.

In parallel, proteomic profiling was performed on leaf samples collected from the same plants, yielding quantification for $2,055$ proteins. Of these, $973$ proteins were measured using continuous data, while the remaining $1,082$ were recorded as count data \cite{blein2020systems}. For our analysis, we focused on a subset of the dataset, specifically two ecophysiological traits (late leaf area and water use together with the $973$ continuously quantified proteins, all collected under water deficit conditions.

After removing samples with missing values, the final dataset consisted of $233$ maize genotypes (i.e., \( N=233 \)), each with measurements for the two selected traits (\( Q=2 \)) and $973$ protein expression levels (\( P=973 \)). To mitigate numerical issues associated with high latent dimensions \( P \) and \( Q \), we applied Lasso regression with feature selection (from \texttt{scikit-learn}), see also \cite{friedman_regularization_2010,tibshirani_regression_1996}, to identify the 11 most relevant features from \( \vx \). These 11 features were used as inputs, and the first dimension of \( \vy \) was selected as the output for the experiment.

The dataset was split into a training set with $185$ samples and a test set with $47$ samples. We set both \( D_V \) and \( D_W \) to 1. Because the true underlying parameters are unknown in this real-world dataset, we adopted K-means for initialization for all the baseline. Each method was trained on the training set and evaluated on the test set using MSE, MAPE, and NRMSE as performance metrics. To further ensure robustness, we conducted 5-fold cross-validation, reporting the average performance over the folds along with the corresponding standard deviations to reflect result stability.

Learning rates for all solvers were finely tuned to ensure optimal performance. For initialization of \( \vomega \), we considered two strategies: (1) setting all entries to zero, and (2) applying logistic regression between the inputs \( \vx \) and the labels generated via K-means. The results are summarized in \cref{tab_maize_genotypes}, where each cell presents the mean metric value, with the standard deviation shown in parentheses. Incremental MM refers to the approach with zero-initialized \( \vomega \), while Incremental MM* denotes initialization via logistic regression. Overall, Incremental MM and Incremental MM* consistently outperform the baseline methods, even on real-world data, and maintains comparable standard deviations, highlighting both accuracy and stability.

\subsubsection{Communities and Crime dataset}
The second dataset we consider is the Communities and Crime dataset, available at \cite{communities_and_crime_183}. This dataset integrates socio-economic data from the $1990$ U.S. Census, law enforcement data from the $1990$ Law Enforcement Management and Administrative Statistics survey, and crime statistics from the $1995$ FBI Uniform Crime Reporting program. It comprises $1,994$ instances, each corresponding to a distinct U.S. community, and includes $127$ attributes capturing a range of factors such as demographics, economic indicators, housing conditions, and characteristics of local law enforcement agencies. The primary target variable is the per capita rate of violent crime, computed based on population size and the aggregated number of reported incidents of murder, rape, robbery, and assault.

Similar to the previous experiment, we employ Lasso regression with feature selection from \texttt{scikit-learn}, see also \cite{friedman_regularization_2010,tibshirani_regression_1996}, to identify the ten most informative attributes. These ten features are used as input variables $\vx$, alongside the output variable $\vy$, which corresponds to the per capita rate of violent crimes. The dataset is subsequently divided into a training set comprising $1600$ samples and a test set containing 397 samples. We set $D_W = 1$ by default and evaluate the performance of Incremental MM under two configurations: one with $D_V = 1$ and another with $D_V = 2$. As the true underlying parameters are unknown in this context, we employ K-means initialization.
To assess robustness, we apply five-fold cross-validation, reporting the mean and standard deviation of each metric across the folds.

To ensure fair comparisons, the learning rates of all solvers are carefully fine-tuned. Consistent with our previous experimental settings, we examine two strategies for initializing $\vomega$: (1) setting all entries to zero, and (2) a ``warm-start" approach where we cluster the first data batch and utilize the assignment labels as targets to solve a logistic regression problem to get the initial gating parameters. The final evaluation results are presented in \cref{Criminal_res}.

\begin{table}[H]
\centering
\caption{Prediction errors for the Dent maize genotypes dataset with data-driven initialization. The best result for each metric is shown in \textbf{bold}, and the second-best is \underline{underlined}. \textit{Incremental MM*} denotes the case where $\vomega$ is initialized using logistic regression.}
\vspace{0.3em}
\label{tab_maize_genotypes}
\begin{tabular}{|c|c|c|c|}
\toprule
\textbf{Model} & \textbf{MSE$\downarrow$} & \textbf{MAPE$\downarrow$} & \textbf{NMRSE$\downarrow$} \\
\midrule
Incremental MM & \underline{0.1190} (0.021) & \underline{0.1073} (0.013) & \underline{0.1421} (0.016) \\
Incremental MM* & \textbf{0.1156} (0.017) & \textbf{0.1065} (0.013) & \textbf{0.1397} (0.017) \\
SGD       & 0.1332 (0.030) & 0.1147 (0.018) & 0.1514 (0.021) \\
Adam      & 0.1329 (0.026) & 0.1164 (0.015) & 0.1522 (0.024) \\
Adam W    & 0.1327 (0.026) & 0.1163 (0.015) & 0.1521 (0.024) \\
RMSProp   & 0.1324 (0.028) & 0.1168 (0.017) & 0.1514 (0.023) \\
Sophia    & 0.1332 (0.030) & 0.1147 (0.018) & 0.1514 (0.021) \\
\bottomrule
\end{tabular}
\end{table}

\begin{table}[H]
\centering
\caption{Prediction errors for the Communities and Crime dataset with data-driven initialization. The best result for each metric is highlighted in bold, while the second-best is underlined. Incremental MM$^i$ denotes the Incremental MM algorithm with $D_V$ set to $i$. Incremental MM* indicates the case where $\vomega$ is initialized using logistic regression.} 
\label{Criminal_res}
\vspace{0.3em}
\begin{tabular}{|c|c|c|c|}
\toprule
\textbf{Model} & \textbf{MSE$\downarrow$} & \textbf{MAPE$\downarrow$} & \textbf{NMRSE$\downarrow$} \\
\midrule
Incremental MM$^2$ & 0.0321 (0.004) & \underline{0.9537} (0.116) & \textbf{0.1307} (0.016) \\
Incremental MM$^1$ & \textbf{0.0283} (0.003) & 1.0440 (1.083) & 0.1461 (0.012) \\
Incremental MM* & \underline{0.0316} (0.010) & \textbf{0.9500} (0.229) & 0.1594 (0.010) \\
SGD           & 0.0462 (0.012) & 1.8288 (0.231) & 0.1591 (0.010) \\
Adam          & 0.0314 (0.005) & 2.2066 (0.914) & \underline{0.1413} (0.007) \\
Adam W        & 0.0315 (0.004) & 1.5257 (0.507) & 0.1417 (0.006) \\
RMSProp       & 0.0322 (0.005) & 2.126 (1.184)  & 0.1415 (0.006) \\
Sophia        & 0.0309 (0.005) & 1.9419 (1.059) & 0.1430 (0.007) \\
\bottomrule
\end{tabular}
\end{table}

\section{Conclusion and future investigation}\label{sec_conclusion_future}

This work introduces a novel incremental stochastic MM algorithm that generalizes the incremental stochastic EM algorithm and provides a robust framework for handling high-volume and incremental data. By relaxing the stringent assumptions typically required by EM-based methods, such as convexity and explicit latent variable representations, our approach broadens the applicability of incremental stochastic estimation to a wider class of models. Notably, we demonstrate the effectiveness of the proposed method in fitting softmax-gated MoE models, a setting in which the incremental stochastic EM algorithm fails, achieving superior empirical performance over widely-used stochastic optimization algorithms. Looking forward, several promising directions for future research emerge. First, our incremental stochastic MM framework naturally accommodates extensions beyond Gaussian experts, allowing for more flexible modeling with non-Gaussian components. This generalization opens the door to integrating deep neural network architectures within each expert, enabling powerful hybrid models that combine statistical rigor with the representational capacity of deep learning. Second, while the current formulation processes one data point per iteration, extending the algorithm to support mini-batch updates could improve both stability and scalability, particularly in high-throughput or distributed environments. Developing a theoretically grounded mini-batch MM algorithm with provable convergence guarantees represents an important step toward scalable incremental stochastic learning. Third, addressing the challenges posed by high-dimensional datasets requires incorporating dimensionality reduction techniques or sparsity-inducing regularization within the MM framework to mitigate overfitting and enhance computational efficiency. Lastly, extending the algorithm to handle mixed-type data, including both quantitative and qualitative covariates and parameters, would broaden its applicability to diverse real world problems. This could be achieved by designing expert components capable of modeling discrete structures or employing data transformation techniques compatible with the MM principle. Collectively, these avenues underscore the potential of the proposed incremental stochastic MM framework as a versatile and foundational tool in modern statistical and machine learning pipelines.

\section*{Acknowledgments} TrungTin Nguyen, Hien Duy Nguyen, Florence Forbes, and Gersende Fort acknowledge funding from the Australian Research Council grant DP230100905, and from Inria Project WOMBAT. Christopher Drovandi and TrungTin Nguyen acknowledge funding from the Australian Research Council Centre of Excellence for the Mathematical Analysis of Cellular Systems (CE230100001). Christopher Drovandi was also supported by an Australian Research Council Future Fellowship
(FT210100260).


\appendix

\begin{center}
\textbf{{\Large Supplementary Materials for\\``Revisiting Incremental Stochastic Majorization–Minimization Algorithms with Applications to Mixture of Experts''}}
\end{center}

In this supplementary material, we present the proofs of the main results in \cref{sec_proofs_main_results} and the technical proofs and results in \cref{sec_technical_proof} and \cref{sec_technical_results}, respectively.


\section{Proofs of main results}\label{sec_proofs_main_results}
We start with a preliminary lemma
\begin{lemma}\label{lem: gradientofF}
Under \cref{eq_minorizer_fct}, \Cref{assumptionA1} and \Cref{assumptionObjectiveC1},
    it holds, 
    \begin{enumerate}
    \item for all $\vtau \in \sT$ and $ \rvz \in \sZ$,  $
\nabla_\vtheta f(\cdot;\rvz) \vert_{\vtheta=\vtau}
  + \nabla \psi(\vtau)
  - \nabla \phib(\vtau)^\top \bars(\vtau;\rvz)  = \zero .
$
        \item for all $\vtau \in \sT$, $
\nabla_\vtheta \PE\left[f(\cdot;\rvz)\right] \vert_{\vtheta=\vtau}
  + \nabla \psi(\vtau)
  - \nabla \phib(\vtau)^\top \PE\left[\bars(\vtau;\rvz) \right] = \zero .
$
    \end{enumerate}
\end{lemma}
\begin{proof}
Let $\rvz \in \sZ$, and $\vtau\in \sT$. Applying \cref{eq_minorizer_fct} and \cref{eq_exponential_family}, we have
\begin{equation} \label{eq:majorizingtangent}
    f(\cdot; \rvz) - f(\vtau; \rvz)  
        - \psi\left(\vtau\right) + \psi\left(\cdot\right) + \left[ \phib(\vtau)-\phib(\cdot)\right]^\top\bars(\vtau;\rvz)  \leq 0
\end{equation}  
and this function is equal to zero at $\vtau$. This implies that the derivative at $\vtau$ is zero, thus yielding the first property. 

We assumed that the objective function is finite on $\sT$ (see \Cref{sec_original_optimization}), i.e.
for every $\vtheta \in \sT$, $\PE\big[|f(\vtheta;\rvz)|\big] < \infty$.  
Take the expectation w.r.t. $\rvz$ in \cref{eq:majorizingtangent}; observing again that the function
\begin{equation*} 
    \PE\left[ f(\cdot; \rvz) \right]-  \PE\left[f(\vtau; \rvz) \right] 
        - \psi\left(\vtau\right) + \psi\left(\cdot\right) + \left[ \phib(\vtau)-\phib(\cdot)\right]^\top \,  \PE\left[\bars(\vtau;\rvz) \right] \leq 0
\end{equation*} 
is non-positive on 
$\sT$ and equal to zero at $\vtau$, we conclude that its derivative is zero at $\vtau$, which yields the second property. 
\end{proof}

\subsection{Proof of \texorpdfstring{\cref{proposition_stationary}}{Proposition \ref*{proposition_stationary}}}\label{proof_proposition_stationary}
Let $\vs^0 \in \sF$. Apply \cref{lem: gradientofF} with $\vtau \leftarrow \barparam(\vs^0)$. It then follows that 
\begin{align}
\nabla_\vtheta \PE\left[f(\cdot;\rvz)\right] \vert_{\vtheta = \barparam(\vs^0)}
  + \nabla \psi(\barparam(\vs^0))
  - \nabla \phib(\barparam(\vs^0))^\top \vs^{0} = \zero,\nn
\end{align}
which implies $\barparam(\vs^0) \in \sL$ by \cref{assumptionA4}. Conversely, let $\vtheta^{0} \in \sL$;  then by \cref{lem: gradientofF}, we have 
\[
\nabla \psi(\vtheta^{0}) - \nabla \phib(\vtheta^{0})^\top \PE\left[\bars(\vtheta^{0};\rvz) \right] = \zero,
\]
which, by \cref{assumptionA3} and \Cref{assumptionA4}, implies that
$\vtheta^{0} = \barparam\left(\PE\left[\bars(\vtheta^{0};\rvz) \right]\right)$.
By definition of $\vs^0$, this yields $\vtheta^{0} = \barparam(\vs^0)$, from which it follows that $\vs^0  =\PE\left[ \bars\left(\barparam(\vs^0);\rvz\right)\right]$; \ie\ $\vs^0 \in \sF$.

\subsection{Proof of \texorpdfstring{\cref{prop_Lyapunov_function}}{Proposition \ref*{prop_Lyapunov_function}}}\label{proof_prop_Lyapunov_function}

Recall that
\[
  \lyap(\rvs) \eqdef \PE\big[f(\barparam(\vs); \rvz)\big],
  \qquad
  \veta(\vs) \eqdef \PE\big[\bars(\barparam(\vs);\rvz)\big] - \rvs,
\]
and that, by definition of $\barparam(\vs)$,
\begin{equation}
  \label{eq_global_maximum_prop2}
  \barparam(\vs) 
  \eqdef \argmin_{\vtheta \in \sT} h(\vs;\vtheta),
  \qquad 
  h(\vs;\vtheta) \eqdef -\psi(\vtheta) + \pscal{\vs}{\phib(\vtheta)},
\end{equation}
so that the first–order optimality condition reads
\begin{equation}
  \label{eq_stationarity_barparam}
  -\nabla \psi(\barparam(\vs)) + \nabla \phib(\barparam(\vs))^\top \vs = \zero,
  \qquad \vs \in \sS.
\end{equation}

\paragraph{(a) Differentiability of $V$.}
By ~\cref{assumptionA6} and \Cref{assumptionObjectiveC1}, $\lyap$ is continuously differentiable on $\sS$.

\paragraph{(b) Sign of $\langle \nabla \lyap(\vs), \veta(\vs)\rangle$}

Using \cref{lem: gradientofF} with $\vtau = \barparam(\vs)$ gives
\begin{equation}
  \label{eq_consistency_tool2}
  \nabla_\vtheta \PE\big[f(\barparam(\vs);\rvz)\big]
  = - (\nabla \psi)(\barparam(\vs))
    + (\nabla \phib)(\barparam(\vs))^\top \PE\big[\bars(\barparam(\vs);\rvz)\big].
\end{equation}
Therefore, by \cref{eq_stationarity_barparam} and \Cref{eq_consistency_tool2}, 
\begin{equation}
    \label{eq:GdtObjective}
    \nabla_\vtheta \PE\big[f(\cdot;\rvz)\big] \vert_{\vtheta = \barparam(\vs)} = \big\{(\nabla \phib)(\barparam(\vs))\big\}^\top\veta(\vs). 
\end{equation}
Applying the chain rule to $\lyap(\rvs) = \PE[f(\barparam(\vs);\rvz)]$ and using
\cref{eq_consistency_tool2} and \Cref{eq:GdtObjective} yields
\begin{align}
  \nabla \lyap(\vs) 
  = \nabla \barparam(\vs)^\top  \nabla_\vtheta \PE\big[f(\cdot;\rvz)\big] \vert_{\vtheta = \barparam(\vs)}
  = \nabla \barparam(\vs)^\top \big\{(\nabla \phib)(\barparam(\vs))\big\}^\top \, \veta(\vs).
  \label{eq_gradV_m}
\end{align}

Next we differentiate the optimality condition in \cref{eq_stationarity_barparam}
with respect to $\vs$.
Writing \(h(\vs;\vtheta)\) as in \cref{eq_global_maximum_prop2}, \cref{assumptionA4} and \cref{assumptionA5}
imply that, for each fixed $\vs \in \sS$, the function
\(\vtheta \mapsto h(\vs;\vtheta)\) is twice continuously differentiable and has a
unique minimizer at $\barparam(\vs)$, with positive definite Hessian
\(H(\vs) \eqdef \nabla_\vtheta^2 h(\vs;\vtheta)\vert_{\vtheta = \barparam(\vs)}\).
Differentiating \(\rvs \mapsto \nabla_\vtheta h(\vs;\vtheta)\vert_{\vtheta=\barparam(\vs)} = \zero\) with respect to \(\rvs\) gives
\[
  \nabla_\vtheta^2 h(\vs;\vtheta)\big\vert_{\vtheta=\barparam(\vs)}\, \nabla \barparam(\vs)
  + \nabla_s \big[\nabla_\vtheta h(\vs;\vtheta)\big]\big\vert_{\vtheta=\barparam(\vs)} = \zero.
\]
Using
\(
  \nabla_\vtheta h(\vs;\vtheta) = -\nabla \psi(\vtheta) + \nabla \phib(\vtheta)^\top \vs
\)
and the fact that \(\nabla_s s^\top\) is the identity matrix, we obtain
\[
  \nabla_s \big[\nabla_\vtheta h(\vs;\vtheta)\big]\big\vert_{\vtheta=\barparam(\vs)}
  = \nabla_\vtheta \phib(\barparam(\vs))^\top.
\]
Hence 
\begin{equation}
  \label{prop2_iden0}
  \nabla \barparam(\vs)^\top
  = -\nabla \phib(\barparam(\vs))
    H(\vs)^{-1}.
\end{equation}
Plugging \cref{prop2_iden0} into \cref{eq_gradV_m} yields
\begin{align}
  \label{eq_prop2_desired_result}
  \langle \nabla \lyap(\vs), \veta(\vs)\rangle
  &= \veta(\vs)^\top \nabla \lyap(\vs) \nonumber\\
  &= -\veta(\vs)^\top (\nabla \phib)(\barparam(\vs))  H(\vs)^{-1}
     \{(\nabla \phib)(\barparam(\vs)) \}^\top  \veta(\vs) \nonumber\\
  &= -\Big( \{(\nabla \phib)(\barparam(\vs)) \}^\top \veta(\vs)\Big)^\top
     H(\vs)^{-1}
     \Big( \{(\nabla \phib)(\barparam(\vs))\}^\top \veta(\vs)\Big).
\end{align}
By~\cref{assumptionA4},
$H(\vs)$ is positive definite,
so its inverse is also positive definite. Consequently,
\(\langle \nabla \lyap(\vs), \veta(\vs)\rangle \le 0\) for all \(\vs \in \sS\), as claimed.

 If \(\vs \in \sF\), then \(\veta(\vs) =\zero\) by definition, and
\(\langle \nabla \lyap(\vs), \veta(\vs)\rangle = 0\) follows immediately.

\subsection{Proof of \texorpdfstring{\cref{theorem_consistency}}{Theorem \ref*{theorem_consistency}}} \label{proof_proposition_consistency}
{\bf Definitions.} Define the filtration \(\{\mathcal{F}_n,\, n \ge 0\}\) by
\[
\mathcal{F}_n \eqdef \sigma(\vs_0,\, \rvz_1,\, \ldots,\, \rvz_n).
\]
Then \(\vs_n \in \mathcal{F}_n\) for all \(n \ge 0\).  
Define the martingale increment 
\begin{align}
\xi_{n+1}\eqdef\bars\left(\barparam\left(\vs_n\right);\vz_{n+1}\right)-\mathbb{E}\left[ \bars\left(\barparam\left(\vs_n\right);\rvz_{n+1}\right)\,\middle|\, \mathcal{F}_n \right].\nn
    \end{align} Note that for any positive measurable function \(\tau\), $\mathbb{E}\!\left[\, \tau(\vs_n, \rvz_{n+1}) \,\middle|\, \mathcal{F}_n \right]
= \mathbb{E}_{\pi}\!\left[\, \tau(\vs, \rvz) \right]_{\vs = \vs_n}$.   \par Our algorithm \cref{eq_s_step} can be rewritten as:
    \begin{align} \label{s_step_with_eps}
        \vs_{n+1} & =\vs_{n}+\gamma_{n+1}\left\{ \veta(\vs_n)+\xi_{n+1}\right\}.
    \end{align} The proof consists in checking the assumptions of \cite[Theorem 2.3]{andrieu_stability_2005}, which provide a general result for the almost-sure convergence of a stable Stochastic Approximation algorithm.

{\bf Checking the assumption A1 in \cite{andrieu_stability_2005}.}  It is verified under  \Cref{assumptionA9}, \Cref{assumptionA8},   \Cref{assumption_compact_sublevel} and 
by \Cref{prop_Lyapunov_function}.

{\bf A stable algorithm w.p.1.} By \Cref{assumptionA10}, with probability one,  the sequence $\{\vs_n, n \geq 0\}$ remains in a compact set $\mathbb{K}$ of $\sS$ ($\mathbb{K}$ is path-dependent). Hence, the stability condition is verified. Furthermore,  $\mathbb{K} \cap \sF \neq \emptyset$.

{\bf The conditions on the step size sequence.} By \Cref{assumptionA11}, the step size sequence is monotone nonincreasing, $\lim_k \gamma_k = 0$ and $\sum_k \gamma_k = +\infty$. For any $\lambda_0>0$, there exists $K$  such that $\gamma_K \leq \lambda_0$; since we are interested in the convergence when $k \to \infty$, we can assume without loss of generality that the algorithm is studied for all $k \geq K$. 

{\bf Show that with probability one, $\limsup_n \sup_{k \geq n} \| \sum_{j=n}^k \gamma_j \xi_j\| =0$.} We prove the equivalent property: with probability one, $\lim_n M_n$ exists (and is finite), where $M_n \eqdef \sum_{j = 1}^n \gamma_j \xi_j$.  

Let $\{\mathbb{K}_q, q \geq 1\}$ be a sequence of compact sets of $\sS$ such that $\mathbb{K}_q \subseteq  \mathbb{K}_{q+1}$  and $\bigcup_{q \in \nset} \mathbb{K}_q = \sS$; for example, set $\mathbb{K}_q \eqdef \left\{\vs \in \sS: \|\vs\| \leq q, d(\vs, \sS^c) \geq \frac{1}{q} \right\}$.  By \Cref{assumptionA10}, 
\begin{equation}\label{eq:prop3:tool4}
\mathbb{P}\left( \bigcup_{q \in \nset} \{\forall n \geq 0, \vs_n \in \mathbb{K}_q \}\right)= 1.
\end{equation}
We write
\begin{align}
\mathbb{P} \left( \lim_n M_n \ \text{exists} \right) & = \mathbb{P} \left( \{ \lim_n M_n \ \text{exists} \} \cap  \bigcup_{q \in \nset} \{\forall n \geq 0, \vs_n \in \mathbb{K}_q \}  \right)  \nonumber \\
& = \lim_q \uparrow \mathbb{P} \left( \{ \lim_n M_n \ \text{exists} \} \cap  \{\forall n \geq 0, \vs_n \in \mathbb{K}_q \}  \right) \nonumber \\
& = \lim_q \uparrow \mathbb{P} \left( \{ \lim_n M_{n,q} \ \text{exists} \} \cap  \{\forall n \geq 0, \vs_n \in \mathbb{K}_q \}  \right) \label{eq:prop3:tool1}
\end{align}
where $M_{n,q} \eqdef \sum_{j=1}^n \gamma_j \xi_j \1_{\vs_{j-1} \in \mathbb{K}_q}$. For fixed $q \geq 0$, the random variables $\{M_{n,q}, n \geq 1 \}$ is a $\F_n$-adapted martingale sequence, upon noting that by definition of $\xi_n$, $\E\left[ \xi_n \1_{\vs_{n-1} \in \mathbb{K}_q} \vert \F_{n-1} \right] =0$. By \Cref{assumptionA7}, it is a square-integrable martingale:
\[
\E\left[ \| M_{n,q} \|^2 \right] = \sum_{j=1}^n \gamma_j^2 \E\left[ \| \xi_j \|^2 \1_{\vs_{j-1} \in \mathbb{K}_q} \right] \leq  \sup_{\vs \in \mathbb{K}_q} \E_\pi\left[ \| \bars\left(\barparam\left(\vs\right);\rvz\right)\|^2 \right] \, \sum_{j=1}^n \gamma_j^2 < \infty.
\]
Therefore, by \cite[Theorem 2.15]{hall:heyde} 
\begin{equation}\label{eq:prop3:tool2}
\{\lim_n M_{n,q} \ \text{exists} \} \supseteq \{ \sum_{j \geq 1} \gamma_j^2  \E\left[ \|\xi_j \|^2 \1_{ \vs_{j-1} \in \mathbb{K}_q} \vert \F_{j-1} \right]< \infty \}.
\end{equation}
Using again  \Cref{assumptionA7} and since $\sum_j \gamma_j^2< \infty$ (see \Cref{assumptionA11}), we have 
\begin{equation}\label{eq:prop3:tool3}
\mathbb{P}\left( \sum_{j \geq 1} \gamma_j^2  \E\left[ \|\xi_j \|^2 \1_{ \vs_{j-1} \in \mathbb{K}_q}  \vert \F_{j-1} \right]< \infty\right) =1.
\end{equation}
        By  \cref{eq:prop3:tool1}, \cref{eq:prop3:tool2}, \cref{eq:prop3:tool3} and \cref{eq:prop3:tool4}, we have
\[
\mathbb{P} \left( \lim_n M_n \ \text{exists} \right)  \geq \lim_q \uparrow \mathbb{P} \left( \forall n, \vs_n \in \mathbb{K}_q \right) =1.
\]
The proof is concluded.

{\bf Convergence of the sequence $\{\vs_n, n \geq 0\}$.} By \cite[Theorem 2.3]{andrieu_stability_2005},  with probability one, $\limsup_n d(\vs_n, \{\pscal{\nabla V(\vs)}{\veta(\vs)} =0 \} \cap \mathbb{K}) =0.$ By \Cref{assumptionA9}, this yields $\limsup_n d(\vs_n, \sF\cap \mathbb{K})=0$. 

{\bf Convergence of the sequence $\{\vtheta_n, n \geq 0\}$.} We prove an inclusion of sets: $\{ \limsup_n d(\vs_n, \sF) =0 \} \subseteq \{\limsup_n d(\vtheta_n, \sL)= 0\}$; since the set on the LHS is almost-sure, this will imply that the set on the RHS is almost-sure.
Let $\omega \in \{ \omega: \limsup_n d(\vs_n(\omega), \sF)=0\}$. Set $\vtheta_n(\omega) \eqdef \barparam(\vs_n(\omega))$. Let $\{\vtheta_{n_q}(\omega), q \geq 0 \}$ be a converging subsequence with limiting value $\vtheta_\star$; we have to prove that $\vtheta_\star \in \sL$.

From the subsequence $\{\vs_{n_q}(\omega),q \geq 0\}$, we extract a converging subsequence  $\{\vs_{n_{\rho(q)}}(\omega),q \geq 0\}$ with limit $\vs_\star$; by \Cref{assumptionA10}, such a subsequence exists and $\vs_\star \in \sS$. Since $\omega \in \{\limsup_n d(\vs_n, \sF) =0 \}$ and $\sF$ is closed (as a consequence of \Cref{assumptionA9}), then $\vs_\star \in \sF$. By \Cref{assumptionA6}, $\barparam$ is continuous so that $\vtheta_\star =\lim_q \barparam(\vs_{n_q}(\omega))= \lim_q \barparam(\vs_{n_{\rho(q)}}(\omega)) = \barparam(\vs_\star)$. Therefore, by \Cref{proposition_stationary}, $\vtheta_\star \in \sL$. This concludes the proof.

\subsection{Proof of \texorpdfstring{\cref{theorem_onlineMM_MoE}}{Theorem \ref*{theorem_onlineMM_MoE}}}\label{proof_theorem_onlineMM_MoE}
\subsubsection{Proof of exponential family majorizer surrogate via \texorpdfstring{\cref{eq_surrogate_Taylor_SGGMoE}}{Equation (\ref*{eq_surrogate_Taylor_SGGMoE})} in \texorpdfstring{\cref{proposition_surrogate_construction_SGMoE}}{Proposition \ref*{proposition_surrogate_construction_SGMoE}}}\label{sec_exponential_family_surrogate}
\par Denote:
\begin{align*}
    \underline{\sfg}(\vomega;\vx)=\log\left(1+\displaystyle\sum_{k=1}^{K-1}\exp\left(w_k(\vx)\right)\right) \text{ where } w_k(\vx)=\displaystyle\sum_{d=0}^{D_W}\vomega_{k,d}^{\top}\vx^{d}.
\end{align*}
\par We have $\underline{\sfg}$ is $C^2$, then Taylor's theorem with Lagrange remainder implies that there exists some $\alpha \in (0,1)$ such that
\begin{align}\label{eq_Taylor}
    \underline{\sfg}(\vomega;\vx)-\underline{\sfg}(\vomega_n;\vx)=\left(\vomega-\vomega_n\right)^{\top}\nabla  \underline{\sfg} \left(\vomega_n;\vx\right)+\frac{1}{2}\left(\vomega-\vomega_n\right)^{\top}\nabla^{2}\underline{\sfg}\left(\vomega_n+\alpha\left(\vomega-\vomega_n\right);\vx_n\right)\left(\vomega-\vomega_n\right).
\end{align}

{\bf Calculate the gradient of $\underline{\sfg}$.} For some $k \in [K-1]$, $p\in [P]$ and $d \in \{0,1,\dots,D_W\}$, the gradient can be calculate as follow:
\begin{align*}
    \nabla_{\evomega_{k,d,p}}\underline{\sfg}=\frac{\evx^{d}_p\exp\left(\displaystyle\sum_{d=0}^{D_W}\vomega_{k,d}^{\top}\vx^{d}\right)}{1+\displaystyle\sum_{l=1}^{K-1}\exp\left(\displaystyle\sum_{d=0}^{D_W}\vomega_{l,d}^{\top}\vx^{d}\right)}=\evx^{d}_p\sfg_k(\vw(\vx)).
\end{align*}
\par Hence, the gradient of $\underline{\sfg}$ is given as:
\begin{align*}
    \nabla \underline{\sfg}
	=&\left[x_1^{0} \sfg_1(\vomega(\xb)),\dots,x_1^{D_W} \sfg_1(\vomega(\xb)),\ldots,x_{P}^{0}\sfg_{K-1}(\vomega(\xb)),\dots,x_{P}^{D_W}\sfg_{K-1}(\vomega(\xb))\right]^{\top}=\hat{\vg}\otimes\hat{\vx},
\end{align*}
\par where $\hat{\vx}=\left[x_1^{0},\dots,x_1^{D_W},\dots,x_P^{0},\dots,x_P^{D_W}\right]^{\top}$, $\hat{\vg}=\left[ \sfg_1(\vw(\vx)),\dots,\sfg_{K-1}(\vw(\vx))\right]^{\top}$ and $\otimes$ is the Kronecker product.\\
{\bf Calculate for the Hessian the gradient of $\underline{\sfg}$.} For some $k_1, k_2 \in [K-1]$, $d_1,d_2 \in \{0,1,\dots,D_W\}$ and $p_1, p_2\in [P]$, the Hessian can be calculated as follow:
\begin{align*}
    \nabla_{\evomega_{k_1,d_1,p_1}\evomega_{k_2,d_2,p_2}}\underline{\sfg}&=\frac{-\evx^{d_1}_{p_1}\exp(w_{k_1}(\vx))\evx^{d_2}_{p_2}\exp(w_{k_2}(\vx))}{\left[1+\displaystyle\sum_{l=1}^{K-1}\exp(w_l(\vx))\right]^2}=-\evx^{d_1}_{p_1}\evx^{d_2}_{p_2}\sfg_{k_1}(\vw(\vx))\sfg_{k_2}(\vw(\vx)) \text{ for $k_1 \neq k_2$,}\\
\nabla_{\evomega_{k_1,d_1,p_1}\evomega_{k_1,d_2,p_2}}\underline{\sfg}&=\frac{\evu^{d_1}_{p_1}\evx^{d_2}_{p_2}\exp(w_{k_1}(\vx))\left(\displaystyle\sum_{l=1}^{K}\exp(w_l(\vx))-\exp(w_{k_1}(\vx))\right)}{\left[1+\displaystyle\sum_{l=1}^{K-1}\exp(w_l(\vx))\right]^2}\nn\\
    &=\evu^{d_1}_{p_1}\evx^{d_2}_{p_2}\sfg_{k_1}(\vw(\vx))(1-\sfg_{k_1}(\vw(\vx))).
\end{align*}
Hence, the Hessian is given as:
\begin{align*}
    \nabla^{2}\underline{\sfg}=\left(\text{diag}(\hat{\vg})-\hat{\vg}\hat{\vg}^{\top}\right)\otimes\hat{\vx}\hat{\vx}^{\top}.
\end{align*}
\begin{lemma}[See, \eg~ \cite{bohning1992multinomial}]\label{Kronecker_prod_property}
    If $\mA \succeq \mB$ then for symmetric, nonnegative definite $\mC$:
    \begin{align*}
        \mA \otimes \mC \succeq \mB \otimes \mC.
    \end{align*}
\end{lemma}

\begin{proposition} \label{prop_theo3}
 Let $\pi_k$ be be real values in $(0,1)$ for $k \in [K+1]$ and $\displaystyle\sum_{k=1}^{K+1}\pi_k=1$. 
 Denote $\hat{\pib}=[\pi_1,\pi_2,\dots,\pi_K]^{\top}$. We have:
\begin{align*}
    \text{diag}(\hat{\pib})-\hat{\pib}\hat{\pib}^{\top}\leq \left(\frac{3}{4}\mI_{K-1}-\frac{\bm{1}_{K-1}\bm{1}_{K-1}^{\top}}{2(K-1)}\right).
\end{align*}
\end{proposition}
The proofs of \cref{Kronecker_prod_property} are provided in \cref{lemma3_proof}, while the proofs of \cref{prop_theo3} are detailed in \cref{proof_prop_theo3}.

{\bf Proof of exponential family majorizer surrogate via \cref{eq_surrogate_Taylor_SGGMoE}.} Applying \cref{Kronecker_prod_property} and \cref{prop_theo3} to \cref{eq_Taylor} we receive the majorization:
\begin{align*}
    \underline{\sfg}(\vomega;\vx_n) &\leq \underline{\sfg}(\vomega_n;\vx_n)+ \left(\vomega-\vomega_n\right)^{\top}\nabla \underline{\sfg}(\vomega_n;\vx_n)+\frac{1}{2}\left(\vomega-\vomega_n\right)^{\top}\mB_{n,K}\left(\vomega-\vomega_n\right).
\end{align*}
Hence, given $\vomega
:=\left(\omegab_{k,d,p}\vert_{p\in [P], k\in[K-1], d\in\{0,\ldots,D_W\}}\right)$, the overall majorization function is defined:
\begin{align*}
    -\log s_{\vtheta}(\vy_n\mid \vx_n)
    &\le \sum_{k=1}^{K}\tau_{n,k}\log(\tau_{n,k})-\sum_{k=1}^{K}\tau_{n,k}w_k(\vx_n)-\sum_{k=1}^{K}\tau_{n,k}\left[\log(\cN(\vy_n;\vmu_k(\vx_n),\mSigma_k))\right]\\
    &\quad +\left\{\underline{\sfg}(\vomega_n;\vx_n)+ \left(\vomega-\vomega_n\right)^{\top}\nabla \underline{\sfg}(\vomega_n;\vx_n)+\frac{1}{2}\left(\vomega-\vomega_n\right)^{\top}\mB_{n,K}\left(\vomega-\vomega_n\right)\right\}, 
   \\
     &=\sum_{k=1}^{K}\tau_{n,k}\log(\tau_{n,k})+\underline{\sfg}(\vomega_n;\vx_n)-\vomega_{n}^{\top}\nabla \underline{\sfg}(\vomega_n;\vx_n)-\sum_{k=1}^{K}\tau_{n,k}w_k(\vx_n)\\
    &\quad +\left\{\vomega^{\top}\nabla \underline{\sfg}(\vomega_n;\vx_n)+\frac{1}{2}(\vomega-\vomega_{n}^{\top})\mB_{n,K}\left(\vomega-\vomega_n\right)\right\} 
   \\
    &\quad -\sum_{k=1}^{K}\tau_{n,k}\left[\log(\cN(\vy_n;\vmu_k(\vx_n),\mSigma_k))\right]=:-\log[g(\vtheta, \vx_{n},\vy_{n}; \vtheta_{n})].  
\end{align*}

\subsubsection{Verification of assumptions in \texorpdfstring{\cref{theorem_onlineMM_MoE}}{Theorem \ref*{theorem_onlineMM_MoE}}}\label{proof_assumption_theorem_onlineMM_MoE}

To verify the model-specific assumptions
\cref{assumptionA1,assumptionA2,assumptionA3,assumptionA4,assumptionObjectiveC1,assumptionA5,assumptionA6,assumptionA7}, we will assume that $\vx$ and $\vy$ are bounded by $\evu^\#$ and $y^\#$, respectively. 

\textbf{Verification of \cref{assumptionA1}.}
    For the problem of softmax-gated MoE, we consider the majorizer surrogate $g$ to have the following form that belongs to the exponential family
\begin{equation} \label{eq_exponential_family_appendix}
g\left(\vtheta,\vz;\vtau\right)\eqdef - \psi\left(\vtheta\right)+ \pscal{\bars(\vtau;\vz)}{\phib(\vtheta)},
\end{equation}
where,
\begin{align*}
    &\phib(\vtheta):=
    \begin{bmatrix}
      \vomega\\
      \vect(\vomega\vomega^{\top})\\
      \vkappa\\
      \vzeta \\
      \mDelta\\
      \bar{\vSigma}
    \end{bmatrix}, ~\bars(\vtau;\vz)=
    \begin{bmatrix}
      -\vxi+\nabla f(\vtau)-\mB\vtau \\
      \frac{1}{2}\vect(\mB)\\
      \vtau\otimes \vy^{2} \\
      -2\vtau\otimes(\vy\otimes\vr) \\      \vtau\otimes\left(\vect\left(\vr\vr^{\top}\right)\otimes \mathbf{1}_Q\right)\\
      \vtau\otimes \mathbf{1}_Q
    \end{bmatrix},~\psi(\vtheta)=C.
\end{align*}
Here, $C$ represents the constants that are independent of $\vomega$ and $\vkappa, \vzeta,\mDelta,\bar{\vSigma}$ and $\vtau$ are defined as follows:
\begin{align*}
    \bar{\vSigma}&=[\log\sigma_{1,1}^{2},\dots, \log\sigma_{1,Q}^2, \log\sigma_{2,1}^{2},\dots,\log\sigma_{K,Q}^{2}]^{\top},~\vkappa=\left[\frac{1}{\sigma_{1,1}^{2}},\dots,\frac{1}{\sigma_{1,Q}^2},\frac{1}{\sigma_{2,1}^{2}},\dots,\frac{1}{\sigma_{K,Q}^{2}}\right]^{\top},\\
    \vzeta&=\vect\left(\left[\frac{\mUpsilon_{1,1}^{\top}}{\sigma_{1,1}^{2}},\dots,\frac{\mUpsilon_{1,Q}^{\top}}{\sigma_{1,Q}^{2}},\frac{\mUpsilon_{2,1}^{\top}}{\sigma_{2,1}^{2}},\dots,\frac{\mUpsilon_{K,Q}^{\top}}{\sigma_{K,Q}^{2}}\right]\right),~
    \vtau=[\tau_{1},\tau_{2},\dots,\tau_{K}]^{\top},\\
    \mDelta&=\vect\left(\frac{\vect(\mUpsilon_{1,1}^{\top}\mUpsilon_{1,1})}{\sigma_{1,1}^{2}},\dots,\frac{\vect(\mUpsilon_{1,Q}^{\top}\mUpsilon_{1,Q})}{\sigma_{1,Q}^{2}},\frac{\vect(\mUpsilon_{2,1}^{\top}\mUpsilon_{2,1})}{\sigma_{2,1}^{2}},\dots,\frac{\vect(\mUpsilon_{K,Q}^{\top}\mUpsilon_{K,Q})}{\sigma_{k,q}^2}\right).
\end{align*}

Furthermore, since $\phib,\psi$ and $\bars$ are composed of elementary functions applied component-wise to their parameters, they are trivially continuously differentiable and measurable due to the smoothness and measurability properties of elementary functions.

\textbf{Verification of \cref{assumptionA2}.} For $\mB_{n,K}$ defined in \cref{eq_define_B_n_with_noise}, it can be proven that for any eigenvalue $\lambda_{\mB_{n,K}}$ of $\mB_{n,K}$, $\lambda_{\mB_{n,K}}$ is bounded in the interval $\left(\dfrac{1}{M_0},M_0\right)$ for some positive real number $M_0$. Additionally, since $\lambda^{max}_{\mB_{n,K}^{-1}}=\dfrac{1}{\lambda^{min}_{\mB_{n,K}}}<\dfrac{1}{M_0}$ and $\lambda^{min}_{\mB_{n,K}^{-1}}=\dfrac{1}{\lambda^{max}_{\mB_{n,K}}}>\dfrac{1}{M_0}$. Thus, for all eigenvalues $\lambda_{\mB_{n,K}^{-1}}$ of $\mB_{n,K}^{-1}$, $\lambda_{\mB_{n,K}^{-1}}\in\left(\dfrac{1}{M_0},M_0\right)$.
\par Knowing this value $M_0$, we can define $\mathbb{S}$ as follows:
\begin{align}\label{eq_define_S2}
    \mathbb{S}=&\big\{(\vs_1, \vect(\mS_2),\vs_3, \vs_4, \vs_5, \vs_6):  \vs_1\in \sR^{(K-1)P(D_W+1)};\nonumber\\
    &\quad \mS_2\in \sR^{(K-1)P(D_W+1)\times (K-1)P(D_W+1)}, \mS_2 \succeq \zero, \lambda \in \left(\dfrac{1}{M_0},M_0\right) \text{ for all eigenvalues $\lambda$ of } \mS_2;\nonumber\\
    &\quad \vs_3 \in \sR^{KQ}; \quad \vs_4\in \sR^{KPQ(D_V+1)}; \quad \vs_6\in\sR_{+}^{KQ};\nonumber\\
    &\quad \vs_5\in \sR^{KQ(P(D_V+1))^2}, \mat(\vs_{5,k:q}) \succeq \zero \text{ for all $k\in[K],q\in [Q]$}\big\}.
\end{align}

It is trivial that $\mathbb{S}$ denoted above satisfies the requirements of \cref{assumptionA2}.

\textbf{Verification of \cref{assumptionA3}.} $\bars(\vtheta,\vz)$ is composed of elementary functions applied component-wise to $\vx$ and $\vy$, hence its expectation exist and lies in $\mathbb{S}$. Furthermore, since $\vx$ and $\vy$ are bounded, it can be shown that $\vtau,\mB,\vr,\vxi$ and $\nabla \underline{\sfg}$ are all bounded. Thus, the expectation $\PE[\bars(\vtheta; \rvz)]$ exists, lies within $\sS$, for any $\vtheta \in \sT$.

\textbf{Verification of \cref{assumptionA4}.} This assumption is proven in \cref{param update} where we have shown that $\vomega,\mUpsilon,\vsigma^2$ of the next iteration are the unique solution to the minimization problem of the current iteration. Indeed, since the objective is strictly convex in each block, the minimizer is unique and measurable.

\textbf{Verification of \cref{assumptionObjectiveC1}.} Since $f(\vtheta;\vz)$ is continuously differentiable in $\vtheta$ for all $\vz$
and dominated by an integrable envelope (due to bounded covariates and responses),
differentiation under the expectation is valid, hence
$\vtheta \mapsto \E[f(\vtheta;\rvz)]$ is $C^1$ on $\sT$.

\textbf{Verification of \cref{assumptionA5}.} Since we have $\vomega \in \sR^{P\times K\times (D_W+1)},\mUpsilon\in \sR^{K\times (D_V+1)\times Q\times P}$ and $\vsigma^2 \in \sR_{+}^{K\times Q\times Q}$, thus $~\mathbb{T}=\sR^{P\times K\times (D_W+1)}\times \sR^{K\times (D_V+1)\times Q\times P}\times\sR_{+}^{K\times Q\times Q}$ is a convex and open set. Furthermore, similar to the verification of \cref{assumptionA1}, we can show that $\psi$ and $\phib$ is twice continuously differentiable throughout $\mathbb{T}$.

\textbf{Verification of \cref{assumptionA6}.} From \cref{param update}, we have the function $\bar{\vtheta}(\vs)$ is defined as:
\begin{align*}
    \bar{\vtheta}(\vs)=\begin{bmatrix}
      &-\left(2\mat(\vs_{2}) \right)^{-1} \vs_{1}\\
      -&\left(\mat\left(\vs_{5,k:q}\right)+\mat\left(\vs_{5,k:q}\right)^{\top}\right)^{-1}\vs_{4,k:q},\\
      &\dfrac{\vs_{3,kq}+\vs_{4,k:q}^{\top}\mUpsilon_{k,q,:}^{\top}+\vs_{5,k:q}^{\top}\vect(\mUpsilon_{k,q,:}^{\top}\mUpsilon_{k,q,:})}{s_{6,kq}}
    \end{bmatrix}, \quad \vs=(\vs_1,\vs_2,\vs_3,\vs_4,\vs_5,\vs_6).
\end{align*}
\par It is trivial to show that $\bar{\vtheta}(\vs)$ is continuously differentiable for $\vs_1, \vs_3, \vs_4, \vs_6$. For $\vs_2, \vs_5$, we have $\vv\rightarrow \mat(\vv)+\mat(\vv)^{\top}$ is an affine transformation, thus it is continuously differentiable. Furthermore, from \cref{param update}, we know that $2\mat(\vs_2)$ and $\mat(\vs_5)+\mat(\vs_5)^{\top}$ are positive definite, thus their inverts are non-singular, thus their inverts are continuously differentiable. This concludes our proof.

\textbf{Verification of \cref{assumptionA7}.}  By applying \cref{lemma_assump_A7} in addition with the fact that $\vkappa, \vzeta,\mDelta,\bar{\vSigma}$ (defined in \cref{section_construction_MM_Oexpert}) are all bounded. We can conclude that for all compact set $\mathbb{K}\subset\mathbb{S}$, $\bars(\bar{\vtheta}(\vs),\vx)$ is bounded by a number $\beta$. Hence,
\begin{align*}
    \sup_{\vs \in \Kset}(\PE\left[ |\bars\left(\barparam(\vs);\rvx\right)|^2\right]) < \beta^2 <\infty.
\end{align*}
\begin{lemma} \label{lemma_assump_A7}
    Let $\mathbb{K}$ be a compact subset of $\mathbb{S}$. Then for all $\vs_1,\vs_2\in \mathbb{K}$, we have $(2\mat(\vs_2))^{-1}\vs_1$ is bounded.
\end{lemma}
\begin{proof}[Proof of \cref{lemma_assump_A7}]
    Since $\vs_2\in\mathbb{K}\subset\mathbb{S}$, we have $\mat(\vs_2)$ eigenvalues are bounded in $\left(\frac{1}{M_0},M_0\right)$. This leads to $\mat(\vs_2)$ being bounded. Furthermore, we have $\vs_1\in\mathbb{K}$ and $\mathbb{K}$ is compact, so $\vs_1$ is also bounded. Combining these two reasons we conclude the proof.
\end{proof}

\textbf{Verification of \cref{assumptionA10}.} First, notice that since $\vx$ and $\vy$ are bounded, $\vkappa,\vxi,\mDelta,\bar{\vSigma}$ are also bounded. Thus by applying \cref{lemma_fir_for_A9}, we can show that $\vs_{i,n}$ is bounded for all $i\in\{2,\dots,6\}$. Additionally, since $\vkappa,\vxi,\mDelta,\bar{\vSigma}$ are bounded, due to how we choose $\mathbb{S}$, for all $i\in\{2,\dots,6\}$, there exist a compact set $\mathbb{J}_i\subset\mathbb{S}$ that satisfy the condition of \cref{lemma_sec_for_A9} for sequence $\vs_i$. Hence, by applying \cref{lemma_sec_for_A9}, we can show that 
    \begin{align*}
        \lim\inf d(\vs_{i,n},\mathbb{S}^c)>0,\quad \forall i \in \{2,\dots,6\}.
    \end{align*}
    \par Since, $\vx$ and $\vy$ are taken from compact sets, we can apply \cref{lemma_sec_for_A9} for $\vs_i, \forall i\in\{2,\dots,6\}$, and thus having
    \begin{align*}
        \lim\inf d(\vs_{i,n},\mathbb{S}^c)>0,\quad \forall i \in \{2,\dots,6\}.
    \end{align*}
\begin{lemma}\label{lemma_fir_for_A9}
    Given a sequence of vectors $(\vv_n)^{+\infty}_{n=0}$ satisfying
    \begin{align*}
        \vv_{n+1}=\vv_n+\gamma_n(\vf(n)-\vv_n),
    \end{align*}
    where $0<\gamma_n<1$. If $\vf$ is bounded (element-wise), then, $\vv_n$ is bounded (element-wise). 
\end{lemma}
    
    \begin{lemma}\label{lemma_sec_for_A9}
        Let $\mathbb{H}$ be a convex, open set. Given a sequence of vectors $(\vv_n)^{+\infty}_{n=0}$ in $\mathbb{H}$ satisfying
    \begin{align*}
        \vv_{n+1}=\vv_n+\gamma_n(\vf(n)-\vv_n),
    \end{align*}
    where $0<\gamma_n<1$. Then, if there exist a compact set $\mathbb{J}\subseteq \mathbb{H}$ such that $\vf(n)\in \mathbb{J}\quad\forall n$, we have
    \begin{align*}
        \lim\inf \{d(\vv_n,\mathbb{S}^c)\}>0.
    \end{align*}
    \end{lemma}

\begin{proof}[Proof of \cref{lemma_fir_for_A9}]
    Let say $\vf$ is bounded (element-wise) by $\valpha$. Denote $\valpha_0=\max(\valpha,\vt_0)$ (element-wise). We will show by induction that $\vv_n$ is upper bounded by $\valpha_0$ (element-wise). We have $\vt_0 \leq \valpha_0$. Let say for some $n$, we have $\vv_n\leq \valpha_0$ (element-wise), then:
    \begin{align*}
        \vv_{n+1}=(1-\gamma_{n+1})\vv_n+\gamma_{n+1}\vf(n)\leq (1-\gamma_{n+1})\valpha_0+\gamma_{n+1}\valpha_0=\valpha_0.
    \end{align*}
    Hence, by induction, we have shown that $\vv_n$ is upper bounded. Doing similarly for the lower bound, we complete the proof.
\end{proof}
    
    \begin{proof}[Proof of \cref{lemma_sec_for_A9}]
        \par Let $\hat{\mathbb{J}}$ be the smallest compact, convex set that contains $\mathbb{J}\cup\{\vt_0\}$ (i.e $\hat{\mathbb{J}}$ is the convex hull of $\mathbb{J}\cup\{\vt_0\}$). We know that $\vt_0\in\mathbb{H}$ and $\mathbb{J}\subseteq\mathbb{H}$, hence $\mathbb{J}\cup\{\vt_0\}\subseteq \mathbb{H}$. Since $\mathbb{H}$ is convex and open, we have $\hat{\mathbb{J}}\subset\mathbb{H}$. We will show by induction that $\vv_n\in\hat{\mathbb{J}}$ for all $n$. We have $\vt_0\in \hat{\mathbb{J}}$. Assume that $\vv_n \in \hat{\mathbb{J}}$, we have:
        \begin{align*}
            \vv_{n+1}=(1-\gamma_{n})\vv_n+\gamma_n \vf(n)
        \end{align*}
        is a convex combination of $\vv_n$ and $\vf(n)$. But since both $\vv_n$ and $\vf(n)$ are elements of $\hat{\mathbb{J}}$,  $\vv_{n+1}\in \hat{\mathbb{J}}$. Hence, by induction we have $\vv_n\in \hat{\mathbb{J}}\quad \forall n$. Lastly, since $\hat{\mathbb{J}}$ is a compact subset of the open set $\mathbb{S}$, there exist $\varepsilon>0$ such that for all $\va\in \hat{\mathbb{J}}$ and $\vb\in \mathbb{S}^c$, $$d(\va,\vb)>\varepsilon.$$ Combining this with the previous proven facts, we yield
        \begin{align*}
            \lim\inf\{d(\vv_n,\mathbb{S}^c)\}>\varepsilon>0.
        \end{align*}
        This completes the proof.
    \end{proof}

\textbf{Verification of \cref{assumptionA11}.} By taking $\gamma_n=\gamma_0 n^{-\alpha}$ with $\alpha\in\left(\tfrac12,1\right]$ and $\gamma_0\in(0,1)$, all requirements of \cref{assumptionA11} are met.

\begin{remark}
    Although formal proofs for \cref{assumptionA10} (in the case of $\vs_{n,1}$) and \cref{assumptionA8,assumptionA9} remain elusive, we adopt these as working assumptions for the MoE model. Nonetheless, our experiments on both synthetic and real-world datasets consistently demonstrate that our proposed method achieves convergence reliably, with a faster convergence rate and higher accuracy compared to the baseline models.
\end{remark}

\subsection{Derivation of the \texorpdfstring{\cref{algorithm_onlineMM_MoE}}{Algorithm \ref*{algorithm_onlineMM_MoE}}}\label{proof_algorithm_onlineMM_MoE}

\subsubsection{Construction of the incremental stochastic MM algorithm for \texorpdfstring{$O_{gate}$}{}}\label{section_construction_MM_Ogate}

\begin{align*}
O_{gate}(\vomega,\vtheta_n) &= -\sum_{k=1}^{K}\tau_{k}w_k(\vx)+\displaystyle\left\{\vomega^{\top}\nabla \underline{\sfg}(\vomega_n;\vx_n)+\frac{1}{2}\left(\vomega-\vomega_n\right)^{\top}\mB_{n,K}\left(\vomega-\vomega_n\right)\right\}.
\end{align*}
Let us denote vector:
\begin{align*}
    \vxi=[\tau_{1}x_1^{0}, \tau_{1}x_1^{1},\dots, \tau_{1}x_1^{D_W},\dots,\tau_{K-1}x_1^{D_W},\dots,\tau_{K-1}x_P^{D_W}]^{\top}.
\end{align*}
Hence:
\begin{align*}
    O_{gate}(\vomega,\vtheta_n)&=\left(-\vxi^{\top}+\nabla \underline{\sfg}(\vomega_n;\vx_n)^{\top}-\vomega_n^{\top}\mB\right)\vomega+\frac{1}{2}\vomega^{\top}\mB\vomega+ C,\\
    \mB&=\left(\frac{3}{4}\mI_{K-1}-\frac{\bm{1}_{K-1}\bm{1}_{K-1}^{\top}}{2(K-1)}\right)\otimes\hat{\vx}\hat{\vx}^{\top},
\end{align*}
 where $C$ represents the constants that are independent of $\vomega$.
 
Here the incremental stochastic MM algorithm can be written as:
\begin{align*}
    &\phib(\vomega):=
    \begin{bmatrix}
      \vomega\\
      \vect(\vomega\vomega^{\top})
    \end{bmatrix}, ~\bars(\vtau;\vz)=
    \begin{bmatrix}
      -\vxi+\nabla \underline{\sfg}(\vtau)-\mB\vtau \\
      \frac{1}{2}\vect(\mB)
    \end{bmatrix},~\psi(\vomega)=C.
\end{align*}
\subsubsection{Construction of the incremental stochastic MM algorithm for \texorpdfstring{$O_{expert}$}{}}\label{section_construction_MM_Oexpert}
Recall that the expert function to be minimized is given by:
\begin{align*}  
  O_{expert}(\mUpsilon,\vsigma^2,\vtheta) &=\displaystyle\sum_{k=1}^{K}\left[ \tau_{k} {(\vy - \vupsilon_k(\vx))^\top \mSigma_k^{-1} (\vy - \vupsilon_k(\vx))}+ \tau_k \log \lvert \mSigma_k \rvert\right]^{\top}\\
  &=\displaystyle\sum_{k=1}^{K}\left[ \tau_{k} {(\vy - \mUpsilon_k\vr)^\top \mSigma_k^{-1} (\vy - \mUpsilon_k\vr)}+ \tau_k \log \lvert \mSigma_k \rvert\right]^{\top}.
\end{align*}
Where $\mUpsilon_k=\left[\mUpsilon_{k,0},\dots,\mUpsilon_{k,D_V}\right]$ is a $Q\times P(D_V+1)$ matrix and $\vr=\vect\left(\left[\vx^{0},\dots,\vx^{Dv}\right]\right)$ is a $P(D_V+1)$-vector. We further denote $\mUpsilon_{k,q,:}$ to be the vector containing the $q$-th row of $\mUpsilon_k$ and $\sigma_{k,q}^2$ to be the values on the diagonal of $\mSigma_k$ for $q \in [Q]$.\\
Here the incremental stochastic MM algorithm can be written as:
\begin{align*}
    \phib(\vupsilon,\vsigma)&:=
    \begin{bmatrix}
      \vkappa\\
      \vzeta \\
      \mDelta\\
      \bar{\vSigma}
    \end{bmatrix},~ 
    \bars(\vtau;\vx)=
    \begin{bmatrix}
      \vtau\otimes \vy^{2} \\
      -2\vtau\otimes(\vy\otimes\vr) \\      \vtau\otimes\left(\vect\left(\vr\vr^{\top}\right)\otimes \mathbf{1}_Q\right)\\
      \vtau\otimes \mathbf{1}_Q
    \end{bmatrix}, \text{ where}\\
    \bar{\vSigma}&=[\log\sigma_{1,1}^{2},\dots, \log\sigma_{1,Q}^2, \log\sigma_{2,1}^{2},\dots,\log\sigma_{K,Q}^{2}]^{\top},~\vkappa=\left[\frac{1}{\sigma_{1,1}^{2}},\dots,\frac{1}{\sigma_{1,Q}^2},\frac{1}{\sigma_{2,1}^{2}},\dots,\frac{1}{\sigma_{K,Q}^{2}}\right]^{\top},\\
    \vzeta&=\vect\left(\left[\frac{\mUpsilon_{1,1}^{\top}}{\sigma_{1,1}^{2}},\dots,\frac{\mUpsilon_{1,Q}^{\top}}{\sigma_{1,Q}^{2}},\frac{\mUpsilon_{2,1}^{\top}}{\sigma_{2,1}^{2}},\dots,\frac{\mUpsilon_{K,Q}^{\top}}{\sigma_{K,Q}^{2}}\right]\right),~
    \vtau=[\tau_{1},\tau_{2},\dots,\tau_{K}]^{\top},\\
    \mDelta&=\vect\left(\frac{\vect(\mUpsilon_{1,1}^{\top}\mUpsilon_{1,1})}{\sigma_{1,1}^{2}},\dots,\frac{\vect(\mUpsilon_{1,Q}^{\top}\mUpsilon_{1,Q})}{\sigma_{1,Q}^{2}},\frac{\vect(\mUpsilon_{2,1}^{\top}\mUpsilon_{2,1})}{\sigma_{2,1}^{2}},\dots,\frac{\vect(\mUpsilon_{K,Q}^{\top}\mUpsilon_{k,q,:})}{\sigma_{k,q}^2}\right).
\end{align*}
Results from \cref{section_construction_MM_Ogate,section_construction_MM_Oexpert} imply the desired updates of \cref{eq_s_step} presented in \cref{algorithm_onlineMM_MoE}.
\subsubsection{Parameters update}\label{param update}
\begin{lemma} \label{lemma_nonegative_serie}
    Consider the series $\vu_m$ defined as follows:
    \begin{align*}
        &\vu_0=\mA \in \sR^{m \times m},~\vu_{m+1}=\vu_m+\Lambda(\mB-\vu_m),~ \forall m \geq 1, \text{ here } \mB\in \sR^{m\times m}.
    \end{align*}
    Then, $\vu_m \succeq \zero$ for all $m$ if the following conditions are satisfied: $\vu_0\succeq \zero$, $\mB\succeq \zero$ and $1 >\Lambda > 0$.
\end{lemma}
\par The proofs for \cref{lemma_nonegative_serie} are shown in \cref{lemma4_proof}.

\textbf{Gating update:}\\
At the $n+1$ iteration, our gating parameters are calculated as followed:
\begin{align*}
    \vomega_{n+1}&=\displaystyle\argmin_{\vomega}\left[\vomega^{\top} \vs_{n+1,1}+\vect(\vomega\vomega^{\top})^{\top}\vs_{n+1,2}\right].
\end{align*}
The first and second derivatives of $r_1(\vomega)=\vomega^{\top}\vs_{n+1,1}+\vect(\vomega\vomega^{\top})^{\top}\vs_{n+1,2}$ are:
\begin{align} \label{1deGatting}
   \nabla r_1(\vomega)=\vs_{n+1,1}+(\mat(\vs_{n+1,2})+\mat(\vs_{n+1,2})^{\top})\vomega,
\end{align}
\begin{align} \label{2deGatting}
   \nabla^2 r_1(\vomega)=\mat(\vs_{n+1,2})+\mat(\vs_{n+1,2})^{\top}.
\end{align}
\par Since
\begin{align*}
    \vs_{n+1,2}=\vs_{n,2}+\gamma_{n+1}\left[\frac{1}{2}\vect\mB_{n,K}-\vs_{n,2}\right],
\end{align*}
hence:
\begin{align*}
    \mat(\vs_{n+1,2})
    &=\mat(\vs_{n,2})+\gamma_{n+1}\left[\frac{1}{2}\mB_{n,K}-\mat(\vs_{n,2})\right].
\end{align*}
\par Now by applying \cref{lemma_nonegative_serie} (knowing $\mat(\vs_{0,2})\succeq \zero)$, we have $\mat(\vs_{n,2})$ is positive semi-definite for all $n$. Combine this with \cref{2deGatting}, we have $\vomega_{n+1}$ is the solution of \cref{1deGatting}. Moreover, since $\vs_{0,2}$ and $\mB_{n,K}$ are PSD matrices for all integer $n$, 
 by induction, we have $\mat(\vs_{n,2})$ is a PSD matrix for all $n$. 

Hence,
\begin{align*}
\vomega_{n+1} = -\left(2\mat(\vs_{n+1,2})\right)^{-1} \vs_{n+1,1},
\end{align*}
where $\mat(\vs_{n+1,2})$ represents the reshaping of $\vs_{n+1,2}$ into a $K(D_W+1)P \times K(D_W+1)P$ square matrix.

\textbf{Expert mean functions update:}

At the $n+1$ iteration, our expert mean functions parameters are calculated as followed:
\begin{align*}
    \mUpsilon_{n+1,k,q,:}&=\displaystyle\argmin_{\mUpsilon_{k,q,:}}\left[\vs_{n+1,4,k:q}^{\top}\frac{\mUpsilon_{k,q,:}^{\top}}{\sigma_{k,q}^2}+\vs_{n+1,5,k:q}^{\top}\frac{\vect(\mUpsilon_{k,q,:}^{\top}\mUpsilon_{k,q,:})}{\sigma_{k,q}^2}\right].\\   
\end{align*}
\par The first and second derivatives of 
\begin{align*}
    r_2(\mUpsilon_{k,q,:})=\vs_{n+1,4,k:q}^{\top}\frac{\mUpsilon_{k,q,:}^{\top}}{\sigma_{k,q}^2}+\vs_{n+1,5,k:q}^{\top}\frac{\vect(\mUpsilon_{k,q,:}^{\top}\mUpsilon_{k,q,:})}{\sigma_{k,q}^2}
\end{align*} are:
\begin{align} \label{1dupsi}
        \nabla r_2(\mUpsilon_{k,q,:})=\frac{\vs_{n+1,4,k:q}}{\sigma_{k,q}^{2}}+\frac{\left(\mat\left(\vs_{n+1,5,k:q}\right)+\mat\left(\vs_{n+1,5,k:q}\right)^{\top}\right)}{\sigma_{k,q}^{2}}\mUpsilon_{k,q,:}^{\top}, \text{ and } 
\end{align}
\begin{align} \label{2dupsi}
    \nabla^2 r_2(\mUpsilon_{k,q,:})=\frac{\left(\mat\left(\vs_{n+1,5,k:q}\right)+\mat\left(\vs_{n+1,5,k:q}\right)^{\top}\right)}{\sigma_{k,q}^{2}}.
\end{align}
\par Since
\begin{align*}
    \vs_{n+1,5,k:q}=\vs_{n,5,k:q}+\gamma_{n+1}\left(\vect\left(\vr\vr^{\top}\right)-\vs_{n,5,k:q}\right),
\end{align*}
hence:
\begin{align*}
    \mat(\vs_{n+1,5,k:q})=\mat(\vs_{n,5,k:q})+\gamma_{n+1}\left(\vr\vr^{\top}-\mat(\vs_{n,5,k:q})\right).
\end{align*}
\par Now by applying \cref{lemma_nonegative_serie}, knowing that:
\begin{align*}
\mat(\vs_{0,5,k:q})=\widehat{\vs}\widehat{\vs}^{\top}+\mI_{D_{V}+1}\succeq \zero,
\end{align*}
we have $\mat(\vs_{n,5,k:q})$ is positive definite for all $n$ and $k \in [K]$, $q\in [Q]$. Combine this with \cref{2dupsi}, we have the solution of \cref{1dupsi} is the global minimum of $r_2(\mUpsilon_{k,q,:})$. Additionally, it can be shown by induction that $\mat(\vs_{n,5,k:q})$ is always a symmetric matrix. 
\par
Hence: 
\begin{align*}
\mUpsilon_{n+1,k,q,:}=-\left(2\mat\left(\vs_{n+1,5,k:q}\right)\right)^{-1}\vs_{n+1,4,k:q},
\end{align*}
where $\mat(\vs_{n+1,5,k:q})$ is $\vs_{n+1,5,k:q}$ reshape into $P(D_{V}+1) \times P(D_{V}+1)$ square matrix.

\textbf{Expert variance update:}\\
At the $n+1$ iteration, our expert variance parameters are calculated as followed:
\begin{align*}
    \sigma_{n+1,k,q}^{2}&=\displaystyle\argmin_{\sigma_{k,q}^2}\left[s_{n+1,3,kq}\frac{1}{\sigma_{k,q}^2}+\vs_{n+1,4,k:q}^{\top}\frac{\mUpsilon_{k,q,:}^{\top}}{\sigma_{k,q}^2}+\vs_{n+1,5,k:q}^{\top}\frac{\vect(\mUpsilon_{k,q,:}^{\top}\mUpsilon_{k,q,:})}{\sigma_{k,q}^2}+s_{n+1,6,kq}\log(\sigma_{k,q}^2)\right].
\end{align*}
\par The first and second derivatives of: 
\begin{align*}
    r_3(\sigma_{k,q}^2)=s_{n+1,3,kq}\frac{1}{\sigma_{k,q}^2}+\vs_{n+1,4,k:q}^{\top}\frac{\mUpsilon_{k,q,:}^{\top}}{\sigma_{k,q}^2}+\vs_{n+1,5,k:q}^{\top}\frac{\vect(\mUpsilon_{k,q,:}^{\top}\mUpsilon_{k,q,:})}{\sigma_{k,q}^2}+s_{n+1,6,kq}\log(\sigma_{k,q}^2)
\end{align*} are:
\begin{align} \label{d1sigma}
    \nabla r_3(\sigma_{k,q}^2)=-\frac{s_{n+1,3,kq}+\vs_{n+1,4,k:q}^{\top}\mUpsilon_{k,q,:}^{\top}+\vs_{n+1,5,k:q}^{\top}\vect(\mUpsilon_{k,q,:}^{\top}\mUpsilon_{k,q,:})}{\sigma_k^4}+\frac{s_{n+1,6,kq}}{\sigma_{k,q}^2},
\end{align}
\begin{align} 
    \nabla^2 r_3(\sigma_{k,q}^2)=2\frac{s_{n+1,3,kq}+\vs_{n+1,4,k:q}^{\top}\mUpsilon_{k,q,:}^{\top}+\vs_{n+1,5,k:q}^{\top}\vect(\mUpsilon_{k,q,:}^{\top}\mUpsilon_{k,q,:})}{\sigma_{k,q}^6}-\frac{s_{n+1,6,kq}}{\sigma_{k,q}^4}.\nn
\end{align}
We shall prove that the solution to \cref{d1sigma} is a global minimum of: 
\begin{align*}
    s_{n+1,3,kq}\frac{1}{\sigma_{k,q}^2}+\vs_{n+1,4,k:q}^{\top}\frac{\mUpsilon_{k,q,:}^{\top}}{\sigma_{k,q}^2}+\vs_{n+1,5,k:q}^{\top}\frac{\vect(\mUpsilon_{k,q,:}^{\top}\mUpsilon_{k,q,:})}{\sigma_{k,q}^2}+s_{n+1,6,kq}\log(\sigma_{k,q}^2).
\end{align*}
Denote $\sigma_{n+1,k,q}^{2}=\dfrac{s_{n+1,3,kq}+\vs_{n+1,4,k:q}^{\top}\mUpsilon_{k,q,:}^{\top}+\vs_{n+1,5,k:q}^{\top}\vect(\mUpsilon_{k,q,:}^{\top}\mUpsilon_{k,q,:})}{s_{n+1,6,kq}}$ be such solution, we have:
\begin{align*} 
    &\sigma_{n+1,k,q}^{6}\nabla^2 r_3(\sigma_{n+1,k,q}^{2})\\
    &=s_{n+1,3,kq}+\vs_{n+1,4,k:q}^{\top}\mUpsilon_{k,q,:}^{\top}+\vs_{n+1,5,k:q}^{\top}\vect(\mUpsilon_{k,q,:}^{\top}\mUpsilon_{k,q,:})\\
    &=s_{n,3,kq}+\gamma_{n+1}\left({\tau}_k y_q^{2}-s_{n,3,kq}\right)+\left[\vs_{n,4,k:q}+\gamma_{n+1}\left((-2y_q\tau_k\vr)-\vs_{n,4,k:q}\right)\right]^{\top}\mUpsilon_{k,q,:}^{\top}\\
    &\quad +\left[\vs_{n,5,k:q}+\gamma_{n+1}\left(\tau_k\left(\vect\left(\vr\vr^{\top}\right)\right)-\vs_{n,5,k:q}\right)\right]^{\top}\vect(\mUpsilon_{k,q,:}^{\top}\mUpsilon_{k,q,:})\\
    &=(1-\gamma_{n+1})\left[s_{n,3,kq}+\vs_{n,4,k:q}^{\top}\mUpsilon_{k,q,:}^{\top}+\vs_{n,5,k:q}^{\top}\vect(\mUpsilon_{k,q,:}^{\top}\mUpsilon_{k,q,:})\right]+\gamma_{n+1}\tau_k\left(y_q-\mUpsilon_{k,q,:}\vr\right)^2.
\end{align*}
\par Now notice that for all value of $\mUpsilon_{k,q,:}$ we have:
\begin{align*}
    s_{0,3,kq}+s_{0,4,kq}^{\top}\mUpsilon_{k,q,:}^{\top}+\vs_{0,5,k:q}^{\top}\vect(\mUpsilon_{k,q,:}^{\top}\mUpsilon_{k,q,:})=\left(\frac{1}{2}+\mUpsilon_{k,q,:}\widehat{\vs}\right)^2 +\mUpsilon_{k,q,:}\mUpsilon_{k,q,:}^{\top}+\frac{3}{4} >0.
\end{align*}
Hence using induction we can prove that $\sigma_{n+1,k,q}^{6}\nabla^2 r_3(\sigma_{n+1,k,q}^{2})>0$ for all $n, k,q, \mUpsilon_{k,q,:}$. Which directly leading to $\sigma_{n+1,k,q}^{2}$ being a local minimum. Now since $\sigma_{n+1,k,q}^{2}$ is the only solution of $\nabla r_3(\sigma_{k,q}^2)$, to prove that it is the global minimum, we only need to show that $$\displaystyle\lim_{\sigma_{k,q}^2\rightarrow0^{+}}r_3(\sigma_{k,q}^2)\geq r_3(\sigma_{n+1,k,q}^{2}) \text{ and } \displaystyle\lim_{\sigma_{k,q}^2\rightarrow+\infty}r_3(\sigma_{k,q}^2)\geq r_3(\sigma_{n+1,k,q}^{2}).$$ 

Denote $U=s_{n+1,3,kq}+\vs_{n+1,4,k:q}^{\top}\mUpsilon_{k,q,:}^{\top}+\vs_{n+1,5,k:q}^{\top}\vect(\mUpsilon_{k,q,:}^{\top}\mUpsilon_{k,q,:}) >0$, we have:
\begin{align*}
    &\displaystyle\lim_{\sigma_{k,q}^2\rightarrow+\infty}r_3(\sigma_{k,q}^2)=\displaystyle\lim_{\sigma_{k,q}^2\rightarrow+\infty}\frac{U}{\sigma_{k,q}^2}+\displaystyle\lim_{\sigma_{k,q}^2\rightarrow+\infty}s_{n+1,6,kq}\log(\sigma_{k,q}^2)=0 + \infty=+\infty > r_3(\sigma_{n+1,k,q}^{2}),\\
    &\displaystyle\lim_{\sigma_{k,q}^2\rightarrow0^{+}}r_3(\sigma_{k,q}^2)=\displaystyle\lim_{\sigma_{k,q}^2\rightarrow+\infty}r_3(\frac{1}{\sigma_{k,q}^2})=\displaystyle\lim_{\sigma_{k,q}^2\rightarrow+\infty}(U\sigma_{k,q}^2-s_{n+1,6,kq}\log(\sigma_{k,q}^2))=+\infty > r_3(\sigma_{n+1,k,q}^{2}).
\end{align*}
\par Hence this has proven that $\sigma_{n+1,k,q}^{2}$ is the global minimum, thus:
\begin{align*}
    \sigma_{n+1,k,q}^{2}=\frac{s_{n+1,3,kq}+\vs_{n+1,4,k:q}^{\top}\mUpsilon_{k,q,:}^{\top}+\vs_{n+1,5,k:q}^{\top}\vect(\mUpsilon_{k,q,:}^{\top}\mUpsilon_{k,q,:})}{s_{n+1,6,kq}}.
\end{align*}

{\bf Notes on the updates of $\mUpsilon$ and $\mSigma$:}
\begin{lemma} \label{lemm_notes_Upsilon_Sigma}
    Consider the function 
    \begin{align}
        f(\evu,\vv)=\frac{g(\vv)}{\evu}+c\log(\evu). \nn
    \end{align}
    Where $\vv$ is a vector, $\evu$ is a positive real number, $c$ is a positive constant and $g(\vv): \R^{n}\rightarrow \R^{+}$ be a function that has exactly one global minimum. 
    Denote $\vv_0=\displaystyle\argmin_{\vv}g(\vv)$ be the global minimum of $g(\vv)$ and $\evu_0=\displaystyle\argmin_{\evu}\left(\frac{g(\vv_0)}{\evu}+\log(\evu)\right)$. We have $f(\evu_0,\vv_0)$ is the global minimum of $f$ over all $\evu,~\vv$.
\end{lemma}
\begin{proof}[Proof of \cref{lemm_notes_Upsilon_Sigma}]
    For all $\vv$, we have:
    \begin{align*}
        &\nabla_u f(\evu,\vv)=\frac{1}{\evu^2}(-g(\vv)+cu),~\nabla^2_u f(\evu,\vv)=\frac{1}{\evu^3}(2g(\vv)-cu).
    \end{align*}
    Hence $\evu=\frac{g(\vv)}{c}$ is the local minimum of $f(\evu,\vv$) for all $\vv$. But since $\nabla_{\evu}f(\evu,\vv)$ has exactly one solution for $\evu$ and both $\displaystyle\lim_{\evu\rightarrow+\infty}f(\evu,\vv), \displaystyle\lim_{\evu\rightarrow+0}f(\evu,\vv)$ is $+\infty$. We can conclude that $\evu=\frac{g(\vv)}{c}$ is the global minimum of $f(\evu,\vv$) for all given $\vv$ .\\
    Assume that $(\evu_1,\vv_1)$ is the minimum of $f(\evu,\vv$), we have
    \begin{align*}
        f(\evu_1,\vv_1)\geq f\left(\frac{g(\vv_1)}{c},\vv_1\right)=c\log\left(\frac{g(\vv_1)}{c}\right) \geq c\log\left(\frac{g(\vv_0)}{c}\right)= f\left(\frac{g(\vv_0)}{c},\vv_0\right).
    \end{align*}
    \par Which lead to $g(\vv_1)=g(\vv_0)$. But since  $\vv_0$ is the only global minimum of $g(\vv)$, we must have $\vv_1=\vv_0$. Thus, leading to $\evu_1$ also is equal to $\evu_0$.
\end{proof}

\begin{remark}
    In our problem, $\vv=\mUpsilon_{k,q,:}^{\top}, \evu=\sigma_{k,q}^2$ and $g(\vv)=\sigma_{k,q}^2r_2(\mUpsilon_{k,q,:}^{\top})$. It has been proven in the expert mean functions update section that $\sigma_{k,q}^2r_2(\mUpsilon_{k,q,:}^{\top})$ has exactly unique global minimum (since it is strictly convex). Hence, \cref{lemm_notes_Upsilon_Sigma} shows that our way of updating $\mUpsilon$ and then $\mSigma$ is the same as updating the expert function all at once.
\end{remark} 

\subsection{Proof of \texorpdfstring{\cref{theorem_onlineMM_SGMLMoE}}{Theorem \ref*{theorem_onlineMM_SGMLMoE}}}\label{proof_theorem_onlineMM_SGMLMoE}
\subsubsection{Proof of surrogate function construction in \texorpdfstring{\cref{proposition_surrogate_construction_SGMLMoE}}{Proposition \ref*{proposition_surrogate_construction_SGMLMoE}}}\label{proof_proposition_surrogate_construction_SGMLMoE}

By applying \cref{lemma_multi_D_inq1} for the function $f=-\log$ and: 
\begin{align*}
	&\hat{\vc}_k=[1,1,\dots,1]^{\top},\\
	&\va=[\sfg_k\left(\vx_n; \vomega \right) \ex_k(\vv_{y_n}(\vx_n;\vupsilon))]_{k\in[K]}^{\top},\\
    &\vb=[\sfg_k\left(\vx_n; \vomega_n \right) \ex_k(\vv_{y_n}(\vx_n;\vupsilon_n))]_{k\in[K]}^{\top}.
\end{align*}
\par Here we $\hat{\vc}$ is used as $\vc$ in \cref{lemma_multi_D_inq1} to not be mistaken for the vectorization of $\vupsilon$ in \cref{proposition_surrogate_construction_SGMLMoE}. We obtain:
\begin{align*}
	-\log \left[\sum_{k=1}^{K} \sfg_k\left(\vx_n; \vomega \right) \ex_k(\vv_{y_n}(\vx_n;\vupsilon))\right] \leq -\sum_{k=1}^K\tau_{n,k}\log\left[\frac{1}{\tau_{n,k}}\sfg_k\left(\vx_n; \vomega \right) \ex_k(\vv_{y_n}(\vx_n;\vupsilon))\right],
\end{align*}
\par where $\tau_{n,k}=\frac{\sfg_k\left(\vx_n; \vomega_n \right) \ex_k(\vv_{y_n}(\vx_n;\vupsilon_n))}{\sum_{l=1}^{K}\sfg_l\left(\vx_n; \vomega_n \right) \ex_l(\vv_{y_n}(\vx_n;\vupsilon_n))}$.
\par Hence:
\begin{align*}
	-\log s_{\vtheta}(y_n\mid \vx_n) &\le -\sum_{k=1}^K\tau_{n,k}\log\left[\frac{1}{\tau_{n,k}}\sfg_k\left(\vx_n; \vomega \right) \ex_k(\vv_{y_n}(\vx_n;\vupsilon))\right]\nn\\
    &=\sum_{k=1}^{K}\tau_{n,k}\log(\tau_{n,k})-\sum_{k=1}^{K}\tau_{n,k}\left[\log(\sfg_k\left(\vx_n; \vomega \right))+\log(\ex_k(\vv_{y_n}(\vx_n;\vupsilon)))\right].
\end{align*}    

Similar to \cref{sec_exponential_family_surrogate}, by applying \cref{Kronecker_prod_property} and \cref{prop_theo3}, we obtain the first part of the majorization function as follows:
\begin{align} 
	-\log s_{\vtheta}(\vy_n\mid \vx_n) 
    &\le\sum_{k=1}^{K}\tau_{n,k}\log(\tau_{n,k})+\underline{\sfg}(\vomega_n;\vx_n)-\vomega_{n}^{\top}\nabla \underline{\sfg}(\vomega_n;\vx_n)-\sum_{k=1}^{K}\tau_{n,k}w_k(\vx_n)\nn\\
    &\quad+\left\{\vomega^{\top}\nabla \underline{\sfg}(\vomega_n;\vx_n)+\frac{1}{2}(\vomega-\vomega_{n}^{\top})\mB_{K,n}\left(\vomega-\vomega_n\right)\right\}\nn
   \\
    &\quad-\sum_{k=1}^{K}\tau_{n,k}\log(\ex_k(\vv_{y_n}(\vx_n;\vupsilon))).\label{eq_surrogate_Taylor_SGMLMoE_1} 
\end{align}
Hence, using the same idea of the majorization construction for $-\sum_{k=1}^{K}\tau_{n,k}\log(\ex_k(\vv_{y_n}(\vx_n;\vupsilon)))$, our overall majorization function takes the form:
\begin{align*}
    &-\log s_{\vtheta}(y_n\mid \vx_n) \\
    &\le \sum_{k=1}^{K}\tau_{n,k}\log(\tau_{n,k})-\displaystyle\sum_{k=1}^K\tau_{n,k}\evw_k(\vx_n)-\displaystyle\sum_{k=1}^K\tau_{n,k}\left(\sum_{l=1}^{M}\I(y_n,l)\evv_{l,k}(\vx_n)\right)\\
    &\quad +\left\{\underline{\sfg}(\vomega_n;\vx_n)+ (\vomega-\vomega_n)^{\top}\nabla \underline{\sfg}(\vomega_n;\vx_n)+\frac{1}{2}(\vomega-\vomega_n)^{\top}\mB_{n,K}(\vomega-\vomega_n)\right\} 
   \\
    &\quad +\displaystyle\sum_{k=1}^{K}\tau_{n,k}\left\{\underline{\ex}(\vc_{n,k};\vx_n)+ (\vc_{n,k}-\vc_{n,k})^{\top}\nabla \underline{\ex}(\vc_{n,k};\vx_n)+\frac{1}{2}(\vc_{n,k}-\vc_{n,k})^{\top}\mB_{n,M}(\vc_{n,k}-\vc_{n,k})\right\}\\
    &= \sum_{k=1}^{K}\tau_{n,k}\log(\tau_{n,k})-\left(\vomega^{\top}\vs_n\right)-\left(\vupsilon^{\top}\vr_n \right)\\
    &\quad +\left\{\underline{\sfg}(\vomega_n;\vx_n)+ (\vomega-\vomega_n)^{\top}\nabla \underline{\sfg}(\vomega_n;\vx_n)+\frac{1}{2}(\vomega-\vomega_n)^{\top}\mB_{n,K}(\vomega-\vomega_n)\right\} 
   \\
   &\quad +\displaystyle\sum_{k=1}^{K}\tau_{n,k}\left\{\underline{\ex}(\vc_{n,k};\vx_n)+ (\vc_{n,k}-\vc_{n,k})^{\top}\nabla \underline{\ex}(\vc_{n,k};\vx_n)+\frac{1}{2}(\vc_{n,k}-\vc_{n,k})^{\top}\mB_{n,M}(\vc_{n,k}-\vc_{n,k})\right\}.
\end{align*}
Here $\vupsilon$ is ordered as $\left(\evupsilon_{m,k,d,p}\vert_{k\in[K],p\in [P], m\in[M-1], d\in\{0,\ldots,D_W\}}\right)$ and
\begin{align*}
    \vc_{n,k}&=[\vupsilon_{n,1,k,0,1},\dots,\vupsilon_{n,1,k,D_V,1},\dots,\vupsilon_{n,M,k,D_V,1},\dots,\vupsilon_{n,M,k,D_V,P}]^{\top},\\
     \underline{\ex}(\vc_{n,k};\vx_n)&=\log\left(1+\sum_{l=1}^{M-1} \exp\left(\evv_{n,l,k}(\vx_n)\right)\right), \quad \underline{\sfg}(\vomega_n;\vx_n)=\log\left(1+\sum_{k=1}^{K-1} \exp\left(\evw_{n,k}(\vx_n)\right)\right),\\
    \tau_{n,k}&=\frac{\sfg_k\left(\vx_n; \vomega_n \right) \ex_k(\vv_{y_n}(\vx_n;\vupsilon_n))}{\sum_{l=1}^{K}\sfg_l\left(\vx_n; \vomega_n \right) \ex_l(\vv_{y_n}(\vx_n;\vupsilon_n))}.
\end{align*}

\subsubsection{Verification of \texorpdfstring{\cref{assumptionA1,assumptionA2,assumptionA3,assumptionA4,assumptionObjectiveC1}}{\ref*{assumptionA1} to \ref*{assumptionObjectiveC1}}, and \texorpdfstring{\cref{assumptionA5,assumptionA6,assumptionA11}}{\ref*{assumptionA5},\ref*{assumptionA6} and \ref*{assumptionA11}} for the softmax-gated multinomial logistic MoE models}\label{proof_assumptions_theorem_onlineMM_SGMLMoE}

We will assume that $\vx$ and $\evy$ are bounded by $\evx^\#$ and $\evy^\#$, respectively.

\textbf{Verification of \cref{assumptionA1}.} For the problem of softmax-gated MoE, we consider the majorizer surrogate $g$ to have the following form that belongs to the exponential family
\begin{equation} \label{eq_exponential_family_appendix_SGMLMoE}
g\left(\vtheta,\vz;\vtau\right)\eqdef - \psi\left(\vtheta\right)+ \pscal{\bars(\vtau;\vz)}{\phib(\vtheta)},
\end{equation}
where $\psi(\vomega)=C_{n,k}$, and
\begin{align*}
    &\phib(\vomega):=
    \begin{bmatrix}
      \vomega\\
      \vect(\vomega\vomega^{\top})\\
      \vupsilon\\
      \vect\left({\bdiag}_{MP(D_V+1)}(\vupsilon\vupsilon^{\top})\right)
    \end{bmatrix}, ~\bars(\vtau;\vz)=
    \begin{bmatrix}
      -\vs+\nabla \underline{\sfg}(\vtau)-\mB_{K-1}\vtau \\
      \frac{1}{2}\vect(\mB_{K-1})\\
      -\vr+\vect\left([\tau_k\nabla \underline{\ex}(\vtau)-\tau_k\mB_{M-1}\vtau]_{k=1,\dots,K}\right)\\
      \frac{1}{2}\vect(\vtau\otimes\mB_{M-1})
    \end{bmatrix}.
\end{align*}
where $\vtau$, $\mB$ are defined
\cref{section_onlineMM_SGMoE} and $\vupsilon,\vomega,\vr,\vs$ are defined in \cref{section_onlineMM_SGMLMoE}. Furthermore, since $\phib,\psi$ and $\bars$ are composed of elementary functions applied component-wise to their parameters, they are trivially continuously differentiable and measurable due to the smoothness and measurability properties of elementary functions.

\textbf{Verification of \cref{assumptionA2}.} We can define $\mathbb{S}$ as follows:
\begin{align*}
    \mathbb{S}&=\big\{(\vs_1, \vect(\mS_2),\vs_3, \vect(\mS_4)): \vs_1\in \sR^{(K-1)P(D_W+1)};  \mS_2\in \sR^{(K-1)P(D_W+1)\times (K-1)P(D_W+1)}, \\
    &\quad \mS_2 \succeq \zero ; \vs_3\in \sR^{(D_V+1)MPK}; \mS_4\in \sR^{(D_V+1)MPK\times(D_V+1)MPK}, \mS_4 \succeq \zero\big\} .
\end{align*}
It is trivial that $\mathbb{S}$ denoted above satisfies the requirements of \cref{assumptionA2}.

\textbf{Verification of \cref{assumptionA3}.} $\bars(\vtheta,\vz)$ is composed of elementary functions applied component-wise to $\vx$ and $\vy$, hence its expectation exist and lies in $\mathbb{S}$.
Furthermore, since $\vx$ and $\evy$ are bounded, it can be shown that $\vtau,\mB,\vr,\vs$, $\nabla \underline{\sfg}$, and $\nabla \underline{\ex}$ are all bounded. Thus, $\mathbb{E}[\bars(\vtheta,\vz)]$ is finite.

\textbf{Verification of \cref{assumptionA4}.} This assumption is proven in \cref{Paramupdate for SGMLMoE}, where we have shown that $\vupsilon,\vomega$ of the next iteration are the unique solution to the minimization problem of the current iteration.

\textbf{Verification of \cref{assumptionObjectiveC1}.} Since $f(\vtheta;\vz)$ is continuously differentiable in $\vtheta$ for all $\vz$
and dominated by an integrable envelope (due to bounded covariates and responses),
differentiation under the expectation is valid, hence
$\vtheta \mapsto \E[f(\vtheta;\rvz)]$ is $C^1$ on $\sT$.

\textbf{Verification of \cref{assumptionA5}.} Since we have $\vomega \in \sR^{P\times K\times (D_W+1)}$ and $\vupsilon\in \sR^{(D_V+1)\times M\times P\times K}$, thus $~\mathbb{T}=\sR^{P\times K\times (D_W+1)}\times\sR^{(D_V+1)\times M\times P\times K}$ is a convex and open set.
\par Furthermore, similar to the verification of \cref{assumptionA1}, we can show that $\psi$ and $\phib$ is twice continuously differentiable throughout $\mathbb{T}$.

\textbf{Verification of \cref{assumptionA6}.} From \cref{Paramupdate for SGMLMoE}, given $\vs=(\vs_1,\vs_2,\vs_3,\vs_4)$, we have the function $\bar{\vtheta}(\vs)$ is defined as:
\begin{align*}
    \bar{\vtheta}(\vs)=\begin{bmatrix}
      &-\left(\mat(\vs_{2,n+1}) + \mat(\vs_{2,n+1})^{\top}\right)^{-1} \vs_{1,n+1}\\
      &-\left[{\bdiag}_{\iota}^{-1}({\mat}_{\iota}(\vs_{4,n+1}))+{\bdiag}_{\iota}^{-1}({\mat}_{\iota}(\vs_{4,n+1}))^{\top}\right]^{-1} \vs_{3,n+1}
    \end{bmatrix}.
\end{align*}
It is trivial to show that $\bar{\vtheta}(\vs)$ is continuously differentiable for $\vs_1, \vs_3$. 
\par We have $\mA\rightarrow \mat(\mA)+\mat(\mA)^{\top}$ and $\mA\rightarrow {\bdiag}_{\iota}^{-1}({\mat}_{\iota}(\mA))+{\bdiag}_{\iota}^{-1}({\mat}_{\iota}(\mA))^{\top}$ are affine transformations, thus they are continuously differentiable. Furthermore, from \cref{Paramupdate for SGMLMoE}, we know that $\mat(\vs_2)+\mat(\vs_2)^{\top}$ and ${\bdiag}_{\iota}^{-1}({\mat}_{\iota}(\vs_{4}))+{\bdiag}_{\iota}^{-1}({\mat}_{\iota}(\vs_{4})$ are positive definite, thus their inverts are non-singular, thus their inverts are continuously differentiable. This concludes our proof.

\textbf{Verification of \cref{assumptionA11}.} By simply choosing $\gamma_n=\gamma_0 n^{-\alpha}$ with $\alpha\in\left(\tfrac12,1\right]$ and $\gamma_0\in(0,1)$, all requirements of \cref{assumptionA11} will be met.

\begin{remark}
    \cref{assumptionA7} cannot be verified due to the ambiguity in the definition of the $\bdiag$ function. Similarly to softmax-gated Gaussian MoE, \cref{assumptionA10} for scenarios $\vs_1$ and $\vs_3$, as well as \cref{assumptionA8,assumptionA9}, also remain unproven.
\end{remark}

\subsection{Derivation of \texorpdfstring{\cref{algorithm_onlineMM_SGMLMoE}}{Algorithm \ref*{algorithm_onlineMM_SGMLMoE}}}\label{proof_algorithm_onlineMM_SGMLMoE}
\subsubsection{Construction of the incremental MM algorithm}
From here the majorization from the previous section, the incremental MM algorithm can be constructed as:
\begin{align*}
    &\phib(\vomega):=
    \begin{bmatrix}
      \vomega\\
      \vect(\vomega\vomega^{\top})\\
      \vupsilon\\
      \vect\left({\bdiag}_{MP(D_V+1)}(\vupsilon\vupsilon^{\top})\right)
    \end{bmatrix}, ~\bars(\vtau;\vz)=
    \begin{bmatrix}
      -\vs+\nabla \underline{\sfg}(\vtau)-\mB_{K-1}\vtau \\
      \frac{1}{2}\vect(\mB_{K-1})\\
      -\vr+\vect\left([\tau_k\nabla \underline{\ex}(\vtau)-\tau_k\mB_{M-1}\vtau]_{k=1,\dots,K}\right)\\
      \frac{1}{2}\vect(\vtau\otimes\mB_{M-1})
    \end{bmatrix}.
\end{align*}
and $\psi(\vomega)=C_{n,k}$, which is the independent constant with respect to the parameter $\omegab$. Here, recall that we have $\vtau=[\tau_1,\dots,\tau_k]^{\top}$ and for any matrix $\mA \in \sR^{tn\times tn}$ (n is a positive integer) we have ${\bdiag}_t(\mA)\in \sR^{nt\times t}$ defined as:
\begin{align*}
    {\bdiag}_t(\mA)_{[it+1:it+t];[1:t]}=\mA_{[it+1:it+t];[it+1:it+t]},\quad \forall i \in \{0,\dots,n-1\}.
\end{align*}
\subsubsection{Parameters update}\label{Paramupdate for SGMLMoE}
\textbf{Gating parameters update:}\\
At the $n+1$ iteration, our gating parameters are calculated as followed:
\begin{align*}
    \vomega_{n+1}&=\displaystyle\argmin_{\vomega}\left[\vomega^{\top}\vs_{n+1,1}+\vect(\vomega\vomega^{\top})^{\top}\vs_{n+1,2}\right].
\end{align*}
The first and second derivatives of $r_1(\vomega)=\vomega^{\top}\vs_{n+1,1}+\vect(\vomega\vomega^{\top})^{\top}\vs_{n+1,2}$ are:
\begin{align} \label{1deGatting_}
   \nabla r_1(\vomega)=\vs_{n+1,1}+(\mat(\vs_{n+1,2})+\mat(\vs_{n+1,2})^{\top})\vomega,
\end{align}
\begin{align} \label{2deGatting_}
   \nabla^2 r_1(\vomega)=\mat(\vs_{n+1,2})+\mat(\vs_{n+1,2})^{\top}.
\end{align}
\par Since
\begin{align*}
    \vs_{n+1,2}=\vs_{n,2}+\gamma_{n+1}\left[\frac{1}{2}\vect\mB_{n,K}-\vs_{n,2}\right],
\end{align*}
hence:
\begin{align*}
    \mat(\vs_{n+1,2})=\mat(\vs_{n,2})+\gamma_{n+1}\left[\frac{1}{2}\mB_{n,K}-\mat(\vs_{n,2})\right].
\end{align*}
\par Now by applying \cref{lemma_nonegative_serie} (knowing $\mat(\vs_{0,2})\succeq \zero)$, we have $\mat(\vs_{n,2})$ is positive definite for all $n$. Combine this with \cref{2deGatting_}, we have $\vomega_{n+1}$ is the solution of \cref{1deGatting_}.

Hence,
\begin{align*}
\vomega_{n+1} = -\left(\mat(\vs_{n+1,2}) + \mat(\vs_{n+1,2})^{\top}\right)^{-1} \vs_{n+1,1},
\end{align*}
where $\mat(\vs_{n+1,2})$ represents the reshaping of $\vs_{n+1,2}$ into a $K(D_W+1)P \times K(D_W+1)P$ square matrix.

\textbf{Expert parameters update:}\\
Let us denote $\iota=MP(D_V+1)$ and for any positive integer $k$, let us define the two functions $mat_k:\sR^{ks}\rightarrow \sR^{s\times k}$ and ${\bdiag}_k^{-1}:\sR^{gk \times k}\rightarrow \sR^{gk\times gk}$ satisfying:
$$\vect(\mathrm{mat}_{k}(A))=A\quad \forall A\in \sR^{ks},$$
$${\bdiag}_k({\bdiag}_k^{-1}(A))=A\quad \forall A\in \sR^{ks\times k}.$$
At the $n+1$ iteration, our expert parameters are calculated as followed:
\begin{align*}
    \vupsilon_{n+1}&=\displaystyle\argmin_{\vupsilon}\left[\vupsilon^{\top}s_{n+1,3}+\vect\left({\bdiag}_{\iota}(\vupsilon\vupsilon^{\top})\right)^{\top}\vs_{n+1,4}\right].
\end{align*}
The first and second derivatives of $r_1(\vupsilon)=\vupsilon^{\top}s_{n+1,3}+\vect\left({\bdiag}_{\iota}(\vupsilon\vupsilon^{\top})\right)^{\top}\vs_{n+1,4}$ are:
\begin{align} \label{1deGatting__}
   \nabla r_1(\vupsilon)=\vs_{n+1,3}+\left[{\bdiag}_{\iota}^{-1}({\mat}_{\iota}(\vs_{n+1,4}))+{\bdiag}_{\iota}^{-1}({\mat}_{\iota}(\vs_{n+1,4}))^{\top}\right]\vupsilon,
\end{align}
\begin{align} \label{2deGatting__}
   \nabla^2 r_1(\vupsilon)={\bdiag}_{\iota}^{-1}({\mat}_{\iota}(\vs_{n+1,4}))+{\bdiag}_{\iota}^{-1}({\mat}_{\iota}(\vs_{n+1,4}))^{\top}.
\end{align}
\par Since
\begin{align*}
    \vs_{n+1,4}=\vs_{n,4}+\gamma_{n+1}\left[\frac{1}{2}\vect\left[\taub\otimes\mB_{n,M}\right]-\vs_{n,4}\right],
\end{align*}
hence:
\begin{align*}
    &{\bdiag}_{\iota}^{-1}({\mat}_{\iota}(\vs_{n+1,4}))=\\
    &{\bdiag}_{\iota}^{-1}({\mat}_{\iota}(\vs_{n,4}))+\gamma_{n+1}\left[\frac{1}{2}{\bdiag}_{\iota}^{-1}\left[\taub\otimes\mB_{n,M}\right]-{\bdiag}_{\iota}^{-1}({\mat}_{\iota}(\vs_{n,4}))\right].
\end{align*}
\par Now by applying \cref{lemma_nonegative_serie} (knowing ${\bdiag}_{\iota}^{-1}({\mat}_{\iota}(\vs_{0,4}))\succeq \zero)$, we have $\mat(\vs_{n,4})$ is positive definite for all $n$. Combine this with \cref{2deGatting__}, we have $\vupsilon_{n+1}$ is the solution of \cref{1deGatting__}.

Hence,
\begin{align*}
\vupsilon_{n+1} = -\left[{\bdiag}_{\iota}^{-1}({\mat}_{\iota}(\vs_{n+1,4}))+{\bdiag}_{\iota}^{-1}({\mat}_{\iota}(\vs_{n+1,4}))^{\top}\right]^{-1} \vs_{n+1,3}.
\end{align*}

\section{Technical proofs}\label{sec_technical_proof}

\subsection{Proof of \texorpdfstring{\cref{Kronecker_prod_property}}{Lemma \ref*{Kronecker_prod_property}}}\label{lemma3_proof}

Using the spectral decomposition we can write:
    \begin{align*}
        &\mC=\mM_{1}^{\top}\mD_{1}\mM_{1},~
        \mB-\mA=\mM_{2}^{\top}\mD_{2}\mM_{2}.
    \end{align*}
    Hence:
    \begin{align*}
        \mC\otimes(\mB-\mA)=(\mM_{1}^{\top}\mD_{1}\mM_{1})\otimes(\mM_{2}^{\top}\mD_{2}\mM_{2})=(\mM_{1}\otimes \mM_{2})^{\top}(\mD_{1}\otimes \mD_{2})(\mM_{1}\otimes \mM_{2}),
    \end{align*}
    which lead to $\mC\otimes(\mB-\mA)$ are nonnegative definite, thus proven the lemma.
\subsection{Proof of \texorpdfstring{\cref{prop_theo3}}{Proposition \ref*{prop_theo3}}}\label{proof_prop_theo3}
Denote $\varepsilon=\frac{1}{\pi_1+\pi_2+\dots+\pi_N}$. Hence $\varepsilon\hat{p}$ is a stochastic vector which according to \cref{upperbound_H_fastMM}:
    \begin{align*}
        \varepsilon\text{diag}(\hat{\pib})-\varepsilon^{2}\hat{\pib}\hat{\pib}^{\top}\leq \dfrac{1}{2}\left(\mI_K-\dfrac{\bm{1}\bm{1}^{\top}}{K}\right).
    \end{align*}
Hence:
    \begin{align*}
        \text{diag}(\hat{\pib})-\hat{\pib}\hat{\pib}^{\top}&\leq \dfrac{1}{2\varepsilon^{2}}\left(\mI_K-\dfrac{\bm{1}\bm{1}^{\top}}{K}\right)+\left (1-\frac{1}{\varepsilon}\right)\text{diag}(\hat{\pib})\\
        &\leq\dfrac{1}{2}\left(\mI_K-\dfrac{\bm{1}\bm{1}^{\top}}{K}\right)+\pi_{N+1}\max_{n\in[N]}\pi_n \mI_K\\
        &\leq\dfrac{1}{2}\left(\mI_K-\dfrac{\bm{1}\bm{1}^{\top}}{K}\right)+\frac{1}{4}\mI_K\\
        &\leq\dfrac{1}{2}\left(\dfrac{3}{2}\mI_K-\dfrac{\bm{1}\bm{1}^{\top}}{K}\right).
    \end{align*}

\subsection{Proof of \texorpdfstring{\cref{lemma_nonegative_serie}}{Lemma \ref*{lemma_nonegative_serie}}} \label{lemma4_proof}
Let say that all the conditions are met, we shall prove that $\vu_m\succeq \zero$ for all $m$ using induction.
    \par
    Since $\vu_1 \succeq \zero$, we assume that $\vu_i\succeq \zero$ for all $i \leq m$. Our target is to prove that $\vu_{m+1} \succeq \zero$.
    \par
    This is true since the term $\vu_m(1-\Lambda)\succeq \zero$ and the term $\Lambda \mB \succeq \zero$ , hence their sum which is $\vu_{m+1}$ is also positive definite.

\section{Technical results}\label{sec_technical_results}

Surrogate function construction plays a pivotal role in optimization techniques such as MM algorithms. By leveraging mathematical inequalities, it is possible to approximate or bound complex objective functions with simpler, more tractable alternatives. Below are several fundamental inequalities commonly employed in the development of surrogate functions:

\begin{lemma}[See, \eg~page 191 in \cite{lange_numerical_2010}]\label{lemma_multi_D_inq1}
    Given three vectors $\va,\vb,\vc \in \sR^{D}$ such that all components of $\va,\vb,\vc$ are positive and a convex function $f$. The inequality below holds:
\begin{inequality} 
	f(\vc^{\top}\va) \leq \displaystyle\sum_{d=1}^{D}\frac{\evc_{d}\evy_{d}}{\vc^{\top}\vb}f\left(\frac{\vc^{\top}\vb}{\evy_{d}}\evx_{d}\right), \text{ where the equality holds when $\va=\vb$.}\nn
\end{inequality}
\end{lemma}

\begin{lemma}[See, \eg~page 191 in \cite{lange_numerical_2010}]
    Given four vectors $\va,\vb,\vc,\valpha \in \sR^{D}$, and a convex function $f$, such that all $\alpha_d \geq 0, \displaystyle\sum_{d=1}^D\alpha_d=1$, and $\alpha_d > 0$ whenever $c_{d}\neq 0$ for all $d\in[D]$. The inequality below holds:
\begin{inequality} 
	f(\vc^{\top}\va) \leq \displaystyle\sum_{d=1}^{D}\alpha_{d}f\left(\frac{c_{d}}{\alpha_{d}}(x_{d}-y_{d})\right)+\vc^{\top}\vb, \text{ where the equality holds when $\va=\vb$.}\nn
\end{inequality}
\end{lemma}

\begin{lemma}[See, \eg~page 206 in \cite{lange_numerical_2010}]
    Given 2 positive reals $\evu, y$. The inequality below holds:
\begin{align} 
\label{multi_D_inq3}
	\log(\evu) \leq \frac{\evu}{y} + \log(y)-1, \text{ where the equality holds when $\evu=y$.}\nn
\end{align}
\end{lemma}

\begin{lemma}[Theorem 5.3 in \cite{bohning1988monotonicity}]\label{upperbound_H_fastMM}
    Let $\pib_N$ be a stochastic vector of dimension $N$. We have:
    \begin{align*}
        \text{diag}(\pib_N)-\pib_N\pib_N^{\top}\leq \dfrac{1}{2}\left(\mI_K-\dfrac{\bm{1}_N\bm{1}_N^{\top}}{N}\right).
    \end{align*}
\end{lemma}


\bibliographystyle{abbrv}
\bibliography{ref}

@article{cai_survey_2025,
	title = {A {Survey} on {Mixture} of {Experts} in {Large} {Language} {Models}},
	volume = {37},
	number = {7},
	journal = {IEEE Transactions on Knowledge and Data Engineering},
	author = {Cai, Weilin and Jiang, Juyong and Wang, Fan and Tang, Jing and Kim, Sunghun and Huang, Jiayi},
	year = {2025},
	pages = {3896--3915},
}

@article{dereich_ode_2025,
	title = {{ODE} approximation for the {Adam} algorithm: {General} and overparametrized setting},
	journal = {arXiv preprint arXiv:2511.04622},
	author = {Dereich, Steffen and Jentzen, Arnulf and Kassing, Sebastian},
	year = {2025},
}

@inproceedings{zhang_moefication_2022,
	address = {Dublin, Ireland},
	title = {{MoEfication}: {Transformer} {Feed}-forward {Layers} are {Mixtures} of {Experts}},
	url = {https://aclanthology.org/2022.findings-acl.71},
	doi = {10.18653/v1/2022.findings-acl.71},
	booktitle = {Findings of the {Association} for {Computational} {Linguistics}: {ACL} 2022},
	publisher = {Association for Computational Linguistics},
	author = {Zhang, Zhengyan and Lin, Yankai and Liu, Zhiyuan and Li, Peng and Sun, Maosong and Zhou, Jie},
	month = may,
	year = {2022},
	pages = {877--890},
}

@inproceedings{chen_sparse_2023,
	title = {Sparse {MoE} as the {New} {Dropout}: {Scaling} {Dense} and {Self}-{Slimmable} {Transformers}},
	url = {https://openreview.net/forum?id=w1hwFUb_81},
	booktitle = {The {Eleventh} {International} {Conference} on {Learning} {Representations}},
	author = {Chen, Tianlong and Zhang, Zhenyu and JAISWAL, AJAY KUMAR and Liu, Shiwei and Wang, Zhangyang},
	year = {2023},
}

@inproceedings{do_hyperrouter_2023,
	address = {Singapore},
	title = {{HyperRouter}: {Towards} {Efficient} {Training} and {Inference} of {Sparse} {Mixture} of {Experts}},
	copyright = {All rights reserved},
	url = {https://aclanthology.org/2023.emnlp-main.351},
	booktitle = {Proceedings of the 2023 {Conference} on {Empirical} {Methods} in {Natural} {Language} {Processing}},
	publisher = {Association for Computational Linguistics},
	author = {Do, Truong and Khiem, Le and Pham, Quang and Nguyen, TrungTin and Doan, Thanh-Nam and Nguyen, Binh and Liu, Chenghao and Ramasamy, Savitha and Li, Xiaoli and Hoi, Steven},
	editor = {Bouamor, Houda and Pino, Juan and Bali, Kalika},
	month = dec,
	year = {2023},
	pages = {5754--5765},
}

@inproceedings{le_mixture_2024,
	title = {Mixture of {Experts} {Meets} {Prompt}-{Based} {Continual} {Learning}},
	booktitle = {Advances in {Neural} {Information} {Processing} {Systems}},
	author = {Le, Minh and Nguyen, An and Nguyen, Huy and Nguyen, Trang and Pham, Trang and Van Ngo, Linh and Ho, Nhat},
	year = {2024},
	note = {arXiv preprint arXiv:2405.14124},
}

@article{nguyen_non_asymptotic_2022,
	title = {A non-asymptotic approach for model selection via penalization in high-dimensional mixture of experts models},
	volume = {16},
	copyright = {All rights reserved},
	url = {https://doi.org/10.1214/22-EJS2057},
	doi = {10.1214/22-EJS2057},
	number = {2},
	journal = {Electronic Journal of Statistics},
	author = {Nguyen, TrungTin and Nguyen, Hien Duy and Chamroukhi, Faicel and Forbes, Florence},
	year = {2022},
	pages = {4742 -- 4822},
}

@article{thai_model_2025,
	title = {Model {Selection} for {Gaussian}-gated {Gaussian} {Mixture} of {Experts} {Using} {Dendrograms} of {Mixing} {Measures}},
	journal = {arXiv preprint arXiv:2505.13052},
	author = {Thai, Tuan and Nguyen, TrungTin and Do, Dat and Ho, Nhat and Drovandi, Christopher},
	year = {2025},
}

@inproceedings{hai_dendrograms_2026,
	series = {Proceedings of {Machine} {Learning} {Research}},
	title = {Dendrograms of {Mixing} {Measures} for {Softmax}-{Gated} {Gaussian} {Mixture} of {Experts}: {Consistency} without {Model} {Sweeps}},
	booktitle = {Proceedings of {The} 29th {International} {Conference} on {Artificial} {Intelligence} and {Statistics}},
	publisher = {PMLR},
	author = {Hai, Do Tien and Mai, Trung Nguyen and Nguyen, TrungTin and Ho, Nhat and Nguyen, Binh T and Drovandi, Christopher},
	year = {2026},
}

@inproceedings{fort_perturbed_2021,
	title = {The {Perturbed} {Prox}-{Preconditioned} {Spider} {Algorithm}: {Non}-{Asymptotic} {Convergence} {Bounds}},
	doi = {10.1109/SSP49050.2021.9513846},
	booktitle = {2021 {IEEE} {Statistical} {Signal} {Processing} {Workshop} ({SSP})},
	author = {Fort, G. and Moulines, E.},
	year = {2021},
	pages = {96--100},
}

@inproceedings{fort_geom_spider_em_2021,
	title = {Geom-{Spider}-{EM}: {Faster} {Variance} {Reduced} {Stochastic} {Expectation} {Maximization} for {Nonconvex} {Finite}-{Sum} {Optimization}},
	doi = {10.1109/ICASSP39728.2021.9414271},
	booktitle = {{ICASSP} 2021 - 2021 {IEEE} {International} {Conference} on {Acoustics}, {Speech} and {Signal} {Processing} ({ICASSP})},
	author = {Fort, Gersende and Moulines, Eric and Wai, Hoi-To},
	year = {2021},
	pages = {3135--3139},
}

@inproceedings{wang_spiderboost_2019,
	title = {{SpiderBoost} and {Momentum}: {Faster} {Variance} {Reduction} {Algorithms}},
	volume = {32},
	url = {https://proceedings.neurips.cc/paper_files/paper/2019/file/512c5cad6c37edb98ae91c8a76c3a291-Paper.pdf},
	booktitle = {Advances in {Neural} {Information} {Processing} {Systems}},
	publisher = {Curran Associates, Inc.},
	author = {Wang, Zhe and Ji, Kaiyi and Zhou, Yi and Liang, Yingbin and Tarokh, Vahid},
	editor = {Wallach, H. and Larochelle, H. and Beygelzimer, A. and Alché-Buc, F. d' and Fox, E. and Garnett, R.},
	year = {2019},
}

@inproceedings{chen_stochastic_2018,
	title = {Stochastic {Expectation} {Maximization} with {Variance} {Reduction}},
	volume = {31},
	url = {https://proceedings.neurips.cc/paper_files/paper/2018/file/aba22f748b1a6dff75bda4fd1ee9fe07-Paper.pdf},
	booktitle = {Advances in {Neural} {Information} {Processing} {Systems}},
	publisher = {Curran Associates, Inc.},
	author = {Chen, Jianfei and Zhu, Jun and Teh, Yee Whye and Zhang, Tong},
	editor = {Bengio, S. and Wallach, H. and Larochelle, H. and Grauman, K. and Cesa-Bianchi, N. and Garnett, R.},
	year = {2018},
}

@article{dieuleveut_federated_2025,
	title = {Federated {Majorize}-{Minimization}: {Beyond} {Parameter} {Aggregation}},
	journal = {arXiv preprint arXiv:2507.17534},
	author = {Dieuleveut, Aymeric and Fort, Gersende and Hegazy, Mahmoud and Wai, Hoi-To},
	year = {2025},
}

@article{dieuleveut_stochastic_2023,
	title = {Stochastic {Approximation} {Beyond} {Gradient} for {Signal} {Processing} and {Machine} {Learning}},
	volume = {71},
	doi = {10.1109/TSP.2023.3301121},
	journal = {IEEE Transactions on Signal Processing},
	author = {Dieuleveut, Aymeric and Fort, Gersende and Moulines, Eric and Wai, Hoi-To},
	year = {2023},
	pages = {3117--3148},
}

@article{fort_stochastic_2023,
	title = {Stochastic variable metric proximal gradient with variance reduction for non-convex composite optimization},
	volume = {33},
	issn = {1573-1375},
	url = {https://doi.org/10.1007/s11222-023-10230-6},
	doi = {10.1007/s11222-023-10230-6},
	number = {3},
	journal = {Statistics and Computing},
	author = {Fort, Gersende and Moulines, Eric},
	month = apr,
	year = {2023},
}

@inproceedings{fort_stochastic_2020,
	title = {A {Stochastic} {Path} {Integral} {Differential} {EstimatoR} {Expectation} {Maximization} {Algorithm}},
	volume = {33},
	url = {https://proceedings.neurips.cc/paper_files/paper/2020/file/c589c3a8f99401b24b9380e86d939842-Paper.pdf},
	booktitle = {Advances in {Neural} {Information} {Processing} {Systems}},
	publisher = {Curran Associates, Inc.},
	author = {Fort, Gersende and Moulines, Eric and Wai, Hoi-To},
	editor = {Larochelle, H. and Ranzato, M. and Hadsell, R. and Balcan, M. F. and Lin, H.},
	year = {2020},
	pages = {16972--16982},
}

@incollection{neal_view_1998,
	address = {Dordrecht},
	title = {A {View} of the {EM} {Algorithm} that {Justifies} {Incremental}, {Sparse}, and other {Variants}},
	isbn = {978-94-011-5014-9},
	url = {https://doi.org/10.1007/978-94-011-5014-9_12},
	booktitle = {Learning in {Graphical} {Models}},
	publisher = {Springer Netherlands},
	author = {Neal, Radford M. and Hinton, Geoffrey E.},
	editor = {Jordan, Michael I.},
	year = {1998},
	doi = {10.1007/978-94-011-5014-9_12},
	pages = {355--368},
}

@incollection{lange_nonconvex_2021,
	title = {Nonconvex {Optimization} via {MM} {Algorithms}: {Convergence} {Theory}},
	isbn = {978-1-118-44511-2},
	url = {https://onlinelibrary.wiley.com/doi/abs/10.1002/9781118445112.stat08295},
	booktitle = {Wiley {StatsRef}: {Statistics} {Reference} {Online}},
	publisher = {John Wiley \& Sons, Ltd},
	author = {Lange, Kenneth and Won, Joong-Ho and Landeros, Alfonso and Zhou, Hua},
	year = {2021},
	doi = {https://doi.org/10.1002/9781118445112.stat08295},
	pages = {1--22},
}

@inproceedings{kingma_adam_2015,
	title = {Adam: {A} method for stochastic optimization},
	booktitle = {The 3rd {International} {Conference} for {Learning} {Representations}},
	author = {Kingma, Diederik P and Ba, Jimmy},
	year = {2015},
}

@article{tibshirani_regression_1996,
	title = {Regression shrinkage and selection via the {Lasso}},
	volume = {58},
	number = {1},
	journal = {Journal of the Royal Statistical Society: Series B (Methodological)},
	author = {Tibshirani, Robert},
	year = {1996},
	note = {Publisher: Wiley Online Library},
	pages = {267--288},
}

@article{friedman_regularization_2010,
	title = {Regularization {Paths} for {Generalized} {Linear} {Models} via {Coordinate} {Descent}},
	volume = {33},
	url = {https://www.jstatsoft.org/index.php/jss/article/view/v033i01},
	doi = {10.18637/jss.v033.i01},
	number = {1},
	urldate = {2025-05-26},
	journal = {Journal of Statistical Software},
	author = {Friedman, Jerome H. and Hastie, Trevor and Tibshirani, Rob},
	month = feb,
	year = {2010},
	note = {Section: Articles},
	pages = {1 -- 22},
}

@article{hinton_neural_2012,
	title = {Neural networks for machine learning: lecture 6a overview of mini-batch gradient descent},
	volume = {14},
	number = {8},
	journal = {University of Toronto, Technical Report},
	author = {Hinton, Geoffrey and Srivastava, Nitish and Swersky, Kevin},
	year = {2012},
	pages = {2},
}

@article{robbins_stochastic_1951,
	title = {A {Stochastic} {Approximation} {Method}},
	volume = {22},
	url = {https://doi.org/10.1214/aoms/1177729586},
	doi = {10.1214/aoms/1177729586},
	number = {3},
	journal = {The Annals of Mathematical Statistics},
	author = {Robbins, Herbert and Monro, Sutton},
	year = {1951},
	note = {Publisher: Institute of Mathematical Statistics},
	pages = {400 -- 407},
}

@inproceedings{liu_sophia_2024,
	title = {Sophia: {A} {Scalable} {Stochastic} {Second}-order {Optimizer} for {Language} {Model} {Pre}-training},
	url = {https://openreview.net/forum?id=3xHDeA8Noi},
	booktitle = {The {Twelfth} {International} {Conference} on {Learning} {Representations}},
	author = {Liu, Hong and Li, Zhiyuan and Hall, David Leo Wright and Liang, Percy and Ma, Tengyu},
	year = {2024},
}

@article{jiang_mixtral_2024,
	title = {Mixtral of experts},
	journal = {arXiv preprint arXiv:2401.04088},
	author = {Jiang, Albert Q and Sablayrolles, Alexandre and Roux, Antoine and Mensch, Arthur and Savary, Blanche and Bamford, Chris and Chaplot, Devendra Singh and Casas, Diego de las and Hanna, Emma Bou and Bressand, Florian and {others}},
	year = {2024},
}

@article{guo_deepseek_r1_2025,
	title = {{DeepSeek}-{R1}: {Incentivizing} {Reasoning} {Capability} in {LLMs} via {Reinforcement} {Learning}},
	journal = {arXiv preprint arXiv:2501.12948},
	author = {Guo, Daya and Yang, Dejian and Zhang, Haowei and Song, Junxiao and Zhang, Ruoyu and Xu, Runxin and Zhu, Qihao and Ma, Shirong and Wang, Peiyi and Bi, Xiao and {others}},
	year = {2025},
}

@inproceedings{dai_deepseekmoe_2024,
	address = {Bangkok, Thailand},
	title = {{DeepSeekMoE}: {Towards} {Ultimate} {Expert} {Specialization} in {Mixture}-of-{Experts} {Language} {Models}},
	url = {https://aclanthology.org/2024.acl-long.70/},
	doi = {10.18653/v1/2024.acl-long.70},
	booktitle = {Proceedings of the 62nd {Annual} {Meeting} of the {Association} for {Computational} {Linguistics} ({Volume} 1: {Long} {Papers})},
	publisher = {Association for Computational Linguistics},
	author = {Dai, Damai and Deng, Chengqi and Zhao, Chenggang and Xu, R.x. and Gao, Huazuo and Chen, Deli and Li, Jiashi and Zeng, Wangding and Yu, Xingkai and Wu, Y. and Xie, Zhenda and Li, Y.k. and Huang, Panpan and Luo, Fuli and Ruan, Chong and Sui, Zhifang and Liang, Wenfeng},
	editor = {Ku, Lun-Wei and Martins, Andre and Srikumar, Vivek},
	month = aug,
	year = {2024},
	pages = {1280--1297},
}

@inproceedings{oudoumanessah2025,
      title={Cluster globally, Reduce locally: Scalable efficient dictionary compression for magnetic resonance fingerprinting}, 
      author={Geoffroy Oudoumanessah and Thomas Coudert and Luc Meyer and Aurelien Delphin and Michel Dojat and Carole Lartizien and Florence Forbes},
      year={2025},
      booktitle={IEEE International Symposium on Biological Imaging (ISBI)}
}

@inproceedings{nguyen_joint_2024,
	title = {Joint learning of {Gaussian} graphical models in heterogeneous dependencies of high-dimensional transcriptomic data},
	url = {https://openreview.net/forum?id=slJkBPxw9x},
	booktitle = {The 16th {Asian} {Conference} on {Machine} {Learning} ({Conference} {Track})},
	author = {Nguyen, Dung Ngoc and Li, Zitong},
	year = {2024},
}

@article{le_thi_online_2020,
	title = {Online {Learning} {Based} on {Online} {DCA} and {Application} to {Online} {Classification}},
	volume = {32},
	issn = {0899-7667},
	url = {https://doi.org/10.1162/neco_a_01266},
	doi = {10.1162/neco_a_01266},
	number = {4},
	urldate = {2025-01-19},
	journal = {Neural Computation},
	author = {Le Thi, Hoai An and Ho, Vinh Thanh},
	month = apr,
	year = {2020},
	pages = {759--793},
}

@article{le_thi_dca_2021,
	title = {{DCA} for online prediction with expert advice},
	volume = {33},
	issn = {1433-3058},
	url = {https://doi.org/10.1007/s00521-021-05709-0},
	doi = {10.1007/s00521-021-05709-0},
	number = {15},
	journal = {Neural Computing and Applications},
	author = {Le Thi, Hoai An and Ho, Vinh Thanh},
	month = aug,
	year = {2021},
	pages = {9521--9544},
}

@article{chouzenoux_stochastic_2017,
	title = {A {Stochastic} {Majorize}-{Minimize} {Subspace} {Algorithm} for {Online} {Penalized} {Least} {Squares} {Estimation}},
	volume = {65},
	doi = {10.1109/TSP.2017.2709265},
	number = {18},
	journal = {IEEE Transactions on Signal Processing},
	author = {Chouzenoux, Emilie and Pesquet, Jean-Christophe},
	year = {2017},
	pages = {4770--4783},
}

@article{cappe2009line,
  title={On-line expectation--maximization algorithm for latent data models},
  author={Capp{\'e}, Olivier and Moulines, Eric},
  journal={Journal of the Royal Statistical Society Series B: Statistical Methodology},
  volume={71},
  number={3},
  pages={593--613},
  year={2009},
  publisher={Oxford University Press}
}

@inproceedings{karimi_minimization_2022,
	series = {Proceedings of {Machine} {Learning} {Research}},
	title = {Minimization by {Incremental} {Stochastic} {Surrogate} {Optimization} for {Large} {Scale} {Nonconvex} {Problems}},
	volume = {167},
	url = {https://proceedings.mlr.press/v167/karimi22a.html},
	booktitle = {Proceedings of {The} 33rd {International} {Conference} on {Algorithmic} {Learning} {Theory}},
	publisher = {PMLR},
	author = {Karimi, Belhal and Wai, Hoi-To and Moulines, Eric and Li, Ping},
	editor = {Dasgupta, Sanjoy and Haghtalab, Nika},
	month = apr,
	year = {2022},
	pages = {606--637},
}

@article{mairal2013stochastic,
  title={Stochastic majorization-minimization algorithms for large-scale optimization},
  author={Mairal, Julien},
  journal={Advances in Neural Information Processing Systems},
  volume={26},
  year={2013}
}

@inproceedings{rohde2011online,
  title={Online maximum-likelihood estimation for latent factor models},
  author={Rohde, David and Capp{\'e}, Olivier},
  booktitle={2011 IEEE Statistical Signal Processing Workshop (SSP)},
  pages={565--568},
  year={2011},
  organization={IEEE}
}

@article{mensch_stochastic_2018,
	title = {Stochastic {Subsampling} for {Factorizing} {Huge} {Matrices}},
	volume = {66},
	doi = {10.1109/TSP.2017.2752697},
	number = {1},
	journal = {IEEE Transactions on Signal Processing},
	author = {Mensch, Arthur and Mairal, Julien and Thirion, Bertrand and Varoquaux, Gaël},
	year = {2018},
	pages = {113--128},
}

@inproceedings{cappe2009online,
  title={Online sequential monte carlo em algorithm},
  author={Capp{\'e}, Olivier},
  booktitle={2009 IEEE/SP 15th Workshop on Statistical Signal Processing},
  pages={37--40},
  year={2009},
  organization={IEEE}
}

@article{cappe2011online,
  title={Online EM algorithm for hidden Markov models},
  author={Capp{\'e}, Olivier},
  journal={Journal of Computational and Graphical Statistics},
  volume={20},
  number={3},
  pages={728--749},
  year={2011},
  publisher={Taylor \& Francis}
}

@misc{communities_and_crime_183,
  author       = {Redmond, Michael},
  title        = {{Communities and Crime}},
  year         = {2002},
  howpublished = {UCI Machine Learning Repository},
  note         = {{DOI}: https://doi.org/10.24432/C53W3X}
}

@article{blein2020systems,
  title={A systems genetics approach reveals environment-dependent associations between SNPs, protein coexpression, and drought-related traits in maize},
  author={Blein-Nicolas, M{\'e}lisande and Negro, Sandra Sylvia and Balliau, Thierry and Welcker, Claude and Cabrera-Bosquet, Lloren{\c{c}} and Nicolas, St{\'e}phane Dimitri and Charcosset, Alain and Zivy, Michel},
  journal={Genome Research},
  volume={30},
  number={11},
  pages={1593--1604},
  year={2020},
  publisher={Cold Spring Harbor Lab}
}

@article{prado2018phenomics,
  title={Phenomics allows identification of genomic regions affecting maize stomatal conductance with conditional effects of water deficit and evaporative demand},
  author={Prado, Santiago Alvarez and Cabrera-Bosquet, Lloren{\c{c}} and Grau, Antonin and Coupel-Ledru, Aude and Millet, Emilie J and Welcker, Claude and Tardieu, Fran{\c{c}}ois},
  journal={Plant, Cell \& Environment},
  volume={41},
  number={2},
  pages={314--326},
  year={2018},
  publisher={Wiley Online Library}
}

@misc{oudoumanessah2024,
      title={Scalable magnetic resonance fingerprinting: Incremental inference of high dimensional elliptical mixtures from large data volumes}, 
      author={Geoffroy Oudoumanessah and Thomas Coudert and Carole Lartizien and Michel Dojat and Thomas Christen and Florence Forbes},
      year={2024},
      eprint={2412.10173},
      archivePrefix={arXiv},
      primaryClass={stat.AP},
}

@article{Kugler2022,
  TITLE = {Fast {B}ayesian Inversion for high dimensional inverse problems},
  AUTHOR = {Kugler, B. and Forbes, F. and Dout{\'e}, S.},
  journal = {Statistics and Computing},
  volume={32},
  number={2},
  pages={31},
  year={2022}
}

@article{boux2021,
  title={Bayesian inverse regression for vascular magnetic resonance fingerprinting},
  author={Boux, Fabien and Forbes, Florence and Arbel, Julyan and Lemasson, Benjamin and Barbier, Emmanuel L},
  journal={IEEE Transactions on Medical Imaging},
  volume={40},
  number={7},
  pages={1827--1837},
  year={2021},
  publisher={IEEE}
}

@article{lyu_online_2022,
	title = {Online {Nonnegative} {CP}-dictionary {Learning} for {Markovian} {Data}},
	volume = {23},
	url = {http://jmlr.org/papers/v23/21-0419.html},
	number = {148},
	journal = {Journal of Machine Learning Research},
	author = {Lyu, Hanbaek and Strohmeier, Christopher and Needell, Deanna},
	year = {2022},
	pages = {1--50},
}

@article{lyu2024stochastic,
  title={Stochastic regularized majorization-minimization with weakly convex and multi-convex surrogates},
  author={Lyu, Hanbaek},
  journal={Journal of Machine Learning Research},
  volume={25},
  pages={1--83},
  year={2024}
}

@article{lupu2024convergence,
  title={Convergence analysis of stochastic higher-order majorization--minimization algorithms},
  author={Lupu, Daniela and Necoara, Ion},
  journal={Optimization Methods and Software},
  volume={39},
  pages={384--413},
  year={2024}
}

@article{chouzenoux2022sabrina,
  title={SABRINA: A stochastic subspace majorization-minimization algorithm},
  author={Chouzenoux, Emilie and Fest, Jean-Baptiste},
  journal={Journal of Optimization Theory and Applications},
  volume={195},
  pages={919--952},
  year={2022}}

@article{jiang2024stochastic,
  title={The stochastic proximal distance algorithm},
  author={Jiang, Haoyu and Xu, Jason},
  journal={Statistics and Computing},
  volume={34},
  pages={210},
  year={2024}
}

@book{mclachlan_em_1997,
	title = {The {EM} {Algorithm} and {Extensions}},
	publisher = {Wiley},
	author = {McLachlan, Geoffrey J and Krishnan, Thriyambakam},
	year = {1997},
}

@article{delyon_convergence_1999,
	title = {Convergence of a stochastic approximation version of the {EM} algorithm},
	journal = {Annals of statistics},
	author = {Delyon, Bernard and Lavielle, Marc and Moulines, Eric},
	year = {1999},
	pages = {94--128},
}

@article{kuhn_properties_2020,
	title = {Properties of the stochastic approximation {EM} algorithm with mini-batch sampling},
	volume = {30},
	issn = {1573-1375},
	url = {https://doi.org/10.1007/s11222-020-09968-0},
	doi = {10.1007/s11222-020-09968-0},
	number = {6},
	journal = {Statistics and Computing},
	author = {Kuhn, Estelle and Matias, Catherine and Rebafka, Tabea},
	month = nov,
	year = {2020},
	pages = {1725--1739},
}

@article{pham_competesmoeeffective_2024,
	title = {{CompeteSMoE}–{Effective} {Training} of {Sparse} {Mixture} of {Experts} via {Competition}},
	copyright = {All rights reserved},
	journal = {arXiv preprint arXiv:2402.02526},
	author = {Pham, Quang and Do, Giang and Nguyen, Huy and Nguyen, TrungTin and Liu, Chenghao and Sartipi, Mina and Nguyen, Binh T and Ramasamy, Savitha and Li, Xiaoli and Hoi, Steven and {others}},
	year = {2024},
}

@inproceedings{lathuiliere_deep_2017,
	title = {Deep mixture of linear inverse regressions applied to head-pose estimation},
	booktitle = {Proceedings of the {IEEE} {Conference} on {Computer} {Vision} and {Pattern} {Recognition}},
	author = {Lathuilière, Stéphane and Juge, Rémi and Mesejo, Pablo and Muñoz-Salinas, Rafael and Horaud, Radu},
	year = {2017},
	pages = {4817--4825},
}

@inproceedings{you_speechmoe2_2022,
	title = {Speechmoe2: {Mixture}-of-{Experts} {Model} with {Improved} {Routing}},
	url = {https://doi.org/10.1109/ICASSP43922.2022.9747065},
	doi = {10.1109/ICASSP43922.2022.9747065},
	booktitle = {{IEEE} {International} {Conference} on {Acoustics}, {Speech} and {Signal} {Processing}, {ICASSP} 2022, {Virtual} and {Singapore}, 23-27 {May} 2022},
	publisher = {IEEE},
	author = {You, Zhao and Feng, Shulin and Su, Dan and Yu, Dong},
	year = {2022},
	pages = {7217--7221},
}

@article{kuusela_semi_supervised_2012,
	title = {Semi-supervised anomaly detection – towards model-independent searches of new physics},
	volume = {368},
	url = {https://dx.doi.org/10.1088/1742-6596/368/1/012032},
	doi = {10.1088/1742-6596/368/1/012032},
	number = {1},
	journal = {Journal of Physics: Conference Series},
	author = {Kuusela, Mikael and Vatanen, Tommi and Malmi, Eric and Raiko, Tapani and Aaltonen, Timo and Nagai, Yoshikazu},
	month = jun,
	year = {2012},
	pages = {012032},
}

@article{li_drug_2019,
	title = {Drug sensitivity prediction with high-dimensional mixture regression},
	volume = {14},
	url = {https://doi.org/10.1371/journal.pone.0212108},
	doi = {10.1371/journal.pone.0212108},
	number = {2},
	journal = {PLOS ONE},
	author = {Li, Qianyun and Shi, Runmin and Liang, Faming},
	month = feb,
	year = {2019},
	pages = {1--18},
}

@inproceedings{karimi_global_2019,
	title = {On the {Global} {Convergence} of ({Fast}) {Incremental} {Expectation} {Maximization} {Methods}},
	volume = {32},
	url = {https://proceedings.neurips.cc/paper_files/paper/2019/file/a14ac55a4f27472c5d894ec1c3c743d2-Paper.pdf},
	booktitle = {Advances in {Neural} {Information} {Processing} {Systems}},
	author = {Karimi, Belhal and Wai, Hoi-To and Moulines, Eric and Lavielle, Marc},
	year = {2019},
}

@article{cappe_online_2011,
	title = {Online {EM} {Algorithm} for {Hidden} {Markov} {Models}},
	volume = {20},
	issn = {1061-8600},
	url = {https://doi.org/10.1198/jcgs.2011.09109},
	doi = {10.1198/jcgs.2011.09109},
	number = {3},
	journal = {Journal of Computational and Graphical Statistics},
	author = {Cappé, Olivier},
	month = jan,
	year = {2011},
	pages = {728--749},
}

@article{corff_online_2013,
	title = {Online {Expectation} {Maximization} based algorithms for inference in {Hidden} {Markov} {Models}},
	volume = {7},
	url = {https://doi.org/10.1214/13-EJS789},
	doi = {10.1214/13-EJS789},
	number = {none},
	journal = {Electronic Journal of Statistics},
	author = {Corff, Sylvain Le and Fort, Gersende},
	year = {2013},
	pages = {763 -- 792},
}

@article{razaviyayn_stochastic_2016,
	title = {A {Stochastic} {Successive} {Minimization} {Method} for {Nonsmooth} {Nonconvex} {Optimization} with {Applications} to {Transceiver} {Design} in {Wireless} {Communication} {Networks}},
	volume = {157},
	issn = {1436-4646},
	url = {https://doi.org/10.1007/s10107-016-1021-7},
	doi = {10.1007/s10107-016-1021-7},
	number = {2},
	journal = {Mathematical Programming},
	author = {Razaviyayn, Meisam and Sanjabi, Maziar and Luo, Zhi-Quan},
	month = jun,
	year = {2016},
	pages = {515--545},
}

@inproceedings{nguyen_bayesian_2024,
	title = {Bayesian {Likelihood} {Free} {Inference} using {Mixtures} of {Experts}},
	copyright = {All rights reserved},
	doi = {10.1109/IJCNN60899.2024.10650052},
	booktitle = {2024 {International} {Joint} {Conference} on {Neural} {Networks} ({IJCNN})},
	author = {Nguyen, Hien Duy and Nguyen, TrungTin and Forbes, Florence},
	year = {2024},
	pages = {1--8},
}

@article{mairal_incremental_2015,
	title = {Incremental {Majorization}-{Minimization} {Optimization} with {Application} to {Large}-{Scale} {Machine} {Learning}},
	volume = {25},
	url = {https://doi.org/10.1137/140957639},
	doi = {10.1137/140957639},
	number = {2},
	journal = {SIAM Journal on Optimization},
	author = {Mairal, Julien},
	year = {2015},
	pages = {829--855},
}

@article{shalev_shwartz_online_2012,
	title = {Online {Learning} and {Online} {Convex} {Optimization}},
	volume = {4},
	issn = {1935-8237},
	url = {http://dx.doi.org/10.1561/2200000018},
	doi = {10.1561/2200000018},
	number = {2},
	journal = {Foundations and Trends® in Machine Learning},
	author = {Shalev-Shwartz, Shai},
	year = {2012},
	pages = {107--194},
}

@book{lan_first_order_2020,
	title = {First-order and stochastic optimization methods for machine learning},
	volume = {1},
	publisher = {Springer},
	author = {Lan, Guanghui},
	year = {2020},
}

@book{hazan_introduction_2016,
	title = {Introduction to online convex optimization},
	volume = {2},
	publisher = {Now Publishers, Inc.},
	author = {Hazan, Elad},
	year = {2016},
}

@inproceedings{karimi_non_asymptotic_2019,
	series = {Proceedings of {Machine} {Learning} {Research}},
	title = {Non-asymptotic {Analysis} of {Biased} {Stochastic} {Approximation} {Scheme}},
	volume = {99},
	url = {https://proceedings.mlr.press/v99/karimi19a.html},
	booktitle = {Proceedings of the {Thirty}-{Second} {Conference} on {Learning} {Theory}},
	publisher = {PMLR},
	author = {Karimi, Belhal and Miasojedow, Blazej and Moulines, Eric and Wai, Hoi-To},
	editor = {Beygelzimer, Alina and Hsu, Daniel},
	month = jun,
	year = {2019},
	pages = {1944--1974},
}

@article{fort_fast_2021,
	title = {Fast incremental expectation maximization for finite-sum optimization: nonasymptotic convergence},
	volume = {31},
	issn = {1573-1375},
	url = {https://doi.org/10.1007/s11222-021-10023-9},
	doi = {10.1007/s11222-021-10023-9},
	number = {4},
	journal = {Statistics and Computing},
	author = {Fort, G. and Gach, P. and Moulines, E.},
	month = jun,
	year = {2021},
	pages = {48},
}

@book{borkar_stochastic_2008,
	title = {Stochastic approximation: a dynamical systems viewpoint},
	volume = {9},
	publisher = {Springer},
	author = {Borkar, Vivek S},
	year = {2008},
}

@article{meilijson_fast_1989,
	title = {A {Fast} {Improvement} to the {Em} {Algorithm} on its {Own} {Terms}},
	volume = {51},
	issn = {0035-9246},
	url = {https://doi.org/10.1111/j.2517-6161.1989.tb01754.x},
	doi = {10.1111/j.2517-6161.1989.tb01754.x},
	number = {1},
	journal = {Journal of the Royal Statistical Society: Series B (Methodological)},
	author = {Meilijson, Isaac},
	month = sep,
	year = {1989},
	pages = {127--138},
}

@article{nguyen_introduction_2017,
	title = {An introduction to {Majorization}-{Minimization} algorithms for machine learning and statistical estimation},
	volume = {7},
	url = {https://wires.onlinelibrary.wiley.com/doi/abs/10.1002/widm.1198},
	doi = {https://doi.org/10.1002/widm.1198},
	number = {2},
	journal = {WIREs Data Mining and Knowledge Discovery},
	author = {Nguyen, Hien D.},
	year = {2017},
	pages = {e1198},
}

@article{leeuw_application_1977,
	title = {Application of convex analysis to multidimensional scaling},
	journal = {Recent developments in statistics},
	author = {{De Leeuw}, J},
	year = {1977},
	pages = {133--145},
}

@article{jaquier2021tensor,
  title={Tensor-variate mixture of experts for proportional myographic control of a robotic hand},
  author={Jaquier, No{\'e}mie and Haschke, Robert and Calinon, Sylvain},
  journal={Robotics and Autonomous Systems},
  volume={142},
  pages={103812},
  year={2021}
}

@book{ortega_iterative_1970,
	title = {Iterative {Solution} of {Nonlinear} {Equations} in {Several} {Variables}},
	volume = {30},
	publisher = {SIAM},
	author = {Ortega, JM and Rheinboldt, WC},
	year = {1970},
}

@article{de_leeuw_convergence_1977,
	title = {Convergence of correction matrix algorithms for multidimensional scaling},
	volume = {36},
	journal = {Geometric representations of relational data},
	author = {De Leeuw, Jan and Heiser, Willem J},
	year = {1977},
	pages = {735--752},
}

@article{andrieu_stability_2005,
	title = {Stability of {Stochastic} {Approximation} under {Verifiable} {Conditions}},
	volume = {44},
	number = {1},
	journal = {SIAM Journal on Control and Optimization},
	author = {Andrieu, Christophe and Moulines, Eric and Priouret, Pierre},
	year = {2005},
	pages = {283--312},
}

@book{kushner_stochastic_2003,
	title = {Stochastic {Approximation} and {Recursive} {Algorithms} and {Applications}},
	volume = {35},
	publisher = {Springer Science \& Business Media},
	author = {Kushner, Harold and Yin, G George},
	year = {2003},
}

@inproceedings{puigcerver_sparse_2024,
	title = {From {Sparse} to {Soft} {Mixtures} of {Experts}},
	url = {https://openreview.net/forum?id=jxpsAj7ltE},
	booktitle = {The {Twelfth} {International} {Conference} on {Learning} {Representations}},
	author = {Puigcerver, Joan and Ruiz, Carlos Riquelme and Mustafa, Basil and Houlsby, Neil},
	year = {2024},
}

@article{fedus_switch_2022,
	title = {Switch {Transformers}: {Scaling} to {Trillion} {Parameter} {Models} with {Simple} and {Efficient} {Sparsity}},
	volume = {23},
	url = {http://jmlr.org/papers/v23/21-0998.html},
	number = {120},
	journal = {Journal of Machine Learning Research},
	author = {Fedus, William and Zoph, Barret and Shazeer, Noam},
	year = {2022},
	pages = {1--39},
}

@inproceedings{shazeer_outrageously_2017,
	title = {Outrageously {Large} {Neural} {Networks}: {The} {Sparsely}-{Gated} {Mixture}-of-{Experts} {Layer}},
	url = {https://openreview.net/forum?id=B1ckMDqlg},
	booktitle = {International {Conference} on {Learning} {Representations}},
	author = {Shazeer, Noam and Mirhoseini, *Azalia and Maziarz, *Krzysztof and Davis, Andy and Le, Quoc and Hinton, Geoffrey and Dean, Jeff},
	year = {2017},
}

@article{chong_risk_2024,
	title = {Risk {Bounds} for {Mixture} {Density} {Estimation} on {Compact} {Domains} via the h-{Lifted} {Kullback}–{Leibler} {Divergence}},
	issn = {2835-8856},
	url = {https://openreview.net/forum?id=lAKvQO4vHj},
	journal = {Transactions on Machine Learning Research},
	author = {Chong, Mark Chiu and Nguyen, Hien Duy and Nguyen, TrungTin},
	year = {2024},
}

@article{nguyen_mini_batch_2020,
	title = {Mini-batch learning of exponential family finite mixture models},
	volume = {30},
	issn = {1573-1375},
	url = {https://doi.org/10.1007/s11222-019-09919-4},
	doi = {10.1007/s11222-019-09919-4},
	number = {4},
	journal = {Statistics and Computing},
	author = {Nguyen, Hien D. and Forbes, Florence and McLachlan, Geoffrey J.},
	month = jul,
	year = {2020},
	pages = {731--748},
}

@article{le_thi_online_2024,
	title = {Online {Stochastic} {DCA} {With} {Applications} to {Principal} {Component} {Analysis}},
	volume = {35},
	doi = {10.1109/TNNLS.2022.3213558},
	number = {5},
	journal = {IEEE Transactions on Neural Networks and Learning Systems},
	author = {Le Thi, Hoai An and Luu, Hoang Phuc Hau and Dinh, Tao Pham},
	year = {2024},
	pages = {7035--7047},
}

@article{sun_majorization_minimization_2017,
	title = {Majorization-{Minimization} {Algorithms} in {Signal} {Processing}, {Communications}, and {Machine} {Learning}},
	volume = {65},
	doi = {10.1109/TSP.2016.2601299},
	number = {3},
	journal = {IEEE Transactions on Signal Processing},
	author = {Sun, Ying and Babu, Prabhu and Palomar, Daniel P.},
	year = {2017},
	pages = {794--816},
}

@article{nguyen_laplace_2016,
	title = {Laplace mixture of linear experts},
	volume = {93},
	issn = {0167-9473},
	url = {https://www.sciencedirect.com/science/article/pii/S0167947314003089},
	doi = {https://doi.org/10.1016/j.csda.2014.10.016},
	journal = {Computational Statistics \& Data Analysis},
	author = {Nguyen, Hien D and McLachlan, Geoffrey J},
	year = {2016},
	pages = {177--191},
}

@book{lange_numerical_2010,
	title = {Numerical analysis for statisticians},
	volume = {1},
	publisher = {Springer},
	author = {Lange, Kenneth},
	year = {2010},
}

@inproceedings{nguyen_online_2023,
	title = {An {Online} {Minorization}-{Maximization} {Algorithm}},
	isbn = {978-3-031-09033-2},
	doi = {10.1007/978-3-031-09034-9_29},
	language = {English},
	booktitle = {Classification and {Data} {Science} in the {Digital} {Age} - 17th {Conference} of the {International} {Federation} of {Classification} {Societies}, {IFCS} 2022, {Proceedings}},
	publisher = {Springer},
	author = {Nguyen, Hien Duy and Forbes, Florence and Fort, Gersende and Cappé, Olivier},
	year = {2023},
	pages = {263--271},
}

@article{cappe_Online_2009,
	title = {On-{Line} {Expectation}–{Maximization} {Algorithm} for latent {Data} {Models}},
	volume = {71},
	issn = {1369-7412},
	url = {https://doi.org/10.1111/j.1467-9868.2009.00698.x},
	doi = {10.1111/j.1467-9868.2009.00698.x},
	number = {3},
	journal = {Journal of the Royal Statistical Society Series B: Statistical Methodology},
	author = {Cappé, Olivier and Moulines, Eric},
	month = feb,
	year = {2009},
	pages = {593--613},
}

@article{fort_sequential_2024,
	title = {Sequential {Sample} {Average} {Majorization}–{Minimization}},
	url = {https://hal.science/hal-04607609},
	journal = {Statistics and Computing},
	author = {Fort, Gersende and Forbes, Florence and Nguyen, Hien Duy},
	year = {2025},
note ={Accepted for publication}
}

@inproceedings{nguyen_demystifying_2023,
    title = {Demystifying {Softmax} {Gating} in {Gaussian} {Mixture} of {Experts}},
    author = {Nguyen, Huy and Nguyen, TrungTin and Ho, Nhat},
      booktitle = "Advances in Neural Information Processing Systems",
    month = dec,
      year={2023}
}

@inproceedings{nguyen_general_2024,
	title = {A {General Theory} for {Softmax Gating Multinomial Logistic Mixture of Experts}},
	copyright = {All rights reserved},
	booktitle = {Proceedings of The 41st International Conference on Machine Learning},
	author = {Nguyen, Huy and Akbarian, Pedram and Nguyen, TrungTin and Ho, Nhat},
	year = {2024},
}

@inproceedings{nguyen_towards_2024,
	title = {Towards {Convergence} {Rates} for {Parameter} {Estimation} in {Gaussian}-gated {Mixture} of {Experts}},
	volume = {238},
	copyright = {All rights reserved},
	url = {https://proceedings.mlr.press/v238/nguyen24b.html},
	booktitle = {Proceedings of {The} 27th {International} {Conference} on {Artificial} {Intelligence} and {Statistics}},
	author = {Nguyen, Huy and Nguyen, TrungTin and Nguyen, Khai and Ho, Nhat},
	month = may,
	year = {2024},
	pages = {2683--2691},
}

@article{blein_nicolas_nonlinear_2024,
	title = {Nonlinear network-based quantitative trait prediction from biological data},
	journal = {Journal of the Royal Statistical Society Series C: Applied Statistics},
	author = {Blein-Nicolas, Mélisande and Devijver, Emilie and Gallopin, Mélina and Perthame, Emeline},
	month = mar,
	year = {2024},
	pages = {qlae012},
}

@inproceedings{chen_towards_2022,
	title = {Towards {Understanding} the {Mixture}-of-{Experts} {Layer} in {Deep} {Learning}},
	booktitle = {Advances in {Neural} {Information} {Processing} {Systems}},
	author = {Chen, Zixiang and Deng, Yihe and Wu, Yue and Gu, Quanquan and Li, Yuanzhi},
	year = {2022},
}

@article{forbes_summary_2022,
	title = {Summary statistics and discrepancy measures for approximate {Bayesian} computation via surrogate posteriors},
	volume = {32},
	copyright = {All rights reserved},
	issn = {1573-1375},
	number = {5},
	journal = {Statistics and Computing},
	author = {Forbes, Florence and Nguyen, Hien Duy and Nguyen, TrungTin and Arbel, Julyan},
	month = oct,
	year = {2022},
	pages = {85},
}

@article{nguyen_approximations_2021,
	title = {Approximations of conditional probability density functions in {Lebesgue} spaces via mixture of experts models},
	volume = {8},
	copyright = {All rights reserved},
	issn = {2195-5832},
	number = {1},
	urldate = {2022-01-21},
	journal = {Journal of Statistical Distributions and Applications},
	author = {Nguyen, Hien Duy and Nguyen, TrungTin and Chamroukhi, Faicel and McLachlan, Geoffrey John},
	month = aug,
	year = {2021},
	pages = {13},
}

@article{nguyen_approximation_2022,
	title = {Approximation of probability density functions via location-scale finite mixtures in {Lebesgue} spaces},
	copyright = {All rights reserved},
	issn = {0361-0926},
	journal = {Communications in Statistics - Theory and Methods},
	author = {Nguyen, TrungTin and Chamroukhi, Faicel and Nguyen, Hien D. and McLachlan, Geoffrey J.},
	month = may,
	year = {2022},
	pages = {1--12},
}

@article{nguyen_approximation_2020,
	title={Approximation by finite mixtures of continuous density functions that vanish at infinity},
	author={Nguyen, TrungTin and Nguyen, Hien D and Chamroukhi, Faicel and McLachlan, Geoffrey J},
	journal={Cogent Mathematics \& Statistics},
	volume={7},
	number={1},
	pages={1750861},
	year={2020},
	publisher={Cogent OA}
}

@phdthesis{nguyen_model_2021,
	type = {{PhD} {Thesis}},
	title = {Model {Selection} and {Approximation} in {High}-dimensional {Mixtures} of {Experts} {Models}: from {Theory} to {Practice}},
	copyright = {All rights reserved},
	shorttitle = {Model selection and approximation in high-dimensional mixtures of experts models},
	language = {en},
	urldate = {2022-01-21},
	school = {Normandie Université},
	author = {Nguyen, TrungTin},
	month = dec,
	year = {2021},
}

@article{do_strong_2022,
	title = {Strong identifiability and parameter learning in regression with heterogeneous response},
	journal = {arXiv preprint arXiv:2212.04091},
	author = {Do, Dat and Do, Linh and Nguyen, XuanLong},
	year = {2022},
}

@INPROCEEDINGS{chamroukhi2019regularizedIJCNN,
	author={Chamroukhi, Faicel and Huynh, Bao Tuyen},
	booktitle={2018 International Joint Conference on Neural Networks (IJCNN)}, 
	title={Regularized Maximum-Likelihood Estimation of Mixture-of-Experts for Regression and Clustering}, 
	year={2018},
	volume={},
	number={},
	pages={1-8},}

@article{chamroukhi2019regularized,
	title={Regularized Maximum Likelihood Estimation and Feature Selection in Mixtures-of-Experts Models},
	author={Chamroukhi, Faicel and Huynh, Bao Tuyen},
	journal={Journal de la Soci{\'e}t{\'e} Fran{\c{c}}aise de Statistique},
	volume={160},
	number={1},
	pages={57--85},
	year={2019}
}

@inproceedings{bengio_deep_2013,
	address = {Berlin, Heidelberg},
	title = {Deep {Learning} of {Representations}: {Looking} {Forward}},
	isbn = {978-3-642-39593-2},
	booktitle = {Statistical {Language} and {Speech} {Processing}},
	author = {Bengio, Yoshua},
	year = {2013},
	pages = {1--37},
}

@article{ho_strong_2016,
	title = {On strong identifiability and convergence rates of parameter estimation in finite mixtures},
	volume = {10},
	number = {1},
	journal = {Electronic Journal of Statistics},
	author = {Ho, Nhat and Nguyen, XuanLong},
	year = {2016},
	pages = {271--307},
}

@article{ho_convergence_2016,
	title = {Convergence rates of parameter estimation for some weakly identifiable finite mixtures},
	volume = {44},
	number = {6},
	journal = {The Annals of Statistics},
	author = {Ho, Nhat and Nguyen, XuanLong},
	year = {2016},
	pages = {2726 -- 2755},
}

@article{norets_approximation_2010,
	title = {Approximation of conditional densities by smooth mixtures of regressions},
	volume = {38},
	number = {3},
	journal = {The Annals of Statistics},
	author = {Norets, Andriy},
	year = {2010},
	pages = {1733 -- 1766},
}

@article{jiang_hierarchical_1999,
	title = {Hierarchical mixtures-of-experts for exponential family regression models: approximation and maximum likelihood estimation},
	journal = {Annals of Statistics},
	author = {Jiang, Wenxin and Tanner, Martin A},
	year = {1999},
	pages = {987--1011},
}

@article{ho_convergence_2022,
	title = {Convergence {Rates} for {Gaussian} {Mixtures} of {Experts}},
	journal = {Journal of Machine Learning Research},
	author = {Ho, Nhat and Yang, Chiao-Yu and Jordan, Michael I},
	year = {2022},
}

@article{deleforge2015hyper,
	author = {Deleforge, A and Forbes, F and Ba, S and Horaud, R},
	issn = {1941-0484 VO - 9},
	journal = {IEEE Journal of Selected Topics in Signal Processing},
	number = {6},
	pages = {1037--1048},
	title = {{Hyper-Spectral Image Analysis With Partially Latent Regression and Spatial Markov Dependencies}},
	volume = {9},
	year = {2015}
}

@article{rakhlin2005risk,
	author = {Rakhlin, Alexander and Panchenko, Dmitry and Mukherjee, Sayan},
	journal = {ESAIM: PS},
	pages = {220--229},
	title = {{Risk bounds for mixture density estimation}},
	volume = {9},
	year = {2005}
}

@article{nguyen2019approximation,
	author = {Nguyen, Hien D and Chamroukhi, Faicel and Forbes, Florence},
	issn = {0925-2312},
	journal = {Neurocomputing},
	pages = {208--214},
	title = {{Approximation results regarding the multiple-output Gaussian gated mixture of linear experts model}},
	volume = {366},
	year = {2019}
}

@article{masoudnia2014mixture,
	author = {Masoudnia, Saeed and Ebrahimpour, Reza},
	issn = {1573-7462},
	journal = {Artificial Intelligence Review},
	number = {2},
	pages = {275--293},
	title = {{Mixture of experts: a literature survey}},
	volume = {42},
	year = {2014}
}

@book{lange2016mm,
	author = {Lange, Kenneth},
	title = {MM Optimization Algorithms},
	publisher = {Society for Industrial and Applied Mathematics},
	year = {2016},
	address = {Philadelphia, PA},
	edition   = {},
	eprint = {https://epubs.siam.org/doi/pdf/10.1137/1.9781611974409}
}

@article{hennig2000identifiablity,
	author = {Hennig, C},
	issn = {1432-1343},
	journal = {Journal of Classification},
	number = {2},
	pages = {273--296},
	title = {{Identifiablity of Models for Clusterwise Linear Regression}},
	volume = {17},
	year = {2000}
}

@article{jiang1999identifiability,
	author = {Jiang, W and Tanner, M A},
	issn = {0893-6080},
	journal = {Neural Networks},
	number = {9},
	pages = {1253--1258},
	title = {{On the identifiability of mixtures-of-experts}},
	volume = {12},
	year = {1999}
}

@article{dempster1977maximum,
	author = {Dempster, A P and Laird, N M and Rubin, D B},
	issn = {00359246},
	journal = {Journal of the Royal Statistical Society. Series B (Methodological)},
	month = {mar},
	number = {1},
	pages = {1--38},
	publisher = {[Royal Statistical Society, Wiley]},
	title = {{Maximum Likelihood from Incomplete Data via the EM Algorithm}},
	volume = {39},
	year = {1977}
}

@article{yuksel2012twenty,
	author = {Yuksel, S E and Wilson, J N and Gader, P D},
	issn = {2162-2388 VO - 23},
	journal = {IEEE Transactions on Neural Networks and Learning Systems},
	number = {8},
	pages = {1177--1193},
	title = {{Twenty Years of Mixture of Experts}},
	volume = {23},
	year = {2012}
}

@article{nguyen2013convergence,
	author = {Nguyen, XuanLong},
	journal = {The Annals of Statistics},
	number = {1},
	pages = {370--400},
	publisher = {Institute of Mathematical Statistics},
	title = {{Convergence of latent mixing measures in finite and infinite mixture models}},
	volume = {41},
	year = {2013}
}

@article{mendes2012convergence,
	title={On convergence rates of mixtures of polynomial experts},
	author={Mendes, Eduardo F and Jiang, Wenxin},
	journal={Neural computation},
	volume={24},
	number={11},
	pages={3025--3051},
	year={2012},
	publisher={MIT Press}
}

@article{genovese2000rates,
	author = {Genovese, Christopher R and Wasserman, Larry},
	journal = {The Annals of Statistics},
	month = {aug},
	number = {4},
	pages = {1105--1127},
	title = {{Rates of convergence for the Gaussian mixture sieve}},
	volume = {28},
	year = {2000}
}

@article{nguyen2016universal,
	title={A universal approximation theorem for mixture-of-experts models},
	author={Nguyen, Hien D and Lloyd-Jones, Luke R and McLachlan, Geoffrey J},
	journal={Neural computation},
	volume={28},
	number={12},
	pages={2585--2593},
	year={2016},
	publisher={MIT Press}
}

@article{jordan1994hierarchical,
	author = {Jordan, Michael I and Jacobs, Robert A},
	journal = {Neural computation},
	number = {2},
	pages = {181--214},
	publisher = {MIT Press},
	title = {{Hierarchical mixtures of experts and the EM algorithm}},
	volume = {6},
	year = {1994}
}

@article{jacobs1991adaptive,
	title={Adaptive mixtures of local experts},
	author={Jacobs, Robert A and Jordan, Michael I and Nowlan, Steven J and Hinton, Geoffrey E},
	journal={Neural computation},
	volume={3},
	number={1},
	pages={79--87},
	year={1991},
	publisher={MIT Press}
}

@article{nguyen2018practical,
	title={Practical and theoretical aspects of mixture-of-experts modeling: An overview},
	author={Nguyen, Hien D and Chamroukhi, Faicel},
	journal={Wiley Interdisciplinary Reviews: Data Mining and Knowledge Discovery},
	volume={8},
	number={4},
	pages={e1246},
	year={2018},
	publisher={Wiley Online Library}
}

@article{bohning1992multinomial,
  title={Multinomial logistic regression algorithm},
  author={B{\"o}hning, Dankmar},
  journal={Annals of the institute of Statistical Mathematics},
  volume={44},
  number={1},
  pages={197--200},
  year={1992},
  publisher={Springer}
}

@article{bohning1988monotonicity,
  title={Monotonicity of quadratic-approximation algorithms},
  author={B{\"o}hning, Dankmar and Lindsay, Bruce G},
  journal={Annals of the Institute of Statistical Mathematics},
  volume={40},
  number={4},
  pages={641--663},
  year={1988},
  publisher={Springer}
}

@book{hall:heyde,
author = {Hall, P. and Heyde, C. C.},
address = {New York ;},
booktitle = {Martingale limit theory and its applications},
isbn = {0123193508},
language = {eng},
publisher = {Academic Press},
series = {Probability and mathematical statistics},
title = {Martingale limit theory and its applications / P. Hall, C.C. Heyde},
year = {1980}
}

\end{document}